\documentclass[twoside,11pt]{article}

\usepackage{blindtext}

\usepackage{jmlr2e}

\usepackage[math]{cellspace}
\usepackage{soul}   
\usepackage{makecell}   
\usepackage{accents}
\usepackage{float}  

\usepackage{color}
\newcommand{\rred}[1]{\textcolor{red}{#1}}

\newcommand{\blockComment}[1]{}

\newcommand{\connect}{}
\newcommand{\transit}{}
\newcommand{\explain}{}
\newcommand{\compare}{}

\newcommand{\future}{}

\newcommand{\ppp}[1]{\mathbb{P}\left( #1 \right)}
\newcommand{\prob}[1]{\mathcal{P}_{#1}}
\newcommand{\expt}[2]{\mathbb{E}_{#1}\left[ #2 \right]}

\newcommand{\mmatrix}[1]{\begin{pmatrix} 
#1
\end{pmatrix}}

\newcommand{\icol}[1]{
    \left(\begin{smallmatrix}#1\end{smallmatrix}\right)
}

\newcommand{\mrm}[1]{\mathrm{#1}}   

\newcommand{\eqarr}[1]{\begin{eqnarray}
#1
\end{eqnarray}}

\newcommand{\dx}{\mathrm{d}x}   

\newcommand{\ind}[1]{\mathbb{I}\left[ #1 \right]}   

\newcommand{\rmp}{\mathrm{p}}   
\newcommand{\rmn}{\mathrm{n}}   
\newcommand{\rmu}{\mathrm{u}}   
\newcommand{\rms}{\mathrm{s}}   
\newcommand{\rmd}{\mathrm{d}}   
\newcommand{\rmpc}{\mathrm{pc}}   
\newcommand{\rmsc}{\mathrm{sc}}   

\newcommand{\mcdp}{\gamma_\mathrm{1}}    
\newcommand{\mcdn}{\gamma_\mathrm{2}}
\newcommand{\mcdpp}{\gamma_\rmp}
\newcommand{\mcdnn}{\gamma_\rmn}

\newcommand{\corr}[1]{\bar{#1}}
\newcommand{\Mcorr}[1]{M_{\mathrm{#1}}}

\usepackage{lastpage}
\jmlrheading{oo}{2023}{1-\pageref{LastPage}}{oo/oo; Revised oo/oo}{oo/oo}{oo-oooo}{Chao-Kai Chiang and Masashi Sugiyama}

\ShortHeadings{A Weakly Supervised Learning Framework}{Chiang and Sugiyama}
\firstpageno{1}

\begin{document}

\title{Unified Risk Analysis for Weakly Supervised Learning}

\author{\name Chao{-}Kai Chiang \email chaokai@k.u-tokyo.ac.jp \\
       \name Masashi Sugiyama \email sugi@k.u-tokyo.ac.jp \\
       \addr Department of Complexity Science and Engineering \\
       Graduate School of Frontier Sciences \\
       The University of Tokyo \\
       5-1-5 Kashiwanoha, Kashiwa-shi, Chiba 277-8561, Japan}

\editor{Editor Name(s)}

\maketitle

\begin{abstract}
Among the flourishing research of weakly supervised learning (WSL), we recognize the lack of a unified interpretation of the mechanism behind the weakly supervised scenarios, let alone a systematic treatment of the risk rewrite problem, a crucial step in the empirical risk minimization approach. In this paper, we introduce a framework providing a comprehensive understanding and a unified methodology for WSL. The formulation component of the framework, leveraging a contamination perspective, provides a unified interpretation of how weak supervision is formed and subsumes fifteen existing WSL settings. The induced reduction graphs offer comprehensive connections over WSLs. 
The analysis component of the framework, viewed as a decontamination process, provides a systematic method of conducting risk rewrite.
In addition to the conventional inverse matrix approach, we devise a novel strategy called marginal chain aiming to decontaminate distributions. We justify the feasibility of the proposed framework by recovering existing rewrites reported in the literature.
\end{abstract}

\begin{keywords}
    weakly supervised learning, 
    classification risk,  
    learning with noisy labels,
    pairwise comparison,
    partial-label,
    confidence
\end{keywords}


\section{Introduction}
\label{sec:intro}
Accurate labels allow one to generalize to unseen data via empirical risk minimization (ERM) and analyze the generalization error in terms of the classification risk. 
In practice, there are various situations in which acquiring accurate labels is hard or even impossible.
One obstacle preventing us from acquiring accurate labels is labeling restrictions, such as imperfect supervision due to imperceptibility, time constraints, annotation costs, and even data sensitivity.
Another obstacle is the disruption by unavoidable noise from the environment.

To address the first obstacle of restrictions, various formulations have been studied under the notion of weakly supervised learning (WSL) \citep{wsl_18_survey/Zhou/18, wsl_sugibook/Sugiyama/BILSG/22}.
Based on various types of available label information, it evolves to thriving topics, including the conventional settings 
\citep{uu_18/Lu/NMS/19, uu_19_Lu, uu_21_Lu, pu_08_EN, pu_14/Plessis/NS/14, pu_15/Plessis/NS/15, pu_16/Niu/PSMS/16, pu_17/Kiryo/NPS/17, pu_19_SNZ} that investigating the potential of unlabeled data, complementary-label learning 
\citep{comp_17/Ishida/NHS/17, comp_18/Ishida/NMS/19, comp_17_Tao/Yu/LGT/18, comp_20_MCL/Feng/KHNAS/20, katsura2020bridging, chou2020unbiased}, partial-label learning 
\citep{cour2011learning, wang2019partial, partial_20_PIPL/Lv/XFNGS/20, partial_20_PCPL/Feng/LHXNGAS/20, partial_21_PPL/Wu/LS/23}, learning with confidence information 
\citep{pconf_17/Ishida/NS/18, scconf_21/Cao/FSXANS/21, sconf_21/Cao/FXANS/21, instance_20/Berthon/HLNS/21, soft_22/Ishida/YCNS/22}, and learning with comparative information 
\citep{su_18/Bao/NS/18, sdu_19/Shimada/BSS/21, pcomp_20/Feng/SLS/21, sconf_21/Cao/FXANS/21}.
Developing to resolve the second obstacle of noise, learning with noisy labels (LNL) can be categorized into two major formulations; one is called mutually contaminated distributions (MCD) \citep{mcd_13_Scott/Scott/BH/13, mcd_15/Menon/ROW/15, mcd_19_Scott/Katz-Samuels/BS/19} in which class-conditional distributions contaminate each other, and the other is named class-conditional random label noise (CCN) \citep{ccn_13/Natarajan/DRT/13, ccn_18/Natarajan/DRT/17} where a label is flipped by random noise.


Despite fruitful results and tremendous impact, we recognize a lack of {global understanding} and {systematic treatment} of WSL.
\explain 
From the perspective of {\emph{formulation}}, there are only scattered links among WSLs.
\cite{uu_18/Lu/NMS/19} and \cite{pcomp_20/Feng/SLS/21} showed that parameter substitution could reduce unlabeled-unlabeled to similar-unlabeled and positive-unlabeled settings.
Figure 1 in \cite{partial_21_PPL/Wu/LS/23} showed relationships among four WSLs of partial- and complementary-labels.
A similar observation can be found in the intersection of WSLs and LNLs.
Several WSLs were shown to be special cases of the MCD model, and some other WSLs are special cases of the CCN model.
For details, please refer to the discussions in Sections 8.2.3 and 9.2.4 of \cite{wsl_sugibook/Sugiyama/BILSG/22}.
\transit These connections encourage us to consider the possibility that there exists a unique interpretation that explains the mechanism behind WSL.
\explain 
Luckily, from the {\emph{methodological}} viewpoint, most of the existing WSL research adopted certain forms of the ERM approach.
A crucial shared step is to perform the risk rewrite, a way of rephrasing the uncomputable risk to a computable one in terms of the data-generating distributions.
A successful rewrite is the starting point of many downstream tasks, including but not limited to the following: Devising a practical or robust objective for training, comparing the strengths and properties of loss functions, proving the consistency, and analyzing generalization error bounds. 
However, many rewrite forms (summarized in Tables~\ref{tab:binary_WSL_rewrites} and \ref{tab:multiclass_WSL_rewrites}) look independent as if they are tailored to fit each problem's unique form of supervision and are not adaptable to each other.
\transit These seemingly non-adaptable estimators post a practical challenge: When facing a new form of weak (or noisy) supervision, we do not have a guideline or general strategy to leverage developed methods to address the new situation.

These observations raise the following questions we aim to answer in this paper: What is the essence of WSL? 
From a {formulation} perspective, can a unique interpretation be found to explain the mechanism behind WSL? 
Does a {methodology} exist to address as many WSLs as possible?
 
This paper proposes a framework with the following contributions to answer the research questions. 
\begin{enumerate}
    \item {To the best of our knowledge}, the framework is the first systematic attempt to address how and why WSLs are connected. The framework consists of a formulation component and an analysis component, subsuming fifteen weakly supervised scenarios. Table~\ref{tab:all_theorem_table} summarizes results generated from our framework.
    \item The formulation component, modeling from a \emph{contamination} perspective, provides WSL data generation processes with a coherent interpretation. 
    It produces three reduction graphs, shown in Tables~\ref{tab:MCD_matrices_summary}, \ref{tab:CCN_matrices_summary}, and \ref{tab:conf_matrices_summary}, revealing comprehensive connections between WSL formulations.
    It also unveils a distinctive confidence-based type WSLs that do not belong to the prominent MCD or CCN categories.
    \item The analysis component, leveraging the \emph{decontamination} concept, establishes a generic methodology for conducting risk rewrites for all WSLs discussed in this paper.
    The methodology also discovers the underlying mechanism that forms seemingly different risk rewrites.
    \item Regarding the technical contributions, a combined advantage of our framework and Theorem 1 from \cite{partial_21_PPL/Wu/LS/23} distinguishes two approaches, the inversion approach and the marginal chain approach presented by Theorems~\ref{thm:inv_method} and \ref{thm:marginal_chain}, to carry out the decontamination concept.
    The discovery of the marginal chain injects a brand-new thought to realize decontamination.
    \item We provide alternative proofs to demonstrate how the risk rewrites derived from our framework recover existing results reported in the literature. These alternatives have their respective logic stemming from the proposed framework.
\end{enumerate}

The idea of decontamination has been widely implemented and investigated. 
There are two major approaches, loss correction, and label correction, in LNL.
Closest to the current paper, \cite{partial_12_Cid-Suerio/Cid-Suerio/12}, \cite{18_Rooyen/Rooyen/W/17}, \cite{mcd_19_Scott/Katz-Samuels/BS/19}, \cite{17_Patrini/Patrini/RMNQ/17}, and \cite{15_Rooyen/Rooyen/W/15} exploited the inverse matrix, sometimes known as the backward method \citep{17_Patrini/Patrini/RMNQ/17}, to construct a corrected training loss to obtain an {unbiased estimator}.
There were deep learning methods leveraging the contamination assumption, sometimes called the forward method \citep{17_Patrini/Patrini/RMNQ/17}, to train a classifier \citep{17_Patrini/Patrini/RMNQ/17, comp_17_Tao/Yu/LGT/18, 15_Sukhbaatar/Sukhbaatar/F/14, 17_Goldberger/Goldberger/B/17, instance_20/Berthon/HLNS/21}. 
Besides modifying the loss function, one has two other strategies to manipulate the corrupted labels. 
The (iterative) pseudo-label method modified the labels for training \citep{18_Ma/Ma/WHZEXWB/18, label_18_Tanaka/Tanaka/IYA/18, label_14_Reed/Reed/LASER/14}. 
Filtering clean data points for training is the other option
\citep{17_Northcutt/Northcutt/WC/17, 19_Northcutt/Northcutt/JC/21, 18_Jiang/Jiang/ZLLF/18, 18_Han/Han/YYNXHTS/18, ICML:Yu+etal:2019}. 
Apart from classification, a different research branch studies conditions and methods for recovering the base distributions \citep{mcd_19_Scott/Katz-Samuels/BS/19, mcd_14_Scott/Blanchard/S/14, 16_Scott/Blanchard/FHPS/16}.

The current work is close to the loss correction approach in LNL. 
Most previous loss correction methods exploited invertibility to construct the corrected losses. 
In contrast, the marginal chain approach we propose in this paper adopts the conditional probability formula to build the corrected losses.
Many of the existing work targeted either the MCD or the CCN models. \cite{portion_20_Scott/Scott/Z/20}, 
\cite{instance_20/Berthon/HLNS/21}, \cite{17_Patrini/Patrini/RMNQ/17}, \cite{17_Goldberger/Goldberger/B/17}, \cite{15_Sukhbaatar/Sukhbaatar/F/14}, \cite{comp_17_Tao/Yu/LGT/18}, \cite{ccn_13/Natarajan/DRT/13}, \cite{ccn_18/Natarajan/DRT/17}, \cite{17_Northcutt/Northcutt/WC/17}, and \cite{19_Northcutt/Northcutt/JC/21} were based on the CCN model, and \cite{mcd_19_Scott/Katz-Samuels/BS/19}, \cite{mcd_14_Scott/Blanchard/S/14}, and \cite{16_Scott/Blanchard/FHPS/16} were based on the MCD model.
\cite{mcd_15/Menon/ROW/15}, \cite{18_Rooyen/Rooyen/W/17}, and \cite{mcd_19_Scott/Katz-Samuels/BS/19} studied multiple noise models at the same time. 
However, the current paper investigates the connections between MCD, CCN, and confidence-based settings simultaneously through the lens of matrix decontamination as broadly as possible to identify a generic methodology for WSLs.
Different from the current paper aiming for risk minimization, research also studied various performance measures, such as the balanced error rate \citep{portion_20_Scott/Scott/Z/20, portion_19_Scott/abs-1910-04665, mcd_15/Menon/ROW/15, TAAI:duPlessis+etal:2013}, the area under the receiver operating characteristic curve \citep{icml:2019:Charoenphakdee+LS, ML:Sakai+etal:2018, mcd_15/Menon/ROW/15}, and cost-sensitive measures \citep{ICML:Charoenphakdee+etal:2021, ccn_18/Natarajan/DRT/17}. 
We choose the classification risk as the only measure due to the focus of this paper.

The remaining sections are organized as follows.
Section~\ref{sec:background} reviews ERM in supervised learning, the risk rewrite problem, and the existing results.
Section~\ref{sec:recipe} presents the proposed framework.
We show that the proposed framework provides a unified way to formulate diverse weakly supervised scenarios in Section~\ref{sec:matrixFormulations}.
Section~\ref{sec:riskRewrite} demonstrates how to instantiate the framework to conduct risk rewrite.
Finally, we conclude the paper and discuss outlooks in Section~\ref{sec:future}.

\section{Preliminaries}
\label{sec:background}
Let $(y,x)$ be a training example where {the instance} $x \in \mathcal{X}$ and {the label} $y \in \mathcal{Y}$.
For binary classification, the label space $\mathcal{Y}$ is $\{ \rmp, \rmn\}$, and for multiclass classification with $K$ classes, $\mathcal{Y} = \{1,2, \ldots, K\} := [K]$.
The joint distribution is $\ppp{Y, X}$, the class prior is $\ppp{Y}$, the class-conditional distribution is $\ppp{X|Y}$, and the class probability function is $\ppp{Y|X}$.
Given a space of hypotheses $\mathcal{G}$, we denote the loss of a hypothesis $g\in\mathcal{G}$ on predicting $y$ of $(y, x)$ as $\ell_{Y=y}(g(x))$.
To accommodate concise expressions and readability for all WSLs considered in this paper simultaneously, we use alias notations when the context is unambiguous. 
Table~\ref{tab:small_notation_table} provides a set of common notations used in this paper.

\begin{table}[H]    
\centering
\caption{\label{tab:small_notation_table} Alias of Common Notations.}
{\renewcommand{\arraystretch}{1.2}
\begin{tabular}[t]{ |l|l|l| } 
    \hline
    Name of the notation & Expression & Aliases \\
    \hline
    Binary classes & $\{\rmp, \rmn\}$ &  \\
    Multiple classes & $\{1, \ldots, K\}$ & $[K]$ \\
    Compound set of $[K]$ & $2^{[K]} \backslash \left\{\emptyset, [K] \right\}$ & $\mathcal{S}$ \\
    \hline
    Joint distribution & $\ppp{Y=y, X=x}$ & $\prob{Y=y,x}$, $\prob{Y=y,X}$, or $\prob{Y,X}$ \\
    \hline
    Hypothesis and its space & $g \in \mathcal{G}$ & \\
    \hline
    Loss of $g$ & $\ell_{Y=y}(g(x))$ & $\ell_{y}$, $\ell_{y}(X)$, or $\ell_{Y}(g(X))$ \\
    \hline
    Classification risk & $\expt{Y,X}{\ell_{Y}(g(X))}$ & $R(g)$ \\
    \hline
    The $j$-th entry of vector $V$ & $\left( V\right)_{j}$ & $V_j$ \\
    \hline
    Class prior & $\ppp{Y=y}$ & $\pi_{y}$ \\
    Marginal & $\ppp{X}$ & $\prob{X}$ \\
    Class-conditional & $\ppp{X=x|Y=y}$ & $\prob{X|Y}$, $\prob{X|Y=y}$, or $\prob{x|Y=y}$ \\
    Confidence & $\ppp{Y=y|X=x}$ & $r_{y}(X)$, $r_{y}(x)$, or $r(X)$ if $y=\rmp$ \\
    \hline
\end{tabular}
}
\end{table}

We use $(y,x)$ instead of the convention $(x,y)$ to represent a data instance because, in the current paper, we focus on discussing different types of supervision. 
Placing the label before the instance emphasizes the type of supervision under investigation in theorems and derivations.

\subsection{Supervised Learning and the ERM Method}
In supervised learning with $K$ classes, the observed data is of the form
\eqarr{
    &&\{x_i^{y}\}_{i=1}^{n_y} \stackrel{\text{i.i.d.}}{\sim} \prob{X|Y=y}, \forall y \in [K]. \nonumber 
}
Notation $x_i^{y}$ denotes the shorthand of $(y, x_i).$
The goal of learning is to find a classifier $g \in \mathcal{G}$ that minimizes the classification risk
\eqarr{
    R(g) := \expt{Y, X}{\ell_{Y}(g(X))} = \sum_{y=1}^{K} \int_{x\in\mathcal{X}} \prob{Y=y, x} \, \ell_{Y=y}(g(x)) \, \dx. \label{eq:erm_b0}
}

To find such a classifier, ERM first constructs an empirical risk estimator with the data in hand:
\eqarr{
    \hat{R}(g) = 
    \sum_{y=1}^{K} \frac{1}{n_y} \sum_{i=1}^{n_y} \pi_y \ell_{Y=y}(g(x_i^{y})). \label{eq:erm_b2}
}
The estimator approximates $R(g)$ consistently since it can be shown that (\ref{eq:erm_b2}) approaches (\ref{eq:erm_b0}) as $N\rightarrow\infty$ \citep{14_Tewari_LearnTheory, pu_17/Kiryo/NPS/17} and \citep[Chapter~3]{wsl_sugibook/Sugiyama/BILSG/22}.
Then, ERM takes $\hat{R}(g)$ as the training objective and optimizes it to find the optimal classifier 
\eqarr{
    g^* = \arg\min_{g \in \mathcal{G}} \hat{R}(g) \label{eq:erm_b4}
}
in the hypothesis space $\mathcal{G}$ as the output of ERM.

\subsection{The Risk Rewrite Problem and Existing Results}
\label{sec:formulations_old}
In every WSL scenario, the goal of learning is the same as supervised learning.
However, the observed data is no longer as perfectly labeled as in supervised learning.
That said, there are differences in the formulations of the observed data and the ways of estimating the classification risk.
We begin with reviewing WSLs derived from binary classes. 
For $K=2$, we assign $\mathcal{Y} := \{\rmp, \rmn\}$.

\subsubsection{Positive-Unlabeled (PU) learning}
\label{sec:review_PU}
The observed data in PU learning \citep{pu_15/Plessis/NS/15} is of the form
\eqarr{
    \begin{aligned}
    \label{eq:formulate_PU}
        & \left\{x_i^{\rmp}\right\}_{i=1}^{n_{\rmp}} 
        \stackrel{\text{i.i.d.}}{\sim} \prob{\mrm{P}} 
        := \prob{X|Y=\rmp}, \\
        & \left\{x_j^{\rmu}\right\}_{j=1}^{n_{\rmu}} 
        \stackrel{\text{i.i.d.}}{\sim} \prob{\mrm{U}} 
        := \pi_\rmp \; \prob{X|Y=\rmp} + \pi_\rmn \; \prob{X|Y=\rmn},
    \end{aligned}
}
where $x_j^{\rmu}$ is viewed as the shorthand of $(\rmu, x_j)$ symbolizing the unlabeled data\footnote{Seemingly being redundant, but it is helpful to use $(\rmu, x_j)$ to distinguish it from the positively labeled instance $(\rmp, x_i)$.}.
The unlabeled data set $\{x_j^{\rmu}\}_j$ consists of a mixture of samples from $\prob{X|Y=\rmp}$ and $\prob{X|Y=\rmn}$ with proportion $\pi_\rmp$.
Since the information of negatively sampled data is unavailable, (\ref{eq:erm_b2}) is uncomputable, causing directly optimizing (\ref{eq:erm_b4}) infeasibility.
Therefore, to make ERM applicable, the \emph{risk rewrite problem} \citep{wsl_sugibook/Sugiyama/BILSG/22} asks:
\begin{center}
    Can one rephrase the classification risk $R(g)$ (\ref{eq:erm_b0}) in terms of the given data formulation?
\end{center}
\cite{pu_15/Plessis/NS/15} rewrote the classification risk in terms of the data-generating distributions $\prob{\mrm{P}}$ and $\prob{\mrm{U}}$ as
\eqarr{
    R(g) 
    = 
    \expt{\mrm{P}}{
        \pi_\rmp \ell_\rmp - \pi_\rmp \ell_\rmn
    } + \expt{\mrm{U}}{
        \ell_\rmn
    }. \label{eq:review_rewrite_PU} 
}

\subsubsection{Positive-confidence (Pconf) Learning Learning}
\label{sec:review_Pconf}
The observed data in Pconf learning \citep{pconf_17/Ishida/NS/18} is of the form
\eqarr{
    \left\{x_i, r(x_i)\right\}_{i=1}^{n}, \nonumber
}
where
\eqarr{
    \begin{aligned}
    \label{eq:formulate_Pconf}
        & x_i \stackrel{\text{i.i.d.}}{\sim} \prob{\mrm{P}} := \prob{X|Y=\rmp}, \\
        & r(x_i) := \prob{Y=\rmp|X=x_i}.
    \end{aligned}
}
The function $r(x)$ represents how confident an example $x$ would be positively labeled.
\cite{pconf_17/Ishida/NS/18} rewrote the classification risk as
\eqarr{
    R(g)
    =
    \pi_\rmp
    \expt{\mrm{P}}{
        \ell_\rmp + \frac{1-r(X)}{r(X)} \ell_\rmn
    }. \label{eq:review_rewrite_Pconf}
}

\subsubsection{Unlabeled-Unlabeled (UU) learning}
\label{sec:review_UU}
The observed data in UU learning \citep{uu_18/Lu/NMS/19} is of the form
\eqarr{
    \begin{aligned}
    \label{eq:formulate_UU}
        &\left\{x_i^{\rmu_1}\right\}_{i=1}^{n_{\rmu_1}} 
        \stackrel{\text{i.i.d.}}{\sim} \prob{\mrm{U}_1} 
        := (1-\mcdp) \; \prob{X|Y=\rmp} + \mcdp \; \prob{X|Y=\rmn}, \\
        &\left\{x_j^{\rmu_2}\right\}_{j=1}^{n_{\rmu_2}} 
        \stackrel{\text{i.i.d.}}{\sim} \prob{\mrm{U}_2} 
        := \mcdn \; \prob{X|Y=\rmp} + (1-\mcdn) \; \prob{X|Y=\rmn},
    \end{aligned}
}
where $x_i^{\rmu_1}$ (resp.\ $x_j^{\rmu_2}$) being the shorthand of $(\rmu_1, x_i)$ (resp.\ $(\rmu_2, x_j)$) represents $x_i$ (resp.\ $x_j$) belonging to the unlabeled data whose mixture parameter is $\mcdp$ (resp.\ $\mcdn$).
Notice a difference that the mixture proportion of the unlabeled data in PU learning is $\pi_\rmp$.  
\cite{uu_18/Lu/NMS/19} rewrote the classification risk in terms of the data-generating distributions $\prob{\mrm{U}_1}$ and $\prob{\mrm{U}_2}$ as follows:
Assume $\mcdp + \mcdn \neq 1$. Then,
\eqarr{
    R(g)
    =
    \expt{\mrm{U}_1}{
        \frac{(1-\mcdn)\pi_\rmp}{1-\mcdp-\mcdn} \ell_\rmp + \frac{-\mcdn \pi_\rmn}{1-\mcdp-\mcdn} \ell_\rmn
    } 
    + 
    \expt{\mrm{U}_2}{
        \frac{-\mcdp\pi_\rmp}{1-\mcdp-\mcdn} \ell_\rmp + \frac{(1-\mcdp)\pi_\rmn}{1-\mcdp-\mcdn} \ell_\rmn
    }. \label{eq:review_rewrite_UU}
}

\subsubsection{Similar-Unlabeled (SU) learning}
\label{sec:review_SU}
The observed data in SU learning \citep{su_18/Bao/NS/18} is of the form
\eqarr{
    \begin{aligned}
    \label{eq:formulate_SU}
        &\left\{\left(x_i^{\rms}, x_i^{\rms'}\right)\right\}_{i=1}^{n_{\rms}} 
        \stackrel{\text{i.i.d.}}{\sim} \prob{\mrm{S}} 
        := \frac{\pi_\rmp^2 \prob{X|Y=\rmp} \prob{X'|Y=\rmp} + \pi_\rmn^2 \prob{X|Y=\rmn} \prob{X'|Y=\rmn}}{\pi_\rmp^2+\pi_\rmn^2}, \\
        &\left\{x_j^{\rmu}\right\}_{j=1}^{n_{\rmu}} 
        \stackrel{\text{i.i.d.}}{\sim} \prob{\mrm{U}} 
        := \pi_\rmp \; \prob{X|Y=\rmp} + \pi_\rmn \; \prob{X|Y=\rmn}.
    \end{aligned}
}
The word ``similar'' means the examples in every $(x^{\rms},x^{\rms'})$ pair have the same label; either both are positive, or both are negative.
Under the assumption $\pi_\rmp \neq \pi_\rmn$,
\cite{su_18/Bao/NS/18} rewrote the classification risk as
\eqarr{
    R(g)
    =
    \left(\pi_\rmp^2+\pi_\rmn^2\right)\expt{\mrm{S}}{\frac{\mathcal{L}(X) + \mathcal{L}(X')}{2}}
    +
    \expt{\mrm{U}}{\mathcal{L}_{-}(X)}, \label{eq:review_rewrite_SU}
}
where
\eqarr{
    \mathcal{L}(X) &:=& \frac{1}{\pi_\rmp - \pi_\rmn} \ell_{\rmp}(X) - \frac{1}{\pi_\rmp - \pi_\rmn} \ell_{\rmn}(X), \nonumber \\
    \mathcal{L}_{-}(X) &:=& - \frac{\pi_\rmn}{\pi_\rmp -\pi_\rmn} \ell_{\rmp}(X) + \frac{\pi_\rmp}{\pi_\rmp - \pi_\rmn} \ell_{\rmn}(X). \nonumber
}

\subsubsection{Dissimilar-Unlabeled (DU) learning}
\label{sec:review_DU}
The observed data in DU learning \citep{sdu_19/Shimada/BSS/21} is of the form
\eqarr{
    \begin{aligned}
    \label{eq:formulate_DU}
        &\left\{\left(x_i^{\rmd}, x_i^{\rmd'}\right)\right\}_{i=1}^{n_{\rmd}} 
        \stackrel{\text{i.i.d.}}{\sim} \prob{\mrm{D}} 
        := \frac{\prob{X|Y=\rmp} \prob{X'|Y=\rmn} + \prob{X|Y=\rmn} \prob{X'|Y=\rmp}}{2}, \\
        &\left\{x_j^{\rmu}\right\}_{j=1}^{n_{\rmu}} 
        \stackrel{\text{i.i.d.}}{\sim} \prob{\mrm{U}} 
        := \pi_\rmp \; \prob{X|Y=\rmp} + \pi_\rmn \; \prob{X|Y=\rmn}.
    \end{aligned}
}
The word ``dissimilar'' means the examples in every $(x^{\rmd},x^{\rmd'})$ pair have distinct labels.
Under the assumption $\pi_\rmp \neq \pi_\rmn$,
\cite{sdu_19/Shimada/BSS/21} rewrote the classification risk as 
\eqarr{
    R(g)
    =
    2\pi_\rmp\pi_\rmn
    \expt{\mrm{D}}{-\frac{\mathcal{L}(X) + \mathcal{L}(X')}{2}}
    +
    \expt{\mrm{U}}{\mathcal{L}_{+}(X)}, \label{eq:review_rewrite_DU}
}
where
\eqarr{
    \mathcal{L}(X) &=& \frac{1}{\pi_\rmp - \pi_\rmn} \ell_{\rmp}(X) - \frac{1}{\pi_\rmp - \pi_\rmn} \ell_{\rmn}(X), \nonumber \\
    \mathcal{L}_{+}(X) &:=& \frac{\pi_\rmp}{\pi_\rmp -\pi_\rmn} \ell_{\rmp}(X) - \frac{\pi_\rmn}{\pi_\rmp - \pi_\rmn} \ell_{\rmn}(X). \nonumber
}
Note that $\mathcal{L}(X)$ has been defined in the SU setting. 
We repeat it here for clarity.

\subsubsection{Similar-Dissimilar (SD) learning}
\label{sec:review_SD}
The observed data in SD learning \citep{sdu_19/Shimada/BSS/21} is of the form
\eqarr{
    \begin{aligned}
    \label{eq:formulate_SD}
        &\left\{\left(x_i^{\rms}, x_i^{\rms'}\right)\right\}_{i=1}^{n_{\rms}} 
        \stackrel{\text{i.i.d.}}{\sim} \prob{\mrm{S}} 
        := \frac{\pi_\rmp^2 \prob{X|Y=\rmp} \prob{X'|Y=\rmp} + \pi_\rmn^2 \prob{X|Y=\rmn} \prob{X'|Y=\rmn}}{\pi_\rmp^2+\pi_\rmn^2}, \\
        &\left\{\left(x_i^{\rmd}, x_i^{\rmd'}\right)\right\}_{i=1}^{n_{\rmd}} 
        \stackrel{\text{i.i.d.}}{\sim} \prob{\mrm{D}} 
        := \frac{\prob{X|Y=\rmp} \prob{X'|Y=\rmn} + \prob{X|Y=\rmn} \prob{X'|Y=\rmp}}{2}.
    \end{aligned}
}
Under the assumption $\pi_\rmp \neq \pi_\rmn$,
\cite{sdu_19/Shimada/BSS/21} rewrote the classification risk as 
\eqarr{
    R(g)
    =
    \left( \pi_\rmp^2 + \pi_\rmn^2 \right)
    \expt{\mrm{S}}{
        \frac{\mathcal{L}_{+}(X) + \mathcal{L}_{+}(X')}{2}
    } + 
    2\pi_\rmp\pi_\rmn
    \expt{\mrm{D}}{
        \frac{\mathcal{L}_{-}(X) + \mathcal{L}_{-}(X')}{2}
    }, \label{eq:review_rewrite_SD}
}
where
\eqarr{
    \mathcal{L}_{+}(X) &=& \frac{\pi_\rmp}{\pi_\rmp -\pi_\rmn} \ell_{\rmp}(X) - \frac{\pi_\rmn}{\pi_\rmp - \pi_\rmn} \ell_{\rmn}(X), \nonumber \\
    \mathcal{L}_{-}(X) &=& - \frac{\pi_\rmn}{\pi_\rmp -\pi_\rmn} \ell_{\rmp}(X) + \frac{\pi_\rmp}{\pi_\rmp - \pi_\rmn} \ell_{\rmn}(X). \nonumber
}
Note that $\mathcal{L}_{+}(X)$ and $\mathcal{L}_{-}(X)$ have been defined in the DU and SU settings. 
We repeat them here for clarity.

\subsubsection{Pairwise Comparison (Pcomp) Learning}
\label{sec:review_Pcomp}
The observed data in Pcomp learning \citep{pcomp_20/Feng/SLS/21} is of the form
\eqarr{
    \left\{\left(x_i^{\rmpc}, x_i^{\rmpc'}\right)\right\}_{i=1}^{n_{\rmpc}} 
    \stackrel{\text{i.i.d.}}{\sim} \prob{\mrm{PC}} 
    := \frac{\pi_\rmp^2 \prob{X|Y=\rmp} \prob{X'|Y=\rmp} + \pi_\rmp\pi_\rmn \prob{X|Y=\rmp} \prob{X'|Y=\rmn} + \pi_\rmn^2 \prob{X|Y=\rmn} \prob{X'|Y=\rmn}}{\pi_\rmp^2 + \pi_\rmp\pi_\rmn + \pi_\rmn^2}. \nonumber \\ \label{eq:formulate_Pcomp}
}
The pairwise comparison encodes a meaning that each $x^{\rmpc}$ ``can not be more negative'' than $x^{\rmpc'}$ in the $(x^{\rmpc}, x^{\rmpc'})$ pair.
That is, the labels in $(x^{\rmpc}, x^{\rmpc'})$ are of the form $(\rmp, \rmp)$, $(\rmp, \rmn)$, or $(\rmn, \rmn)$.
\cite{pcomp_20/Feng/SLS/21} rewrote the classification risk as
\eqarr{
    R(g)
    =
    \expt{\mrm{Sup}}{\ell_\rmp - \pi_\rmp \ell_\rmn}
    +
    \expt{\mrm{Inf}}{-\pi_\rmn \ell_\rmp + \ell_\rmn}, \label{eq:review_rewrite_Pcomp}
}
where the expectations are computed over the following distributions
\eqarr{
    \prob{\mrm{Sup}} &:=& \int_{x'\in\mathcal{X}} \prob{\mrm{PC}} \, \dx', \nonumber \\
    \prob{\mrm{Inf}} &:=& \int_{x\in\mathcal{X}} \prob{\mrm{PC}} \, \dx. \nonumber
}

\subsubsection{Similarity-Confidence Learning (Sconf) Learning}
\label{sec:review_Sconf}
The observed data in Sconf learning \citep{sconf_21/Cao/FXANS/21} is of the form
\eqarr{
    \left\{x_i^{\rmsc}, x_i^{\rmsc'}, r\left(x_i^{\rmsc},x_i^{\rmsc'}\right)\right\}_{i=1}^{n}, \nonumber
}
where
\eqarr{
    \begin{aligned}
    \label{eq:formulate_Sconf}
        & x_i^{\rmsc}
        \stackrel{\text{i.i.d.}}{\sim} \prob{X} 
        := \pi_\rmp \; \prob{X|Y=\rmp} + \pi_\rmn \; \prob{X|Y=\rmn}, \\
        & x_i^{\rmsc'}
        \stackrel{\text{i.i.d.}}{\sim} \prob{X'} 
        := \pi_\rmp \; \prob{X'|Y=\rmp} + \pi_\rmn \; \prob{X'|Y=\rmn}, \\
        & r\left(x_i^{\rmsc},x_i^{\rmsc'}\right) := \prob{Y = y_i^{\rmsc} = Y' = y_i^{\rmsc'}|X = x_i^{\rmsc}, X' = x_i^{\rmsc'}}.
    \end{aligned}
}
\cite{sconf_21/Cao/FXANS/21} rewrote the classification risk as 
\eqarr{
    R(g)
    =
    \expt{X,X'}{
        \frac{r(X,X')-\pi_\rmn}{\pi_\rmp-\pi_\rmn} \mathcal{L}_{\rmp}(X, X') 
        +
        \frac{\pi_\rmp-r(X,X')}{\pi_\rmp-\pi_\rmn} \mathcal{L}_{\rmn}(X, X') 
    }, \label{eq:review_rewrite_Sconf}
}
where
\eqarr{
    \mathcal{L}_{\rmp}(X, X') := \frac{\ell_{\rmp}(X) + \ell_{\rmp}(X')}{2}, \nonumber \\
    \mathcal{L}_{\rmn}(X, X') := \frac{\ell_{\rmn}(X) + \ell_{\rmn}(X')}{2}. \nonumber
}

\subsubsection{Complementary-Label (CL) Learning}
\label{sec:review_CL}
One can also formulate weak supervision from multiclass classification.
For $K$ classes, we assign $\mathcal{Y} := [K]$.

The observed data in CL learning \citep{comp_18/Ishida/NMS/19} is of the form
\eqarr{
    \left\{(\corr{s}_i,x_i)\right\}_{i=1}^{n}
    \stackrel{\text{i.i.d.}}{\sim} \prob{\corr{S},X}
    := \frac{1}{K-1} \sum_{Y\neq \corr{S}} \prob{Y,X}. \label{eq:formulate_CL}
}
As is named ``complementary,'' $\corr{s} \in [K]$ represents that the true label $y$ of $x$ cannot be $\corr{s}$.
\cite{comp_18/Ishida/NMS/19} rewrote the classification risk as
\eqarr{
    R(g)
    =
    \expt{\corr{S}, X}
    {
        \sum_{y=1}^{K} \ell_{y} - (K-1)\ell_{\corr{S}}
    }. \label{eq:review_rewrite_CL}
}

\subsubsection{Multi-Complementary-Label (MCL) Learning}
\label{sec:review_MCL}
The observed data in MCL learning \citep{comp_20_MCL/Feng/KHNAS/20} is of the form
\eqarr{
    \left\{(\corr{s}_i,x_i)\right\}_{i=1}^{n}
    \stackrel{\text{i.i.d.}}{\sim} \prob{\corr{S}, X}
    :=
    \begin{cases}
        \sum_{d=1}^{K-1} \prob{|\corr{S}| = d} \cdot \frac{1}{{K-1 \choose |\corr{S}|}} \sum_{Y\notin \corr{S}} \prob{Y,X}, & \text{if}\ |\corr{S}|=d, \label{eq:formulate_MCL} \\
        0, & \text{otherwise.}
    \end{cases} 
}
Generalized from CL, $\corr{s} \subset [K]$ in MCL is a set of classes of size $d \in [K-1]$, representing multiple exclusions.
In other words, CL is the special case of MCL with $d=1$.
\cite{comp_20_MCL/Feng/KHNAS/20} rewrote the classification risk as
\eqarr{
    R(g)
    =
    \sum_{d=1}^{K-1} \prob{|\corr{S}|=d} \expt{\corr{S},X||\corr{S}|=d}{
        \sum_{y\notin\corr{S}} \ell_{y} - \frac{K-1-|\corr{S}|}{|\corr{S}|} \sum_{\corr{s} \in \corr{S}} \ell_{\corr{s}}
    }. \label{eq:review_rewrite_MCL}
}

\subsubsection{Provably Consistent Partial-Label (PCPL) Learning}
\label{sec:review_PCPL}
The observed data in PCPL learning \citep{partial_20_PCPL/Feng/LHXNGAS/20} is of the form
\eqarr{
    \left\{(s_i,x_i)\right\}_{i=1}^{n}
    \stackrel{\text{i.i.d.}}{\sim} \prob{S, X}
    := \frac{1}{2^{K-1}-1} \sum_{Y\in S} \prob{Y, X}. \label{eq:formulate_PCPL}
}
A partial-label $s \subset [K]$ is a set of classes containing the true label $y$ of $x$.
\cite{partial_20_PCPL/Feng/LHXNGAS/20} rewrote the classification risk as
\eqarr{
    R(g)
    =
    \frac{1}{2} \expt{S,X}{\sum_{y=1}^{K} \frac{\prob{Y=y|X} }{\sum_{a \in S} \prob{Y=a|X}} \ell_{y}}. \label{eq:review_rewrite_PCPL}
}

\subsubsection{Proper Partial-Label (PPL) Learning}
\label{sec:review_PPL}
The observed data in PPL learning \citep{partial_21_PPL/Wu/LS/23} is of the form
\eqarr{
    \left\{(s_i, x_i)\right\}_{i=1}^{n}
    \stackrel{\text{i.i.d.}}{\sim} \prob{S, X}
    := C(S,X) \sum_{Y\in S} \prob{Y, X}. \label{eq:formulate_PPL}
}
The weight $\frac{1}{2^{K-1}-1}$ in PCPL is generalized to $C(S,X)$, a function of the partial-label and the instance, allowing one to characterize the ``properness'' of a partial-label.
\cite{partial_21_PPL/Wu/LS/23} rewrote the classification risk as
\eqarr{
    R(g) 
    = 
    \expt{S,X}{\sum_{y\in S} \frac{\prob{Y=y|X}}{\sum_{a \in S} \prob{Y=a|X}} \ell_{y}}. \label{eq:review_rewrite_PPL}
}

\subsubsection{Single-Class Confidence (SC-Conf) Learning}
\label{sec:review_SC-Conf}
The observed data in SC-Conf learning \citep{scconf_21/Cao/FSXANS/21} is of the form
\eqarr{
    \left\{x_i, r_1(x_i), \ldots, r_K(x_i)\right\}_{i=1}^{n}, \nonumber
}
where
\eqarr{
    \begin{aligned}
    \label{eq:formulate_SC-conf}
        &x_i \stackrel{\text{i.i.d.}}{\sim} \prob{X|Y=y_\mrm{s}} \text{ with } y_\mrm{s} \in [K], \\
        &r_k(x_i) := \prob{Y=k|X=x_i} \text{ for each } k \in [K].
    \end{aligned}
}
The constraint of SC-Conf is that the examples are sampled from a specific class $y_\mrm{s}$.
The key to risk rewrite is the availability of confident information $r_k(x)$ about each class. 
\cite{scconf_21/Cao/FSXANS/21} rewrote the classification risk as 
\eqarr{
    R(g)
    =
    \pi_{y_\mrm{s}} \expt{X|Y=y_\mrm{s}}{\sum_{y=1}^{K} \frac{r_{y}(X)}{r_{y_\mrm{s}}(X)} \ell_{y}}. \label{eq:review_rewrite_SC-conf}
}

\subsubsection{Subset Confidence (Sub-Conf) Learning}
\label{sec:review_Sub-Conf}
The observed data in Sub-Conf learning \citep{scconf_21/Cao/FSXANS/21} is of the form
\eqarr{
    \left\{x_i, r_1(x_i), \ldots, r_K(x_i)\right\}_{i=1}^{n}, \nonumber
}
where
\eqarr{
    \begin{aligned}
    \label{eq:formulate_Sub-conf}
        &x_i \stackrel{\text{i.i.d.}}{\sim} \prob{X|Y \in \mathcal{Y}_\mrm{s}} \text{ with } \mathcal{Y}_\mrm{s} \subset [K], \\
        &r_k(x_i) := \prob{Y=k|X=x_i} \text{ for each } k \in [K].
    \end{aligned}
}
Sub-Conf is a relaxed setting of SC-Conf where the samples come from a set of classes $\mathcal{Y}_\mrm{s}$.
\cite{scconf_21/Cao/FSXANS/21} rewrote the classification risk as 
\eqarr{
    R(g)
    =
    \pi_{\mathcal{Y}_\mrm{s}} \expt{X|Y\in\mathcal{Y}_\mrm{s}}{\sum_{y=1}^{K} \frac{r_{y}(X)}{r_{\mathcal{Y}_\mrm{s}}(X)} \ell_{y}}, \label{eq:review_rewrite_Sub-conf}
}
where
$\pi_{\mathcal{Y}_\mrm{s}} := \sum_{j\in\mathcal{Y}_\mrm{s}} \pi_j$,
and 
$r_{\mathcal{Y}_\mrm{s}}(X) := \prob{Y\in\mathcal{Y}_\mrm{s}|X} = \sum_{j\in\mathcal{Y}_\mrm{s}}\prob{Y=j|X}$.

\subsubsection{Soft-Label Learning}
\label{sec:review_Soft}
\cite{soft_22/Ishida/YCNS/22} formulated soft-label learning under the binary setting, in which the observed data is of the form
\eqarr{
    \left\{x_i, r(x_i)\right\}_{i=1}^{n}, \nonumber
}
where
\eqarr{
    \begin{aligned}
    \label{eq:formulate_binary_Soft}
        &x_i \stackrel{\text{i.i.d.}}{\sim} \prob{X} := \prob{Y=\rmp,X} + \prob{Y=\rmn,X}, \\
        &r(x_i) := \prob{Y=\rmp|X=x_i}.
    \end{aligned}
}
It is straightforward to obtain a corresponding formulation under the multiclass setting: 
\eqarr{
    \left\{x_i, r_1(x_i), \ldots, r_K(x_i)\right\}_{i=1}^{n}, \nonumber
}
where
\eqarr{
    \begin{aligned}
    \label{eq:formulate_Soft}
        &x_i \stackrel{\text{i.i.d.}}{\sim} \prob{X} := \sum_{k=1}^{K} \prob{Y=k,X}, \\
        &r_k(x_i) := \prob{Y=k|X=x_i} \text{ for each } k \in [K].
    \end{aligned}
}
The difference between SC-Conf and multiclass soft-label (resp.\ the difference between Pconf and binary soft-label) is the sample distribution of $x$.
We rewrote the classification risk as
\eqarr{
    R(g)
    =
    \expt{X}{\sum_{y=1}^{K} r_y(X) \ell_{y}}. \label{eq:review_rewrite_Soft}
}

\subsubsection{Summary of Existing WSL Formulations and Risk Rewrites}
\label{sec:review_all_formulations}
We summarize the weakly supervised scenarios discussed and their risk rewrite results.
The formulations are divided into the binary classification settings in Table~\ref{tab:binary_WSL_formulations} and the multiclass classification settings in Table~\ref{tab:multiclass_WSL_formulations}.
We list the formulations in chronological order, according to their publication order.
Tables~\ref{tab:binary_WSL_rewrites} and \ref{tab:multiclass_WSL_rewrites} are the corresponding rewrites. 

\begin{table}[H]
\centering
\caption{Binary WSL formulations.}
\label{tab:binary_WSL_formulations}
{\renewcommand{\arraystretch}{1.2}
\begin{tabular}[t]{ |Sc|Sl| } 
    \hline
    WSL & Formulation 
    \\ \hline
    PU &  
        $\begin{aligned}
            \{x_i^{\rmp}\}_{i=1}^{n_{\rmp}} 
            \stackrel{\text{i.i.d.}}{\sim} \prob{\mrm{P}} 
            &:= \prob{X|Y=\rmp},
            \\
            \{x_j^{\rmu}\}_{j=1}^{n_{\rmu}} 
            \stackrel{\text{i.i.d.}}{\sim} \prob{\mrm{U}} 
            &:= \pi_\rmp \; \prob{X|Y=\rmp} + \pi_\rmn \; \prob{X|Y=\rmn}.
        \end{aligned}$
        \; (\ref{eq:formulate_PU})
    \\ \hline
    Pconf &
        $\begin{aligned}
            & \left\{x_i, r(x_i)\right\}_{i=1}^{n}, \text{ where}
            \\
            &\;\;\;\; x_i \stackrel{\text{i.i.d.}}{\sim} \prob{\mrm{P}} := \prob{X|Y=\rmp}, 
            \\
            &\;\;\;\; r(x_i) := \prob{Y=\rmp|X=x_i}.
        \end{aligned}$
        \; (\ref{eq:formulate_Pconf})
    \\ \hline
    UU &
        $\begin{aligned}
            \{x_i^{\rmu_1}\}_{i=1}^{n_{\rmu_1}} 
            \stackrel{\text{i.i.d.}}{\sim} \prob{\mrm{U}_1} 
            &:= (1-\mcdp) \; \prob{X|Y=\rmp} + \mcdp \; \prob{X|Y=\rmn},
            \\
            \{x_j^{\rmu_2}\}_{j=1}^{n_{\rmu_2}} 
            \stackrel{\text{i.i.d.}}{\sim} \prob{\mrm{U}_2} 
            &:= \mcdn \; \prob{X|Y=\rmp} + (1-\mcdn) \; \prob{X|Y=\rmn}.
        \end{aligned}$
        \; (\ref{eq:formulate_UU})
    \\ \hline
    SU &
        $\begin{aligned}
            \left\{\left(x_i^{\rms}, x_i^{\rms'}\right)\right\}_{i=1}^{n_{\rms}} 
            \stackrel{\text{i.i.d.}}{\sim} \prob{\mrm{S}} 
            &:= \frac{\pi_\rmp^2 \prob{X|Y=\rmp} \prob{X'|Y=\rmp} + \pi_\rmn^2 \prob{X|Y=\rmn} \prob{X'|Y=\rmn}}{\pi_\rmp^2+\pi_\rmn^2},
            \\
            \left\{x_j^{\rmu}\right\}_{j=1}^{n_{\rmu}} 
            \stackrel{\text{i.i.d.}}{\sim} \prob{\mrm{U}} 
            &:= \pi_\rmp \; \prob{X|Y=\rmp} + \pi_\rmn \; \prob{X|Y=\rmn}.
        \end{aligned}$
        \; (\ref{eq:formulate_SU})
    \\ \hline
    DU &
        $\begin{aligned}
            \left\{\left(x_i^{\rmd}, x_i^{\rmd'}\right)\right\}_{i=1}^{n_{\rmd}} 
            \stackrel{\text{i.i.d.}}{\sim} \prob{\mrm{D}} 
            &:= \frac{\prob{X|Y=\rmp} \prob{X'|Y=\rmn} + \prob{X|Y=\rmn} \prob{X'|Y=\rmp}}{2},
            \\
            \left\{x_j^{\rmu}\right\}_{j=1}^{n_{\rmu}} 
            \stackrel{\text{i.i.d.}}{\sim} \prob{\mrm{U}} 
            &:= \pi_\rmp \; \prob{X|Y=\rmp} + \pi_\rmn \; \prob{X|Y=\rmn}.
        \end{aligned}$
        \; (\ref{eq:formulate_DU})
    \\ \hline
    SD &
        $\begin{aligned}
            \left\{\left(x_i^{\rms}, x_i^{\rms'}\right)\right\}_{i=1}^{n_{\rms}} 
            \stackrel{\text{i.i.d.}}{\sim} \prob{\mrm{S}} 
            &:= \frac{\pi_\rmp^2 \prob{X|Y=\rmp} \prob{X'|Y=\rmp} + \pi_\rmn^2 \prob{X|Y=\rmn} \prob{X'|Y=\rmn}}{\pi_\rmp^2+\pi_\rmn^2},
            \\
            \left\{\left(x_i^{\rmd}, x_i^{\rmd'}\right)\right\}_{i=1}^{n_{\rmd}} 
            \stackrel{\text{i.i.d.}}{\sim} \prob{\mrm{D}} 
            &:= \frac{\prob{X|Y=\rmp} \prob{X'|Y=\rmn} + \prob{X|Y=\rmn} \prob{X'|Y=\rmp}}{2}.
        \end{aligned}$
        \; (\ref{eq:formulate_SD})
    \\ \hline
    Pcomp &
        $\begin{aligned}
            & \left\{\left(x_i^{\rmpc}, x_i^{\rmpc'}\right)\right\}_{i=1}^{n_{\rmpc}} 
            \stackrel{\text{i.i.d.}}{\sim} \prob{\mrm{PC}} 
            \\
            &\;\;\;\; := \frac{\pi_\rmp^2 \prob{X|Y=\rmp} \prob{X'|Y=\rmp} + \pi_\rmp\pi_\rmn \prob{X|Y=\rmp} \prob{X'|Y=\rmn} + \pi_\rmn^2 \prob{X|Y=\rmn} \prob{X'|Y=\rmn}}{\pi_\rmp^2 + \pi_\rmp\pi_\rmn + \pi_\rmn^2}.
        \end{aligned}$
        \; (\ref{eq:formulate_Pcomp})
    \\ \hline
    Sconf &
        $\begin{aligned}
            & \left\{x_i^{\rmsc}, x_i^{\rmsc'}, r\left(x_i^{\rmsc},x_i^{\rmsc'}\right) \right\}_{i=1}^{n_{\rmsc}}, \text{ where}
            \\
            &\;\;\;\; x_i^{\rmsc}
            \stackrel{\text{i.i.d.}}{\sim} \prob{X} 
            := \pi_\rmp \; \prob{X|Y=\rmp} + \pi_\rmn \; \prob{X|Y=\rmn},
            \\
            &\;\;\;\; x_i^{\rmsc'}
            \stackrel{\text{i.i.d.}}{\sim} \prob{X'} 
            := \pi_\rmp \; \prob{X'|Y=\rmp} + \pi_\rmn \; \prob{X'|Y=\rmn},
            \\
            &\;\;\;\; r\left(x_i^{\rmsc},x_i^{\rmsc'}\right) := \prob{Y = y_i^{\rmsc} = Y' = y_i^{\rmsc'}|X = x_i^{\rmsc}, X' = x_i^{\rmsc'}}.
        \end{aligned}$ 
        \; (\ref{eq:formulate_Sconf})
    \\ \hline
\end{tabular}
}   
\end{table}     

\begin{table}[H]
\centering
\caption{Multiclass WSL formulations.}
\label{tab:multiclass_WSL_formulations}
{\renewcommand{\arraystretch}{1.2}
\begin{tabular}[t]{ |Sc|Sl| } 
    \hline
    WSL & Formulation
    \\ \hline
    CL &
        $\begin{aligned}
            & \{(\corr{s}_i, x_i)\}_{i=1}^{n}
            \stackrel{\text{i.i.d.}}{\sim} \prob{\corr{S}, X}
            := \frac{1}{K-1} \sum_{Y\neq \corr{S}} \prob{Y, X}.
        \end{aligned}$
        (\ref{eq:formulate_CL})
    \\ \hline
    MCL &
        $\begin{aligned}
            & \{(\corr{s}_i, x_i)\}_{i=1}^{n}
            \stackrel{\text{i.i.d.}}{\sim} \prob{\corr{S}, X}
            := 
            \begin{cases}
                \sum_{d=1}^{K-1} \prob{|\corr{S}| = d} \cdot \frac{1}{{K-1 \choose |\corr{S}|}} \sum_{Y\notin \corr{S}} \prob{Y,X}, & \text{if}\ |\corr{S}|=d, \\
                0, & \text{otherwise.}
            \end{cases}
        \end{aligned}$ 
        (\ref{eq:formulate_MCL})
    \\ \hline
    PCPL &
        $\begin{aligned}
            & \{(s_i, x_i)\}_{i=1}^{n}
            \stackrel{\text{i.i.d.}}{\sim} \prob{S, X}
            := \frac{1}{2^{K-1}-1} \sum_{Y\in S} \prob{Y, X}.
        \end{aligned}$ 
        (\ref{eq:formulate_PCPL})
    \\ \hline
    PPL &
        $\begin{aligned}
            & \{(s_i, x_i)\}_{i=1}^{n}
            \stackrel{\text{i.i.d.}}{\sim} \prob{S, X}
            := C(S,X) \sum_{Y\in S} \prob{Y, X}.
        \end{aligned}$ 
        (\ref{eq:formulate_PPL})
    \\ \hline
    SC-Conf &
        $\begin{aligned}
            & \left\{x_i, r_1(x_i), \ldots, r_K(x_i)\right\}_{i=1}^{n}, \text{ where}
            \\
            &\;\;\;\; x_i \stackrel{\text{i.i.d.}}{\sim} \prob{X|Y=y_\mrm{s}} \text{ with } y_\mrm{s} \in [K],
            \\
            &\;\;\;\; r_k(x_i) := \prob{Y=k|X=x_i} \text{ for each } k \in [K].
        \end{aligned}$
        \; (\ref{eq:formulate_SC-conf})
    \\ \hline
    Sub-Conf &
        $\begin{aligned}
            & \left\{x_i, r_1(x_i), \ldots, r_K(x_i)\right\}_{i=1}^{n}, \text{ where}
            \\
            &\;\;\;\; x_i \stackrel{\text{i.i.d.}}{\sim} \prob{X|Y \in \mathcal{Y}_\mrm{s}} \text{ with } \mathcal{Y}_\mrm{s} \subset [K],
            \\
            &\;\;\;\; r_k(x_i) := \prob{Y=k|X=x_i} \text{ for each } k \in [K].
        \end{aligned}$
        \; (\ref{eq:formulate_Sub-conf})
    \\ \hline
    Soft-label &
        $\begin{aligned}
            & \left\{x_i, r_1(x_i), \ldots, r_K(x_i)\right\}_{i=1}^{n}, \text{ where}
            \\ 
            &\;\;\;\; x_i \stackrel{\text{i.i.d.}}{\sim} \prob{X},
            \\
            &\;\;\;\; r_k(x_i) := \prob{Y=k|X=x_i} \text{ for each } k \in [K].
        \end{aligned}$
        \; (\ref{eq:formulate_Soft})
    \\ \hline
\end{tabular}
}   
\end{table}     

\begin{table}[H]
\centering
\caption{\label{tab:binary_WSL_rewrites} Risk rewrites for binary WSLs.}
{\renewcommand{\arraystretch}{1.2}
\begin{tabular}[t]{ |Sc|Sl| } 
    \hline
    WSL & Risk rewrite for $R(g) = \expt{Y, X}{\ell_{Y}(g(X))}$ (\ref{eq:erm_b0}) 
    \\ \hline
    PU &
        $\begin{aligned}
            R(g) 
            = 
            \expt{\mrm{P}}{
                \pi_\rmp \ell_\rmp - \pi_\rmp \ell_\rmn
            } + \expt{\mrm{U}}{
                \ell_\rmn
            }.
        \end{aligned}$ 
        \; (\ref{eq:review_rewrite_PU})
    \\ \hline
    Pconf &
        $\begin{aligned}
            R(g)
            =
            \pi_\rmp
            \expt{\mrm{P}}{
                \ell_\rmp + \frac{1-r(X)}{r(X)} \ell_\rmn
            }.
        \end{aligned}$ 
        \; (\ref{eq:review_rewrite_Pconf})
    \\ \hline
    UU &
        $\begin{aligned}
            R(g)
            =
            \expt{\mrm{U}_1}{
                \frac{(1-\mcdn)\pi_\rmp}{1-\mcdp-\mcdn} \ell_\rmp + \frac{-\mcdn \pi_\rmn}{1-\mcdp-\mcdn} \ell_\rmn
            } 
            + 
            \expt{\mrm{U}_2}{
                \frac{-\mcdp\pi_\rmp}{1-\mcdp-\mcdn} \ell_\rmp + \frac{(1-\mcdp)\pi_\rmn}{1-\mcdp-\mcdn} \ell_\rmn
            }. 
        \end{aligned}$ 
        \; (\ref{eq:review_rewrite_UU})
    \\ \hline
    SU &
        $\begin{aligned}
            & R(g)
            =
            \left(\pi_\rmp^2+\pi_\rmn^2\right)\expt{\mrm{S}}{\frac{\mathcal{L}(X) + \mathcal{L}(X')}{2}}
            +
            \expt{\mrm{U}}{\mathcal{L}_{-}(X)}, \text{ where}
            \\
            &\;\;\;\; \mathcal{L}(X) := \frac{1}{\pi_\rmp - \pi_\rmn} \ell_{\rmp}(X) - \frac{1}{\pi_\rmp - \pi_\rmn} \ell_{\rmn}(X),
            \\
            &\;\;\;\; \mathcal{L}_{-}(X) := - \frac{\pi_\rmn}{\pi_\rmp -\pi_\rmn} \ell_{\rmp}(X) + \frac{\pi_\rmp}{\pi_\rmp - \pi_\rmn} \ell_{\rmn}(X).
        \end{aligned}$ 
        \; (\ref{eq:review_rewrite_SU})
    \\ \hline
    DU &
        $\begin{aligned}
            & R(g)
            =
            2\pi_\rmp\pi_\rmn
            \expt{\mrm{D}}{-\frac{\mathcal{L}(X) + \mathcal{L}(X')}{2}}
            +
            \expt{\mrm{U}}{\mathcal{L}_{+}(X)}, \text{ where}
            \\
            &\;\;\;\; \mathcal{L}(X) \text{ is defined in the SU setting, and} 
            \\
            &\;\;\;\; \mathcal{L}_{+}(X) := \frac{\pi_\rmp}{\pi_\rmp -\pi_\rmn} \ell_{\rmp}(X) - \frac{\pi_\rmn}{\pi_\rmp - \pi_\rmn} \ell_{\rmn}(X).
        \end{aligned}$
        \; (\ref{eq:review_rewrite_DU})
    \\ \hline
    SD &
        $\begin{aligned}
            & R(g)
            =
            \left( \pi_\rmp^2 + \pi_\rmn^2 \right)
            \expt{\mrm{S}}{
                \frac{\mathcal{L}_{+}(X) + \mathcal{L}_{+}(X')}{2}
            } + 
            2\pi_\rmp\pi_\rmn
            \expt{\mrm{D}}{
                \frac{\mathcal{L}_{-}(X) + \mathcal{L}_{-}(X')}{2}
            }, \mrm{where}
            \\
            &\;\;\;\; \mathcal{L}_{+}(X) \text{ and } \mathcal{L}_{-}(X') \text{ are defined in the SU and DU settings.}
        \end{aligned}$
        \; (\ref{eq:review_rewrite_SD})
    \\ \hline
    Pcomp &
        $\begin{aligned}
            & R(g)
            =
            \expt{\mrm{Sup}}{\ell_\rmp - \pi_\rmp \ell_\rmn}
            +
            \expt{\mrm{Inf}}{-\pi_\rmn \ell_\rmp + \ell_\rmn}, \text{ where}
            \\
            &\;\;\;\; \prob{\mrm{Sup}} := \int_{x'\in\mathcal{X}} \prob{\mrm{PC}} \, \dx', 
            \\
            &\;\;\;\; \prob{\mrm{Inf}} := \int_{x\in\mathcal{X}} \prob{\mrm{PC}} \, \dx.
        \end{aligned}$ 
        \; (\ref{eq:review_rewrite_Pcomp})
    \\ \hline
    Sconf &
        $\begin{aligned}
            R(g)
            =
            \expt{X,X'}{
                \frac{r(X,X')-\pi_\rmn}{\pi_\rmp-\pi_\rmn} \frac{\ell_\rmp(X) + \ell_\rmp(X')}{2} 
                +
                \frac{\pi_\rmp-r(X,X')}{\pi_\rmp-\pi_\rmn} \frac{\ell_\rmn(X) + \ell_\rmn(X')}{2} 
            }.
        \end{aligned}$ 
        \; (\ref{eq:review_rewrite_Sconf})
    \\ \hline
\end{tabular}
}   
\end{table}     

\begin{table}[H]
\centering
\caption{\label{tab:multiclass_WSL_rewrites} Risk rewrites for multiclass WSLs.}
{\renewcommand{\arraystretch}{1.2}
\begin{tabular}[t]{ |Sc|Sl| } 
    \hline
    WSL & Risk rewrite for $R(g) = \expt{Y, X}{\ell_{Y}(g(X))}$ (\ref{eq:erm_b0})
    \\ \hline
    CL &
        $\begin{aligned}
            R(g)
            =
            \expt{\corr{S}, X}
            {
                \sum_{y=1}^{K} \ell_{y} - (K-1)\ell_{\corr{S}}
            }.
        \end{aligned}$ 
        \; (\ref{eq:review_rewrite_CL})
    \\ \hline
    MCL &
        $\begin{aligned}
            R(g)
            =
            \sum_{d=1}^{K-1} \prob{|\corr{S}|=d} \; \expt{\corr{S},X||\corr{S}|=d}{
                \sum_{y\notin\corr{S}} \ell_{y} - \frac{K-1-|\corr{S}|}{|\corr{S}|} \sum_{\corr{s} \in \corr{S}} \ell_{\corr{s}}
            }.
        \end{aligned}$ 
        \; (\ref{eq:review_rewrite_MCL})
    \\ \hline
    PCPL &
        $\begin{aligned}
            R(g)
            =
            \frac{1}{2} \expt{S,X}{\sum_{y=1}^{K} \frac{\prob{Y=y|X} }{\sum_{a \in S} \prob{Y=a|X}} \ell_{y}}.
        \end{aligned}$ 
        \; (\ref{eq:review_rewrite_PCPL})
    \\ \hline
    PPL &
        $\begin{aligned}
            R(g) 
            = 
            \expt{S,X}{\sum_{y\in S} \frac{\prob{Y=y|X}}{\sum_{a \in S} \prob{Y=a|X}} \ell_{y}}.
        \end{aligned}$ 
        \; (\ref{eq:review_rewrite_PPL})
    \\ \hline
    SC-Conf &
        $\begin{aligned}
            R(g)
            =
            \pi_{y_\mrm{s}} \expt{X|Y=y_\mrm{s}}{\sum_{y=1}^{K} \frac{r_{y}(X)}{r_{y_\mrm{s}}(X)} \ell_{y}}.
        \end{aligned}$ 
        \; (\ref{eq:review_rewrite_SC-conf})
    \\ \hline
    Sub-Conf &
        $\begin{aligned}
            R(g)
            =
            \pi_{\mathcal{Y}_\mrm{s}} \expt{X|Y\in\mathcal{Y}_\mrm{s}}{\sum_{y=1}^{K} \frac{r_{y}(X)}{r_{\mathcal{Y}_\mrm{s}}(X)} \ell_{y}}.
        \end{aligned}$ 
        \; (\ref{eq:review_rewrite_Sub-conf})
    \\ \hline
    Soft-label &
        $\begin{aligned}
            R(g)
            =
            \expt{X}{\sum_{y=1}^{K} r_y(X) \ell_{y}}.
        \end{aligned}$ 
        \; (\ref{eq:review_rewrite_Soft})
    \\ \hline
\end{tabular}
}   
\end{table}     

\explain From the above tables, finding a way to reexpress the classification risk $R(g)$ (\ref{eq:erm_b0}) in terms of the data-generating distributions becomes the crux when applying ERM for most WSL studies.
The rewrites also replace loss functions $\ell_{Y}$ defining (\ref{eq:erm_b0}) with various modified losses (shown inside the expectations).
These modified loss functions are sometimes called corrected losses, which is why the approach is also called loss correction.
\transit
Proposing a generic methodology that finds properly corrected losses to achieve risk rewrite in different scenarios is a main topic we would like to elaborate on in this paper.

\subsubsection{Learning with Noisy Labels (LNL) Formulations}
\label{sec:review_MCD_CCN}
Next, we review two related formulations in LNL, the MCD and CCN settings, in Table~\ref{tab:LNL_formulations}.
The observed instances in MCD and CCN are still labeled by $\{\rmp, \rmn\}$ but are polluted by certain noise models.
We use $\corr{Y}$ to represent a polluted label, compared to an unpolluted $Y$.
In MCD, a small portion of the negatively labeled data $\mcdpp \prob{X|Y=\rmn}$ contaminates the positively labeled data $\prob{X|Y=\rmp}$. Likewise, a small portion of the positive data $\mcdnn \prob{X|Y=\rmp}$ contaminates the negatively labeled data $\prob{X|Y=\rmn}$ \citep{mcd_13_Scott/Scott/BH/13}.
In the CCN setting, a label $Y$ is flipped to become $\corr{Y}$ with probability $\prob{\corr{Y}|Y, X}$ \citep{ccn_13/Natarajan/DRT/13}. 
\transit
Although they are formulated for the study of noisy labels, their formulations share similar structures with many WSLs above.
\connect
In Section~\ref{sec:matrixFormulations}, we will use the similarities to categorize WSLs and provide a bird's eye view to reveal connections among WSLs.

\begin{table}[H]
\centering
\caption{\label{tab:LNL_formulations} MCD and CCN formulations.}
{\renewcommand{\arraystretch}{1.2}
\begin{tabular}[t]{ |Sc|Sl| } 
    \hline
    Scenario & Formulation \\
    \hline 
    MCD 
    & 
    $\begin{aligned}
        &\left\{ x_i^{\corr{\rmp}} \right\}_{i=1}^{n_{\corr{\rmp}}} \stackrel{\text{i.i.d.}}{\sim} 
        \prob{X|\corr{Y}=\rmp} := (1-\mcdpp) \; \prob{X|Y=\rmp} + \mcdpp \; \prob{X|Y=\rmn}. \\
        &\left\{ x_j^{\corr{\rmn}} \right\}_{j=1}^{n_{\corr{\rmn}}} \stackrel{\text{i.i.d.}}{\sim}
        \prob{X|\corr{Y}=\rmn} := \mcdnn \; \prob{X|Y=\rmp} + (1-\mcdnn) \; \prob{X|Y=\rmn}.
    \end{aligned}$
    \\
    \hline
    CCN 
    & 
    $\begin{aligned}
        \left\{(\corr{y}_i, x_i)\right\}_{i=1}^{n}
        \stackrel{\text{i.i.d.}}{\sim}
        \prob{\corr{Y}=\corr{y}_i, X} := \sum_{k \in \{\rmp,\rmn\}} \prob{\corr{Y}=\corr{y}_i|Y=k, X} \prob{Y=k, X}, \forall \corr{y}_i \in \{\rmp,\rmn\}.
    \end{aligned}$
    \\
    \hline
\end{tabular}
}   
\end{table}


\section{A Framework for Risk Rewrite}
\label{sec:recipe}
We illustrate the proposed framework in this section. 
Its job is to provide a unified treatment and understanding of WSL.
It consists of a formulation component and an analysis component. The analysis component suggests a generic methodology to solve the risk rewrite problem. Moreover, diving into the formulation component's logic, we can interpret multiple WSL formulations and the diverse risk rewrites from a single perspective.

\subsection{The Formulation Component of the Framework}
\label{sec:framework_formulation}
The construction of the formulation component is to study the connections among WSLs and provide a foundation for developing the generic methodology.
We draw inspiration from Section~\ref{sec:formulations_old}.
Each WSL formulation represents a type of weaken information of the joint distribution $\prob{Y,X}$ in supervised learning.
For instance, unlabeled data discards the label information \citep{uu_19_Lu, uu_21_Lu}, the complementary-label is a label that cannot be the ground truth \citep{comp_17/Ishida/NHS/17, comp_17_Tao/Yu/LGT/18}, and the similarity encodes a comparative relationship of two ground truth labels \citep{su_18/Bao/NS/18, sdu_19/Shimada/BSS/21, sconf_21/Cao/FXANS/21}.
Thus, we are motivated to search for a general way to link data-generating distributions with the joint distribution.

Denote the data-generating distributions in a vector form $\corr{P}$.
Suppose there are basic elements in defining $\corr{P}$ and relevant to the labeling distributions.
We express them in a vector form $B$ and call them the base distributions\footnote{We reserve $P$, $B$, and $\corr{P}$ for vectors of distributions and $L$ and $\corr{L}$ for vectors of loss functions. We address them as ``the distributions'' and ``the losses'' to avoid the verbose ``the vector of distributions/losses.''}.
To connect $\corr{P}$ and $B$, we assume a matrix $\Mcorr{corr}$ formalizes the connection:
\eqarr{
    \corr{P} = \Mcorr{corr} B. \label{eq:recipe1}
}
Taking PU learning (\ref{eq:formulate_PU}) for example, $\Mcorr{corr}$ aims to connect $\corr{P} = \icol{\prob{\mrm{P}} \\ \prob{\mrm{U}}}$ with $B = \icol{\prob{X|Y=\rmp} \\ \prob{X|Y=\rmn}}$.
To keep the framework as abstract as possible, we would like to defer the definitions of all other $\corr{P}$ and $B$ until we realize their corresponding $\Mcorr{corr}$ in Section~\ref{sec:matrixFormulations}.

The matrix formulation has two advantages.
First, it provides a unified way to characterize a wide range of WSL settings. 
By studying the entries of a matrix, we can easily link one WSL scenario to another to form reduction graphs of WSLs. As the first main topic of this work, Section~\ref{sec:matrixFormulations} shows, for a given WSL setting, how to find the corresponding matrix $\Mcorr{corr}$, and Tables \ref{tab:MCD_matrices_summary} -- \ref{tab:conf_matrices_summary} summarize fifteen WSL settings covered by our matrix formulation and depict a reduction graph rooted from $\Mcorr{corr}$. 
The following subsection illustrates the second advantage of aiding the construction of a generic methodology for conducting risk rewrite.


\subsection{The Analysis Component of the Framework}
\label{sec:framework_rewrite}
Formulation (\ref{eq:recipe1}) serves as a stepping stone toward constructing the corrected losses needed for a risk rewrite. 
Denote $P$ as the vector of risk-defining distributions whose $k$-th entry is $\prob{Y=k, x}$ and $L$ as the loss vector whose $k$-th entry is $\ell_{Y=k}(g(x))$. 
Then, the conventional expression of $R(g)$ in (\ref{eq:erm_b0}) can be simplified, by the inner product, to be $\int_{x\in \mathcal{X}} L^{\top} P \dx$.
It is immediate to achieve classification rewrite if one shows $L^{\top} P = \corr{L}^{\top} \corr{P}$, with $\corr{L}$ being the vector form of the corrected losses.

The bridge of connection between $P$ and $\corr{P}$ are the base distributions $B$ we assumed in the previous subsection.
We have shown its connection to $\corr{P}$ via $\Mcorr{corr}$.
Here, we connect $B$ with $P$ by assuming a transform matrix $\Mcorr{trsf}$ satisfies $B = \Mcorr{trsf}P$, which embodies its labeling-relevant nature.
Thus, (\ref{eq:recipe1}) becomes
\eqarr{
    \corr{P} = \Mcorr{corr}\Mcorr{trsf}P. \label{eq:recipe7}
}
The logic of having $\Mcorr{trsf}$ is that people can choose different base distributions to formulate observed data and define various performance measures, and $\Mcorr{trsf}$ provides a flexibility to transform between them.

The reason why connecting $P$ with $\corr{P}$ (\ref{eq:recipe7}) helps the construction of the corrected losses is that if we manage to find a way to compensate for the combined effect of $\Mcorr{corr}$ and $\Mcorr{trsf}$, we can implement the compensation mechanism on the ``corrected'' losses $\corr{L}$.
Specifically, suppose there exists a matrix $\Mcorr{corr}^{\dagger}$ satisfying 
\eqarr{
    P = \Mcorr{corr}^{\dagger} \corr{P}. \label{eq:recipe6}
} 
Then, the {corrected losses} defined by
\eqarr{
    \corr{L}^{\top} := L^{\top} \Mcorr{corr}^{\dagger} \label{eq:recipe5}
}
allows us to rephrase the classification risk as 
\eqarr{
    \int_{x\in\mathcal{X}} \corr{L}^{\top} \corr{P} \, \dx 
    &=& \int_{x\in\mathcal{X}} L^{\top} \Mcorr{corr}^{\dagger} \corr{P} \, \dx \nonumber \\
    &=& \int_{x\in\mathcal{X}} L^{\top} P \, \dx \, = R(g), \label{eq:recipe2a}
}
providing a rewrite for $R(g)$ with respect to $\corr{P}$.

The above procedure describes a generic methodology for the risk rewrite problem. As the second main topic, we instantiate the framework by presenting the corresponding matrices $\Mcorr{corr}^{\dagger}$ and $\Mcorr{trsf}$ for each learning scenario in Section~\ref{sec:riskRewrite} to demonstrate its applicability.


\subsection{Intuition of the Framework}
\label{sec:framework_intuition}
The logic behind the key equations 
\begin{gather*}
    \corr{P} 
        = \Mcorr{corr} B 
        = \Mcorr{corr} \Mcorr{trsf} P, \\
    \corr{L}^{\top} \corr{P}
        = \corr{L}^{\top} \Mcorr{corr}\Mcorr{trsf} P
        = L^{\top} \Mcorr{corr}^{\dagger} \Mcorr{corr}\Mcorr{trsf} P
        = L^{\top} P
\end{gather*}
is succinct and interpretive.
Firstly, from the formulation perspective, 
viewing matrix $\Mcorr{corr}$ as a {contamination matrix} that corrupts the base $B$ to become the contaminated $\corr{P}$, we interpret this \emph{contamination mechanism} as sacrificing certain information in exchange for certain saved costs or privacy, reflecting the essence underlying WSL formulations.
Moreover, $B$ plays a pivotal role in developing the methodology. 
On the one hand, $B$ is a crucial factor in formulating the generation process of the observed data.
On the other hand, its link to the risk-defining distribution connects $\corr{P}$ and $P$ to motivate the design of the corrected losses $\corr{L}$. 
It is this novel viewpoint of connecting the data distributions via the explicit two-step formulation that facilitates the unification work in this paper.

Secondly, regarding the methodological design, 
it becomes easier to devise a countermeasure when the connection between $\corr{P}$ and $P$ is in good shape. 
Therefore, the realizations of $\corr{L}^{\top} = L^{\top} \Mcorr{corr}^{\dagger}$ justify that the seemly different forms of corrected losses
reported in the literature (i.e., referred papers that contribute to Tables~\ref{tab:binary_WSL_rewrites} and \ref{tab:multiclass_WSL_rewrites}, and those referred to as recoveries in Section~\ref{sec:riskRewrite})
are, in fact, determined by $\Mcorr{corr}^{\dagger}$ and can be traced back to one common idea: Restoring the {risk-defining} distributions and the original loss functions are done by the \emph{decontamination} provided by $\corr{L}$.
In summary, the proposed framework is abstract and flexible enough that we use it in the current paper to formulate the contamination mechanisms and provide a generic methodology for a wide range of WSLs. 

\subsection{Building Blocks: The Inversion and the Marginal Chain Approaches}
\label{sec:recipe_building_blocks}
We describe two building blocks, the inversion method and the marginal chain method, that will be used to devise $\Mcorr{corr}^{\dagger}$ that satisfies (\ref{eq:recipe6}) in each scenario we study later.

\begin{theorem}[The inversion method] 
\label{thm:inv_method}
Let $P$ and $\corr{P}$ be vectors.
Suppose $\corr{P} = M P$ holds for an invertible matrix $M$.
Then, choosing $\Mcorr{corr}^{\dagger} = M^{-1}$, we have $P = \Mcorr{corr}^{\dagger} \corr{P}$.
\end{theorem}
\begin{proof}
For any invertible $M$, it is easy to see that, by assigning $\Mcorr{corr}^{\dagger} = M^{-1}$, one has 
\eqarr{
    \Mcorr{corr}^{\dagger} \corr{P}
    = 
    M^{-1} \corr{P}
    = 
    M^{-1} M P
    =
    P. \nonumber
}
\end{proof}
We remark that this simple strategy was adopted in many LNL works. 
A handful of related papers are \cite{partial_12_Cid-Suerio/Cid-Suerio/12}, \cite{mcd_14_Scott/Blanchard/S/14}, \cite{mcd_15/Menon/ROW/15}, \cite{15_Rooyen/Rooyen/W/15}, \cite{17_Patrini/Patrini/RMNQ/17}, \cite{18_Rooyen/Rooyen/W/17}, and \cite{mcd_19_Scott/Katz-Samuels/BS/19}. 
Hence, it can be applied to WSLs that are special cases of certain LNL scenarios.

\begin{theorem}[The marginal chain method] \label{thm:marginal_chain}
Let $Y=k \in [K]$ be a class label, where $[K]$ is the set of classes associated with the classification risk. 
Let $\mathcal{S} = \{s_1, s_2, \ldots, s_{|\mathcal{S}|}\}$ be the set of classes of the observed data and $S$ be the random variable of an observed label.
Denote
$$
P 
=
\mmatrix{
    \prob{Y=1, X} \\
    \vdots \\
    \prob{Y=K, X}
}
\text{ and }
\corr{P}
=
\mmatrix{
    \prob{S=s_{1}, X} \\
    \vdots \\
    \prob{S=s_{|\mathcal{S}|}, X}
}.
$$
Then, 
\eqarr{
    M = 
    \mmatrix{
        \prob{S=s_{1}|Y=1,X} & \prob{S=s_{1}|Y=2,X} & \cdots & \prob{S=s_{1}|Y=K,X} \\
        \prob{S=s_{2}|Y=1,X} & \prob{S=s_{2}|Y=2,X} & \cdots & \prob{S=s_{2}|Y=K,X} \\
        \vdots & \vdots & \ddots & \vdots \\
        \prob{S=s_{|\mathcal{S}|}|Y=1,X} & \prob{S=s_{|\mathcal{S}|}|Y=2,X} & \cdots & \prob{S=s_{|\mathcal{S}|}|Y=K,X}
    } \label{eq:M_MargianlChain}
} 
satisfies $\corr{P} = M P$, and
\eqarr{
    \Mcorr{corr}^{\dagger}
    =
    \mmatrix{
        \prob{Y=1|S=s_{1},X} & \prob{Y=1|S=s_{2},X} & \cdots & \prob{Y=1|S=s_{|\mathcal{S}|},X} \\
        \prob{Y=2|S=s_{1},X} & \prob{Y=2|S=s_{2},X} & \cdots & \prob{Y=2|S=s_{|\mathcal{S}|},X} \\
        \vdots & \vdots & \ddots & \vdots \\
        \prob{Y=K|S=s_{1},X} & \prob{Y=K|S=s_{2},X} & \cdots & \prob{Y=K|S=s_{|\mathcal{S}|},X}
    } \label{eq:MGeneral_inv} \label{eq:MC_matrix}
}
satisfies $P = \Mcorr{corr}^{\dagger} \corr{P}$.
\end{theorem}
\begin{proof}
It suffices to show $\left( MP \right)_{j} = \corr{P}_{j}$ for any $j \in [|\mathcal{S}|]$.
Taking the inner product of the $j$-th row of $M$ and $P$, we have
\eqarr{
    \sum_{k=1}^{K} \prob{S=s_{j}|Y=k, X} \prob{Y=k, X} 
    = \sum_{k=1}^{K} \prob{S=s_{j}, Y=k, X}
    = \prob{S=s_{j}, X} \nonumber
}
that verifies (\ref{eq:M_MargianlChain}).

Next, we prove $P = \Mcorr{corr}^{\dagger} \corr{P}$ by showing $\left( \Mcorr{corr}^{\dagger}\corr{P} \right)_i = P_i$: For each $i \in [K]$,
\eqarr{
    \left(\Mcorr{corr}^{\dagger} \corr{P} \right)_i
    =
    \left(\Mcorr{corr}^{\dagger} M P \right)_i 
    &=&
    \sum_{j=1}^{|\mathcal{S}|} \prob{Y=i|S=s_{j},X} \sum_{k=1}^{K} \prob{S=s_{j}|Y=k,X} \prob{Y=k, X} \nonumber \\
    &\stackrel{\text{(a)}}{=}&
    \sum_{j=1}^{|\mathcal{S}|} \prob{Y=i|S=s_{j},X} \prob{S=s_{j}, X} \nonumber \\
    &\stackrel{\text{(b)}}{=}&
    \prob{Y=i, X} 
    =
    P_i. \label{eq:recipe_CCN} 
}
\end{proof}

\explain 
Besides finding the inverse matrix, we propose a new approach called the \emph{marginal chain} to achieve (\ref{eq:recipe6}).
The development of this approach begins with the observation that $\prob{S=s_j, X}$ in $\corr{P} = M P$ is a distribution where $Y$ is marginalized out.
It inspires an idea that one could perform another marginalization to restore the original distribution $\prob{Y, X}$; specifically, by marginalizing out $S$.
The design of $\Mcorr{corr}^{\dagger}$ in (\ref{eq:MC_matrix}) aims to carry out the idea.
As shown by (a) and (b) in the proof, two consecutive marginalization steps on $Y$ and then $S$ give the name of the marginal chain.

\compare
Both the inversion and marginal chain methods have strengths and weaknesses.
The inversion method only requires $P$ as a real vector but needs the invertible assumption on the contamination matrix $M$.
In contrast, the marginal chain method exploits that $P$, in fact, is a distributional vector, allowing it to find a decontamination matrix $\Mcorr{corr}^{\dagger}$ even for a non-invertible $M$.
A restriction of the marginal chain method is that the construction of $\Mcorr{corr}^{\dagger}$ is regulated by probability equations.


\connect
We are ready to justify the proposed framework through the following two sections.
Section~\ref{sec:matrixFormulations} discusses weakly supervised scenarios that can be subsumed by the formulation component (\ref{eq:recipe1}). 
Section~\ref{sec:riskRewrite} verifies the analysis component by instantiating (\ref{eq:recipe5}) to conduct the risk rewrite for each scenario mentioned in Section~\ref{sec:matrixFormulations}. 
In both sections, we divide the scenarios into three categories. 
The first two are WSLs that can be viewed as special cases in either the prevalent MCD or CCN settings. 
The third category contains confidence-based scenarios. 
The notations listed in Table~\ref{tab:small_notation_table} will still be functional.
For all notations and their abbreviations required in the coming sections, please refer to Appendix~\ref{sec:notations}.


\section{Contamination as Weak Supervision}
\label{sec:matrixFormulations}
In this section, we instantiate the contamination matrix for each weakly supervised scenario listed in Table~\ref{tab:binary_WSL_formulations} and Table~\ref{tab:multiclass_WSL_formulations}. Tables~\ref{tab:MCD_matrices_summary}  -- \ref{tab:conf_matrices_summary} summarize the contamination matrices developed in this section. Each table also represents a reduction graph of WSL settings. These reduction graphs cluster WSL settings into three main categories, providing a hierarchy of relationships. With this hierarchy, we can understand, compare with, and relate to different settings or even grow the hierarchy by adding new branches. Next are the notations for reading the graphs. For two contamination mechanisms, U and V, we use $\Mcorr{U} \rightarrow \Mcorr{V}$ to denote ``$\Mcorr{U}$ is reduced to $\Mcorr{V}$'' or ``$\Mcorr{U}$ is realized as $\Mcorr{V}$'', and $\Mcorr{U} \leadsto \Mcorr{V}$ means ``$\Mcorr{U}$ is generalized to $\Mcorr{V}$''.

\begin{table}[H]
\centering
\caption{\label{tab:MCD_matrices_summary} Contamination matrices of MCD category in Section~\ref{sec:formulations_mcds}.}
{\renewcommand{\arraystretch}{1.2}
\begin{tabular}[t]{ScSlScSl} 
    \hline
    WSLs & Entry Parameter & Contamination Matrix & Reduction path \\
    \hline
    MCD & $\mcdpp$, $\mcdnn$ & $\Mcorr{MCD}$ (\ref{eq:MMCD_2}) 
    & $\Mcorr{corr} \rightarrow \Mcorr{MCD}$ \\ 
    UU & $\mcdp$, $\mcdn$ & $\Mcorr{UU}$ (\ref{eq:MUU}) & $\Mcorr{corr} \rightarrow \Mcorr{UU} \approx \Mcorr{MCD}$ \\
    PU & $\mcdp = 0$, $\mcdn = \pi_\rmp$ & $\Mcorr{PU}$ (\ref{eq:MPU}) & $\Mcorr{UU} \rightarrow \Mcorr{PU}$ \\
    SU & $\mcdp = \frac{\pi_\rmn^2}{\pi_\rmp^2+\pi_\rmn^2}$, $\mcdn = \pi_\rmp$ & $\Mcorr{SU}$ (\ref{eq:MSU}) & $\Mcorr{UU} \rightarrow \Mcorr{SU}$ \\
    Pcomp & $\mcdp = \frac{\pi_\rmn^2}{\pi_\rmp + \pi_\rmn^2}$, $\mcdn = \frac{\pi_\rmp^2}{\pi_\rmp^2 + \pi_\rmn}$ & $\Mcorr{Pcomp}$ (\ref{eq:MPcomp}) & $\Mcorr{UU} \rightarrow \Mcorr{Pcomp}$ \\
    DU & $\mcdp = 1/2$, $\mcdn = \pi_\rmp$ & $\Mcorr{DU}$ (\ref{eq:MDU}) & $\Mcorr{UU} \rightarrow \Mcorr{DU}$ \\
    SD & $\mcdp = \frac{\pi_\rmn^2}{\pi_\rmp^2+\pi_\rmn^2}$, $\mcdn = 1/2$ & $\Mcorr{SD}$ (\ref{eq:MSD}) & $\Mcorr{UU} \rightarrow \Mcorr{SD}$ \\
    Sconf & & $\Mcorr{Sconf}$ (\ref{eq:MSconf}) & $\Mcorr{corr} \rightarrow \Mcorr{Sconf}$ \\
    \hline
\end{tabular}
}
\end{table}

\begin{table}[H]
\centering
\caption{\label{tab:CCN_matrices_summary} Contamination matrices of CCN category in Section~\ref{sec:formulations_ccns}.}
{\renewcommand{\arraystretch}{1.2}
\begin{tabular}[t]{ScSlScSl} 
    \hline
    WSLs & Entry Parameter & Contamination Matrix & Reduction path \\
    \hline
    CCN & $\prob{\corr{Y}|Y,X}$ (\ref{eq:ccn_mechanism}) & $\Mcorr{CCN}$ (\ref{eq:MCCN_2}) & $\Mcorr{corr} \rightarrow \Mcorr{CCN}$ \\
    Generalized CCN & $\prob{S|Y,X}$ (\ref{eq:P_S|YX_general_2}) & $\Mcorr{gCCN}$ (\ref{eq:MGeneral_2}) & $\Mcorr{corr} \rightarrow \Mcorr{CCN} \leadsto \Mcorr{gCCN}$ \\
    PPL & $C(S,X)\ind{Y \in S}$ (\ref{eq:P_S|YX_PPL}) & $\Mcorr{PPL}$ (\ref{eq:MPPL}) & $\Mcorr{gCCN} \rightarrow \Mcorr{PPL}$ \\
    PCPL & $\frac{1}{2^{K-1}-1} \ind{Y\in S}$ & $\Mcorr{PCPL}$ (\ref{eq:MPCPL}) & \makecell[l]{$\Mcorr{gCCN} \rightarrow \Mcorr{PPL} \rightarrow \Mcorr{PCPL}$} \\
    MCL & $\frac{q_{|\corr{S}|}}{{K-1 \choose |\corr{S}|}} \ind{Y\notin \corr{S}}$ (\ref{eq:PPL_to_MCL2}) & $\Mcorr{MCL}$ (\ref{eq:MMCL_detail}) & $\Mcorr{gCCN} \rightarrow \Mcorr{PPL} \rightarrow \Mcorr{MCL}$ \\
    CL & $|S|=1$, $\frac{1}{K-1}\ind{Y\in S}$ & $\Mcorr{CL}$ (\ref{eq:MCL}) & \makecell[l]{$\Mcorr{gCCN} \rightarrow \Mcorr{PPL} \rightarrow \Mcorr{MCL}$ \\ $\rightarrow \Mcorr{CL}$} \\
    \hline
\end{tabular}
}
\end{table}

\begin{table}[H]
\centering
\caption{\label{tab:conf_matrices_summary} Contamination matrices of confidence-based category in Section~\ref{sec:formulations_confs}.}
{\renewcommand{\arraystretch}{1.2}
\begin{tabular}[t]{ScSlScSl} 
    \hline
    WSLs & Entry Parameter & Contamination Matrix & Reduction path \\
    \hline
    Sub-Conf & $\frac{\prob{Y\in\mathcal{Y}_\mrm{s}|X}}{\prob{Y=k|X}}$ & $\Mcorr{Sub}$ (\ref{eq:MSub}) & $\Mcorr{corr} \rightarrow \Mcorr{Sub}$ \\
    SC & $\mathcal{Y}_\mrm{s} = \{ y_\mrm{s} \}$ in $\Mcorr{Sub}$ & $\Mcorr{SC}$ (\ref{eq:MSC}) & $\Mcorr{Sub} \rightarrow \Mcorr{SC}$ \\
    Pconf & \makecell[l]{$K=2$, $y_\mrm{s} = \rmp$ in $\Mcorr{SC}$} & $\Mcorr{Pconf}$ (\ref{eq:MPconf}) & $\Mcorr{Sub} \rightarrow \Mcorr{SC} \rightarrow \Mcorr{Pconf}$ \\
    Soft & $\frac{1}{\prob{Y=k|X}}$ & $\Mcorr{Soft}$ (\ref{eq:MSoft}) & $\Mcorr{Sub} \rightarrow \Mcorr{Soft}$ \\
    \hline
\end{tabular}
}
\end{table}

\subsection{MCD Scenarios}
\label{sec:formulations_mcds}
As listed in Table~\ref{tab:LNL_formulations}, in binary classification, the MCD model \citep{mcd_15/Menon/ROW/15} corrupts the clean class-conditionals $\prob{X|Y=\rmp}$ and $\prob{X|Y=\rmn}$ via parameters $\mcdpp$ and $\mcdnn$ as follows:
\eqarr{
    \begin{aligned}
    \label{eq:mcd2}
        &\prob{X|\corr{Y}=\rmp} 
        := (1-\mcdpp) \; \prob{X|Y=\rmp} + \mcdpp \; \prob{X|Y=\rmn},
        \\
        &\prob{X|\corr{Y}=\rmn} 
        := \mcdnn \; \prob{X|Y=\rmp} + (1-\mcdnn) \; \prob{X|Y=\rmn},
    \end{aligned}
}
where $\mcdpp, \mcdnn \in [0,1]$ and $\mcdpp + \mcdnn < 1$.
Viewing the contamination targets $\prob{X|Y=\rmp}$ and $\prob{X|Y=\rmn}$ as the base distributions 
\eqarr{
    B := 
    \mmatrix{
        \prob{X|Y=\rmp} \\
        \prob{X|Y=\rmn}
    } \label{eq:base_MCD} \nonumber
}
and denoting the vector of data-generating distributions as
\eqarr{
    \corr{P} := 
    \mmatrix{
        \prob{X|\corr{Y}=\rmp} \\
        \prob{X|\corr{Y}=\rmn}
    }, \nonumber
}
we can express (\ref{eq:mcd2}) in the following matrix form
\eqarr{
    \mmatrix{
        \prob{X|\corr{Y}=\rmp} \\
        \prob{X|\corr{Y}=\rmn}
    } 
    = 
    \mmatrix{
        1-\mcdpp & \mcdpp \\
        \mcdnn & 1-\mcdnn
    }
    \mmatrix{
        \prob{X|Y=\rmp} \\
        \prob{X|Y=\rmn}
    }. \label{eq:MMCD}
} 
Comparing (\ref{eq:MMCD}) with $\corr{P} = \Mcorr{corr} B$ (\ref{eq:recipe1}), we find that the contamination matrix $\Mcorr{corr}$ is realized as 
\eqarr{
    \Mcorr{MCD} := 
    \mmatrix{
        1-\mcdpp & \mcdpp \\
        \mcdnn & 1-\mcdnn
    } \label{eq:MMCD_2}
}
in the MCD setting.

\subsubsection{Unlabeled-Unlabeled (UU) Learning \texorpdfstring{\citep{uu_18/Lu/NMS/19}}{Lg}}
\label{sec:GE1_UU}
Next, we show how to characterize UU learning by a contamination matrix. 
Naming
$$\pi_\rmp \; \prob{X|Y=\rmp} + \pi_\rmn \; \prob{X|Y=\rmn}$$
as $\prob{\mrm{U}}$
is feasible since $\pi_\rmp \; \prob{X|Y=\rmp} + \pi_\rmn \; \prob{X|Y=\rmn} = \prob{X}$ generates data that statistically equals to data sampled from $\prob{Y,X}$ with labels removed.
Viewing $\pi_\rmp$ as the mixture rate of samples from $\prob{X|Y=\rmp}$ and $\prob{X|Y=\rmn}$, $\prob{\mrm{U}}$ is parameterized by $\pi_\rmp$.
Therefore, we can interpret (\ref{eq:formulate_UU}) as formulating two unlabeled data distributions w.r.t.\ mixture rates $(1-\mcdp)$ and $\mcdn$, respectively:
\eqarr{
    \prob{\mrm{U}_1} 
    &=& (1-\mcdp) \; \prob{X|Y=\rmp} + \mcdp \; \prob{X|Y=\rmn}, \nonumber \\
    \prob{\mrm{U}_2} 
    &=& \mcdn \; \prob{X|Y=\rmp} + (1-\mcdn) \; \prob{X|Y=\rmn}. \nonumber
}

Taking the class-conditionals as the base distributions
\eqarr{
    B := 
    \mmatrix{
        \prob{X|Y=\rmp} \\
        \prob{X|Y=\rmn}
    } \label{eq:base_dist_MCD}
}
and converting (\ref{eq:formulate_UU}) to the matrix form, we express the data-generating distributions of UU learning 
\eqarr{
    \corr{P} := 
    \mmatrix{
        \prob{\mrm{U}_1} \\
        \prob{\mrm{U}_2}
    } \nonumber
}
as
\eqarr{
    \mmatrix{
        \prob{\mrm{U}_1} \\
        \prob{\mrm{U}_2}
    } 
    =
    \mmatrix{
        1-\mcdp & \mcdp \\
        \mcdn & 1-\mcdn
    }
    \mmatrix{
        \prob{X|Y=\rmp} \\
        \prob{X|Y=\rmn}
    }, \label{eq:verify_MUU}
}
and we arrive at the following lemma.
\begin{lemma}
\label{lma:formulate_UU}
    Given the base distributions B (\ref{eq:base_dist_MCD}) and the parameters $\mcdp, \mcdn \in [0,1]$, the contamination matrix
    \eqarr{
        \Mcorr{UU} := 
        \mmatrix{
            1-\mcdp & \mcdp \\
            \mcdn & 1-\mcdn
        } \label{eq:MUU}
    }
    characterizes the data-generating process of UU learning (\ref{eq:formulate_UU}).
\end{lemma}
Comparing (\ref{eq:verify_MUU}) with the formulation framework $\corr{P} = \Mcorr{corr} B$ (\ref{eq:recipe1}), we see that in UU learning, $\Mcorr{corr}$ is realized as $\Mcorr{UU}$: 
\eqarr{
    \Mcorr{corr} 
    \rightarrow 
    \Mcorr{UU}. \nonumber 
}

Like MCD, we assume $\mcdp + \mcdn \neq 1$.
Our assumption is equivalent to that of MCD since the case of swapping $P_{\text{corr}}$ and $Q_{\text{corr}}$ in \cite{mcd_15/Menon/ROW/15} corresponds to $\mcdp + \mcdn > 1$ in our case. 
For details, refer to the discussion in Section 2.2 of \cite{mcd_15/Menon/ROW/15}.
The need for $\mcdp + \mcdn \neq 1$ can be explained by examining the entries in $\Mcorr{UU}$.
The constraint $\mcdp + \mcdn \neq 1$ guarantees distinct rows in $\Mcorr{UU}$, implying the observed data sets are sampled from two distinct distributions.
On the contrary, allowing $\mcdp + \mcdn = 1$ ends up observing one unlabeled data set (i.e., $\prob{\mrm{U}_1} = \prob{\mrm{U}_2}$) since $1 - \mcdp = \mcdn$.
\cite{uu_18/Lu/NMS/19} proved in Section 3 that it is impossible to conduct a risk rewrite if one only observes one unlabeled data set.

\compare
Assigning $\mcdp = \mcdpp$ and $\mcdn = \mcdnn$ implies that MCD and UU have essentially the same data-generating process from the contamination perspective, as (\ref{eq:MMCD}) and (\ref{eq:verify_MUU}) have the identical right-hand sides (i.e., the same contamination targets and the same contamination matrix).
However, they bear different meanings in respective research topics (i.e., distinct notions on the left-hand sides of the equations): In MCD, one still observes data with labels, nonetheless noisy, while in the UU setting, one observes two distinct unlabeled data sets. 
We use ``$\approx$'' to denote their relation in the UU row of Table~\ref{tab:MCD_matrices_summary}. 

\explain
Connecting UU learning with MCD, and later the generalized CCN with CCN in Section~\ref{sec:GE1_gCCN}, allows us to categorize WSLs from the LNL perspective into Sections \ref{sec:formulations_mcds} and \ref{sec:formulations_ccns}.
In the rest of this subsection, we collect WSLs whose base distributions are class-conditionals and show $\Mcorr{UU}$ instantiates their formulations via respective assignments of $\mcdp$ and $\mcdn$.

\subsubsection{Positive-Unlabeled (PU) Learning \texorpdfstring{\citep{pu_17/Kiryo/NPS/17}}{Lg}}
\label{sec:GE1_PU}
The following lemma describes the contamination matrix of PU learning. 
\begin{lemma}
\label{lma:formulate_PU}
    Given the base distributions $B$ (\ref{eq:base_dist_MCD}), the contamination matrix
    \eqarr{
        \Mcorr{PU} :=
        \mmatrix{
            1 & 0 \\
            \pi_\rmp & \pi_\rmn
        } \label{eq:MPU}
    }
    characterizes the data-generating distributions of PU learning (\ref{eq:formulate_PU}) denoted by
    \eqarr{
        \corr{P} :=
        \mmatrix{
            \prob{\mrm{P}} \\
            \prob{\mrm{U}}
        }. \nonumber
    }
\end{lemma}
\begin{proof}
By definitions,
\eqarr{
    \Mcorr{PU} B
    =
    \mmatrix{
        1 & 0 \\
        \pi_\rmp & \pi_\rmn
    }
    \mmatrix{
        \prob{X|Y=\rmp} \\
        \prob{X|Y=\rmn}
    }
    =
    \mmatrix{
        1 \cdot \prob{X|Y=\rmp} + 0 \cdot \prob{X|Y=\rmn} \\
        \pi_\rmp \cdot \prob{X|Y=\rmp} + \pi_\rmn \cdot \prob{X|Y=\rmn}
    }, \nonumber
}
where 
\eqarr{
    1 \cdot \prob{X|Y=\rmp} + 0 \cdot \prob{X|Y=\rmn} = \prob{X|Y=\rmp} = \prob{\mrm{P}}, \nonumber \\
    \pi_{\rmp} \cdot \prob{X|Y=\rmp} + \pi_{\rmn} \cdot \prob{X|Y=\rmn} = \prob{X} = \prob{\mrm{U}}, \nonumber
}
corresponding to the PU formulation (\ref{eq:formulate_PU}).
Therefore, $\Mcorr{PU}$ is the contamination matrix instantiating the formulation (\ref{eq:recipe1}) to be $\corr{P} = \Mcorr{PU} B$ for PU learning.
\end{proof}
Further, $\Mcorr{PU}$ can be obtained by assigning $\mcdp=0$ and $\mcdn=\pi_\rmp$ in $\Mcorr{UU}$ (\ref{eq:MUU}), and hence, we obtain the reduction path
\eqarr{
    \Mcorr{corr} 
    \rightarrow \Mcorr{UU}  
    \rightarrow \Mcorr{PU}. \nonumber 
}

\subsubsection{Similar-Unlabeled (SU) Learning \texorpdfstring{\citep{su_18/Bao/NS/18}}{Lg}}
\label{sec:GE1_SU}
Recall $\prob{\mrm{S}}$ (\ref{eq:formulate_SU}) is the distribution generating the pair of similar data $(x, x')$. 
We use $(x, x')$ instead of $(x^{\rms}, x^{\rms'})$ (\ref{eq:formulate_SU}) since we are focusing on the matrix formulation of SU and do not need to consider other WSLs in this sub-subsection. In the rest of the paper, for clarity, we will drop the superscripts when the content is explicit.
Let us put back the random variables and represent $\prob{\mrm{S}}$ as $\prob{\mrm{S}}^{(x,x')}$ for clarity. 
Denote 
$
    \prob{\tilde{\mrm{S}}}^{(x)}
    :=
    \int_{x'} \prob{\mrm{S}}^{(x,x')} \dx'
    =
    \frac{\pi_\rmp^2 \prob{x|Y=\rmp}+ \pi_\rmn^2 \prob{x|Y=\rmn}}{\pi_\rmp^2+\pi_\rmn^2}
$
as the marginal distribution of $\prob{\mrm{S}}^{(x,x')}$.
Since the equality $\int_{x'} \prob{\mrm{S}}^{(x,x')} \dx' = \int_{x} \prob{\mrm{S}}^{(x,x')} \dx$ implies $\prob{\tilde{\mrm{S}}}^{(x)} = \prob{\tilde{\mrm{S}}}^{(x')}$, we formulate 
$
\corr{P} =
\mmatrix{
    \prob{\tilde{\mrm{S}}} \\
    \prob{\mrm{U}}
}
$
instead of 
$
\mmatrix{
    \prob{\mrm{S}} \\
    \prob{\mrm{U}}
}.
$
\begin{lemma}
\label{lma:formulate_SU}
    Given the base distributions $B$ (\ref{eq:base_dist_MCD}),
    \eqarr{
        \Mcorr{SU} 
        :=
        \mmatrix{
            \frac{\pi_\rmp^2}{\pi_\rmp^2+\pi_\rmn^2} & \frac{\pi_\rmn^2}{\pi_\rmp^2+\pi_\rmn^2} \\
            \pi_\rmp & \pi_\rmn
        } \label{eq:MSU}
    }
    is the contamination matrix characterizing the data-generating distributions
    \eqarr{
        \corr{P} :=
        \mmatrix{
            \prob{\tilde{\mrm{S}}} \\
            \prob{\mrm{U}}
        }. \nonumber
    }
\end{lemma}
\begin{proof}
Recall that in (\ref{eq:formulate_SU}),
\eqarr{
    \prob{\mrm{S}}^{(x,x')} 
    = \frac{\pi_\rmp^2 \prob{x|Y=\rmp} \prob{x'|Y=\rmp} + \pi_\rmn^2 \prob{x|Y=\rmn} \prob{x'|Y=\rmn}}{\pi_\rmp^2+\pi_\rmn^2}. \nonumber
}
Thus,
\eqarr{
    \prob{\tilde{\mrm{S}}}^{(x)}
    &=&
    \int_{x'} \prob{\mrm{S}}^{(x,x')} \; \dx'
    =
    \int_{x'} \frac{\pi_\rmp^2 \prob{x|Y=\rmp} \prob{x'|Y=\rmp} + \pi_\rmn^2 \prob{x|Y=\rmn} \prob{x'|Y=\rmn}}{\pi_\rmp^2+\pi_\rmn^2} \; \dx' \nonumber \\
    &=&
    \frac{\pi_\rmp^2}{\pi_\rmp^2+\pi_\rmn^2} \prob{x|Y=\rmp} + \frac{\pi_\rmn^2}{\pi_\rmp^2+\pi_\rmn^2} \prob{x|Y=\rmp} \label{eq:pointwise_similar}
}
Then, combining with $\prob{\mrm{U}}$ in (\ref{eq:formulate_SU}), the following equality
\eqarr{
    \mmatrix{
        \prob{\tilde{\mrm{S}}} \\
        \prob{\mrm{U}}
    } 
    =
    \mmatrix{
        \frac{\pi_\rmp^2}{\pi_\rmp^2+\pi_\rmn^2} & \frac{\pi_\rmn^2}{\pi_\rmp^2+\pi_\rmn^2} \\
        \pi_\rmp & \pi_\rmn
    }
    \mmatrix{
        \prob{X|Y=\rmp} \\
        \prob{X|Y=\rmn}
    } \nonumber
}
proves the lemma.
\end{proof}
Further, $\Mcorr{SU}$ can be obtained by assigning $\mcdp = \frac{\pi_\rmn^2}{\pi_\rmp^2+\pi_\rmn^2}$ and $\mcdn = \pi_\rmp$ in $\Mcorr{UU}$ (\ref{eq:MUU}), and hence, we obtain the reduction path
\eqarr{
    \Mcorr{corr} 
    \rightarrow \Mcorr{UU}  
    \rightarrow \Mcorr{SU}. \nonumber 
}


\subsubsection{Pairwise Comparison (Pcomp) Learning \texorpdfstring{\citep{pcomp_20/Feng/SLS/21}}{Lg}} 
\label{sec:Pcomp}
In SU learning, we formulate the pointwise data-generating distributions $\prob{\tilde{\mrm{S}}}$ and $\prob{\mrm{U}}$; likewise, the pointwise distributions 
\eqarr{
    \prob{\mrm{Sup}} &:=& \int_{x'\in\mathcal{X}} \prob{\mrm{PC}} \, \dx = \frac{\pi_\rmp \prob{X|Y=\rmp} + \pi_\rmn^2 \prob{X|Y=\rmn}}{\pi_\rmp + \pi_\rmn^2}, \nonumber \\
    \prob{\mrm{Inf}} &:=& \int_{x\in\mathcal{X}} \prob{\mrm{PC}} \, \dx = \frac{\pi_\rmp^2 \prob{X'|Y=\rmp} + \pi_\rmn \prob{X'|Y=\rmn}}{\pi_\rmp^2 + \pi_\rmn} \nonumber
}
that we use to formulate Pcomp learning are marginal distributions of $\prob{\mrm{PC}}$ (\ref{eq:formulate_Pcomp}).
\begin{lemma}
\label{lma:formulate_Pcomp}
    Given the base distributions $B$ (\ref{eq:base_dist_MCD}),
    \eqarr{
        \Mcorr{Pcomp} :=
        \mmatrix{
            \frac{\pi_\rmp}{\pi_\rmp + \pi_\rmn^2} & \frac{\pi_\rmn^2}{\pi_\rmp + \pi_\rmn^2} \\
            \frac{\pi_\rmp^2}{\pi_\rmp^2 + \pi_\rmn} & \frac{\pi_\rmn}{\pi_\rmp^2 + \pi_\rmn}
        } \label{eq:MPcomp}
    }
    is the contamination matrix characterizing the data-generating distributions
    \eqarr{
        \corr{P}
        :=
        \mmatrix{
            \prob{\mrm{Sup}} \\
            \prob{\mrm{Inf}}
        }. \nonumber
    }
\end{lemma}
\begin{proof}
The equality below
\eqarr{
    \mmatrix{
        \prob{\mrm{Sup}} \\
        \prob{\mrm{Inf}}
    }
    =
    \mmatrix{
        \frac{\pi_\rmp}{\pi_\rmp + \pi_\rmn^2} & \frac{\pi_\rmn^2}{\pi_\rmp + \pi_\rmn^2} \\
        \frac{\pi_\rmp^2}{\pi_\rmp^2 + \pi_\rmn} & \frac{\pi_\rmn}{\pi_\rmp^2 + \pi_\rmn}
    }
    \mmatrix{
        \prob{X|Y=\rmp} \\
        \prob{X|Y=\rmn}
    } \nonumber
}
proves the lemma.
\end{proof}
Further, $\Mcorr{Pcomp}$ can be obtained by assigning $\mcdp = \frac{\pi_\rmn^2}{\pi_\rmp + \pi_\rmn^2}$ and $\mcdn = \frac{\pi_\rmp^2}{\pi_\rmp^2 + \pi_\rmn}$ in $\Mcorr{UU}$ (\ref{eq:MUU}), and hence, we obtain the reduction path
\eqarr{
    \Mcorr{corr} 
    \rightarrow \Mcorr{UU}  
    \rightarrow \Mcorr{Pcomp}. \nonumber 
}

\subsubsection{Similar-dissimilar-unlabeled (SDU) Learning \texorpdfstring{\citep{sdu_19/Shimada/BSS/21}}{Lg}}
\label{sec:GE1_SDU}
Dissimilar-unlabeled (DU) learning and similar-dissimilar (SD) learning are two critical components of SDU learning. 
Hence, we present the matrix formulations of $\Mcorr{DU}$ and $\Mcorr{SD}$.
We have explained the reason of formulating $\prob{\tilde{\mrm{S}}}^{(x)}$ in Section~\ref{sec:GE1_SU}. 
Similarly, given the pairwise dissimilar distribution $\prob{\mrm{D}}$ (\ref{eq:formulate_DU}), we have $\int_{x} \prob{\mrm{D}}^{(x,x')} \dx = \int_{x'} \prob{\mrm{D}}^{(x,x')} \dx'$ implying $\prob{\tilde{\mrm{D}}}^{(x')} = \prob{\tilde{\mrm{D}}}^{(x)}$, where $\prob{\tilde{\mrm{D}}}^{(x)} := \int_{x'} \prob{\mrm{D}}^{(x,x')} \dx'$ and $\prob{\tilde{\mrm{D}}}^{(x')} := \int_{x} \prob{\mrm{D}}^{(x,x')} \dx$.
Therefore, we also formulate the pointwise distribution $\prob{\tilde{\mrm{D}}}$ in DU and SD learning.

We formulate the contamination matrix of DU learning via the following lemma.
\begin{lemma}
\label{lma:formulate_DU}
    Given the base distributions $B$ (\ref{eq:base_dist_MCD}),
    \eqarr{
        \Mcorr{DU} =
        \mmatrix{
            1/2 & 1/2 \\
            \pi_\rmp & \pi_\rmn
        } \label{eq:MDU}
    }
    is the contamination matrix characterizing the data-generating distributions
    \eqarr{
        \corr{P} =
        \mmatrix{
            \prob{\tilde{\mrm{D}}} \\
            \prob{\mrm{U}}
        }. \nonumber
    }
\end{lemma}
\begin{proof}
    Recall that in (\ref{eq:formulate_DU}), 
    \eqarr{
        \prob{\mrm{D}}^{(x,x')} 
        = \frac{\prob{x|Y=\rmp} \prob{x'|Y=\rmn} + \prob{x|Y=\rmn} \prob{x'|Y=\rmp}}{2}. \nonumber
    }
    Thus, 
    \eqarr{
        \prob{\tilde{\mrm{D}}}^{(x)}
        &=&
        \int_{x'} \prob{\mrm{D}}^{(x,x')} \; \dx'
        =
        \int_{x'} \frac{\prob{x|Y=\rmp} \prob{x'|Y=\rmn} + \prob{x|Y=\rmn} \prob{x'|Y=\rmp}}{2} \; \dx' \nonumber \\
        &=&
        \frac{1}{2}\prob{x|Y=\rmp} + \frac{1}{2}\prob{x|Y=\rmn}. \label{eq:pointwise_dissimilar}
    }
    Then, combining with $\prob{\mrm{U}}$ in (\ref{eq:formulate_SU}), the following equality
    \eqarr{
        \mmatrix{
            \prob{\tilde{\mrm{D}}} \\
            \prob{\mrm{U}}
        }
        =
        \mmatrix{
            1/2 & 1/2 \\
            \pi_\rmp & \pi_\rmn
        }
        \mmatrix{
            \prob{X|Y=\rmp} \\
            \prob{X|Y=\rmn}
        } \nonumber
    }
    proves the lemma.
\end{proof}
Furthermore, since $\Mcorr{UU}$ (\ref{eq:MUU}) reduces to $\Mcorr{DU}$ by assigning $\mcdp = 1/2$ and $\mcdn = \pi_\rmp$, we have the reduction path
\eqarr{
    \Mcorr{corr} \rightarrow \Mcorr{UU} \rightarrow \Mcorr{DU}. \nonumber 
}

The next lemma formulates the contamination matrix of SD learning.
\begin{lemma}
\label{lma:formulate_SD}
    Given the base distributions $B$ (\ref{eq:base_dist_MCD}),
    \eqarr{
        \Mcorr{SD} =
        \mmatrix{
            \frac{\pi_\rmp^2}{\pi_\rmp^2+\pi_\rmn^2} & \frac{\pi_\rmn^2}{\pi_\rmp^2+\pi_\rmn^2} \\
            1/2 & 1/2
        } \label{eq:MSD}
    }
    is the contamination matrix characterizing the data-generating distributions
    \eqarr{
        \corr{P} =
        \mmatrix{
            \prob{\tilde{\mrm{S}}} \\
            \prob{\tilde{\mrm{D}}}
        }. \nonumber
    }
\end{lemma}
\begin{proof}
Combining (\ref{eq:pointwise_dissimilar}) with (\ref{eq:pointwise_similar}), we establish the lemma by the following equality.
\eqarr{
    \mmatrix{
        \prob{\tilde{\mrm{S}}} \\
        \prob{\tilde{\mrm{D}}}
    }
    =
    \mmatrix{
        \frac{\pi_\rmp^2}{\pi_\rmp^2+\pi_\rmn^2} & \frac{\pi_\rmn^2}{\pi_\rmp^2+\pi_\rmn^2} \\
        1/2 & 1/2
    }
    \mmatrix{
        \prob{X|Y=\rmp} \\
        \prob{X|Y=\rmn}
    }. \nonumber 
}
\end{proof}
Moreover, because $\Mcorr{UU}$ (\ref{eq:MUU}) reduces to $\Mcorr{SD}$ via $\mcdp = \frac{\pi_\rmn^2}{\pi_\rmp^2+\pi_\rmn^2}$ and $\mcdn = 1/2$, we obtain the reduction path
\eqarr{
    \Mcorr{corr} \rightarrow \Mcorr{UU} \rightarrow \Mcorr{SD}. \nonumber
}

\subsubsection{Similarity-Confidence (Sconf) Learning  \texorpdfstring{\citep{sconf_21/Cao/FXANS/21}}{Lg}}
\label{sec:GE1_Sconf}
Recall from the Sconf setting (\ref{eq:formulate_Sconf}) that $(x,x')$ is a pair of data sampled i.i.d.\ from $\prob{X,X'} := \prob{X}\prob{X'}$.
On seeing $\prob{X}$, one might wonder if it is sufficient to express the data-generating distribution simply as $\prob{X} = \prob{Y=\rmp, X} + \prob{Y=\rmn, X}$.
This approach, however correct, does not consider all available information in the Sconf setting.
Similar to $\Mcorr{UU}$ that uses parameters $\mcdp$ and $\mcdn$ to characterize the data-generating process in UU learning, we use the following lemma that includes the confidence $r(x,x') := \prob{y=y'|x, x'}$ to characterize Sconf learning.
Let us simplify $r(X, X')$ as $r$, $\prob{X|Y=\rmp}$ as $\prob{X|\rmp}$, and adopt the same abbreviations for $X'$ and $Y=\rmn$.
\begin{lemma}
\label{lma:MSconf}
    Assume $\pi_\rmp \neq 1/2$.
    Given the base distribution $B$ (\ref{eq:base_dist_MCD}), 
    \eqarr{
        \Mcorr{Sconf}
        :=
        \mmatrix{
            \frac{\pi_\rmp \left( \pi_\rmp^2\prob{X'|\rmp} - \pi_\rmn^2\prob{X'|\rmn} \right)}{r-\pi_\rmn} & \frac{\pi_\rmp \left( \pi_\rmn^2\prob{X'|\rmn} - \pi_\rmn^2\prob{X'|\rmp} \right)}{r-\pi_\rmn} \\
            \frac{\pi_\rmn \left( \pi_\rmp^2\prob{X'|\rmn} - \pi_\rmp^2\prob{X'|\rmp} \right)}{\pi_\rmp - r} & \frac{\pi_\rmn \left( \pi_\rmp^2\prob{X'|\rmp} - \pi_\rmn^2\prob{X'|\rmn} \right)}{\pi_\rmp - r}
        } \label{eq:MSconf}
    }
    characterizes the data-generating distributions
    \eqarr{
        \corr{P}
        :=
        \mmatrix{
            \prob{X} \prob{X'} \\
            \prob{X} \prob{X'}
        }. \nonumber
    }
\end{lemma}
\begin{proof}
    We prove the lemma by showing 
    \eqarr{
        \mmatrix{
            \prob{X}\prob{X'} \\
            \prob{X}\prob{X'}
        }
        =
        \mmatrix{
            \frac{\pi_\rmp \left( \pi_\rmp^2\prob{X'|\rmp} - \pi_\rmn^2\prob{X'|\rmn} \right)}{r-\pi_\rmn} & \frac{\pi_\rmp \left( \pi_\rmn^2\prob{X'|\rmn} - \pi_\rmn^2\prob{X'|\rmp} \right)}{r-\pi_\rmn} \\
            \frac{\pi_\rmn \left( \pi_\rmp^2\prob{X'|\rmn} - \pi_\rmp^2\prob{X'|\rmp} \right)}{\pi_\rmp - r} & \frac{\pi_\rmn \left( \pi_\rmp^2\prob{X'|\rmp} - \pi_\rmn^2\prob{X'|\rmn} \right)}{\pi_\rmp - r}
        }
        \mmatrix{
            \prob{X|\rmp} \\
            \prob{X|\rmn}
        } \label{eq:Sconf_corruption_mechanism}
    }
    is a realization of $\corr{P} = \Mcorr{corr} B$ (\ref{eq:recipe1}) since it justifies the contamination matrix $\Mcorr{Sconf}$.
    Note that once obtaining
    \eqarr{
        \left( \frac{r-\pi_\rmn}{\pi_\rmp} \right) \prob{X} \prob{X'}
        = 
        \left( \pi_\rmp^2\prob{X'|\rmp} - \pi_\rmn^2\prob{X'|\rmn} \right) \prob{X|\rmp} + \left( \pi_\rmn^2\prob{X'|\rmn} - \pi_\rmn^2\prob{X'|\rmp} \right) \prob{X|\rmn} \label{eq:Pconf_corruption1}
    }
    and
    \eqarr{
        \left( \frac{\pi_\rmp - r}{\pi_\rmn} \right) \prob{X} \prob{X'}
        =  
        \left( \pi_\rmp^2\prob{X'|\rmn} - \pi_\rmp^2\prob{X'|\rmp} \right) \prob{X|\rmp} + \left( \pi_\rmp^2\prob{X'|\rmp} - \pi_\rmn^2\prob{X'|\rmn} \right) \prob{X|\rmn}, \label{eq:Pconf_corruption2}
    }
    (\ref{eq:Sconf_corruption_mechanism}) is a direct implication via reorganizing equalities.
    
    According to (2) of \cite{sconf_21/Cao/FXANS/21}, the confidence $r(X,X')$, measuring how likely $X$ and $X'$ share the same label, is shown to be 
    \eqarr{
        r = r(X,X') = \frac{\pi_\rmp^2\prob{X|\rmp}\prob{X'|\rmp} + \pi_\rmn^2\prob{X|\rmn}\prob{X'|\rmn}}{\prob{X} \prob{X'}}. \nonumber
    }
    It implies 
    \eqarr{
        r \prob{X} \prob{X'}
        =
        \pi_\rmp^2\prob{X|\rmp}\prob{X'|\rmp} + \pi_\rmn^2\prob{X|\rmn}\prob{X'|\rmn} \nonumber
    }
    and 
    \eqarr{
        (1-r) \prob{X} \prob{X'}
        =
        \pi_\rmp \pi_\rmn \left( \prob{X|\rmp}\prob{X'|\rmn} + \prob{X|\rmn}\prob{X'|\rmp} \right). \nonumber
    }
    If $\pi_\rmp \neq 1/2$, $\pi_\rmp-r \neq 0$ and $r-\pi_\rmn \neq 0$.
    As a result, (\ref{eq:Pconf_corruption1}) is achieved as follows
    \eqarr{
        \left( \frac{r-\pi_\rmn}{\pi_\rmp} \right) \prob{X} \prob{X'}
        &=&
        \left( r- \frac{\pi_\rmn}{\pi_\rmp} (1-r) \right) \prob{X} \prob{X'} \nonumber \\
        &=&
        \pi_\rmp^2\prob{X|\rmp}\prob{X'|\rmp} + \pi_\rmn^2\prob{X|\rmn}\prob{X'|\rmn} 
        - \frac{\pi_\rmn}{\pi_\rmp} \pi_\rmp \pi_\rmn \left( \prob{X|\rmp}\prob{X'|\rmn} + \prob{X|\rmn}\prob{X'|\rmp} \right) \nonumber \\
        &=&
        \pi_\rmp^2\prob{X|\rmp}\prob{X'|\rmp} - \pi_\rmn^2\prob{X|\rmp}\prob{X'|\rmn} + \pi_\rmn^2\prob{X|\rmn}\prob{X'|\rmn} - \pi_\rmn^2\prob{X|\rmn}\prob{X'|\rmp}. \nonumber
    }
    Also, (\ref{eq:Pconf_corruption2}) is achieved by having
    \eqarr{
        \left( \frac{\pi_\rmp - r}{\pi_\rmn} \right) \prob{X} \prob{X'}
        &=&
        \left( \frac{\pi_\rmp}{\pi_\rmn}(1-r) -r \right) \prob{X} \prob{X'} \nonumber \\
        &=&
        \frac{\pi_\rmp}{\pi_\rmn} \pi_\rmp \pi_\rmn \left( \prob{X|\rmp}\prob{X'|\rmn} + \prob{X|\rmn}\prob{X'|\rmp} \right)
        - \pi_\rmp^2\prob{X|\rmp}\prob{X'|\rmp} - \pi_\rmn^2\prob{X|\rmn}\prob{X'|\rmn} \nonumber \\
        &=&
       \pi_\rmp^2\prob{X|\rmp}\prob{X'|\rmn} - \pi_\rmp^2\prob{X|\rmp}\prob{X'|\rmp} + \pi_\rmp^2\prob{X|\rmn}\prob{X'|\rmp} - \pi_\rmn^2\prob{X|\rmn}\prob{X'|\rmn}. \nonumber
    }
\end{proof}
The equality $\corr{P} = \Mcorr{Sconf} B$ (\ref{eq:Sconf_corruption_mechanism}) implies that the inner product of the first row (resp. the second row) of $\Mcorr{Sconf}$ and $B$ represents a way (resp. another way) of obtaining $\prob{X}\prob{X'}$.
Although one might suspect that it is redundant to formulate $\prob{X}\prob{X'}$ twice, we show in Section~\ref{sec:GE2_Sconf} this expression is crucial to rewrite the classification risk via the proposed framework.
Furthermore, comparing $\corr{P} = \Mcorr{Sconf} B$ (\ref{eq:Sconf_corruption_mechanism}) with $\corr{P} = \Mcorr{corr} B$ (\ref{eq:recipe1}), we have the reduction path
$$
    \Mcorr{corr} \rightarrow \Mcorr{Sconf}.
$$
Note that $\Mcorr{Sconf}$ does not fit the intuition of mutual contamination perfectly; we list Sconf learning in this subsection as all settings share the same base distributions.

\subsection{CCN Scenarios}
\label{sec:formulations_ccns}
The formulation component (\ref{eq:recipe1}) also applies to the CCN model. Unlike MCD contaminating class-conditionals (distributions of $X$), CCN corrupts class probability functions (labeling distributions). Next, we show how to formulate CCN via (\ref{eq:recipe1}) and extend the formulation to characterize diverse weakly supervised settings.

In binary classification, CCN \citep{ccn_13/Natarajan/DRT/13, ccn_18/Natarajan/DRT/17} corrupts the labels by flipping the positive (resp. negative) labels with probability $\prob{\corr{Y}=\rmn|Y=\rmp,X}$ (resp. $\prob{\corr{Y}=\rmp|Y=\rmn,X}$).
Thus,
\eqarr{
    \begin{aligned}
    \label{eq:ccn_mechanism}
        &\prob{\corr{Y}=\rmp|X} 
        := 
        \prob{\corr{Y}=\rmp|Y=\rmp,X} \; \prob{Y=\rmp|X} + \prob{\corr{Y}=\rmp|Y=\rmn,X} \; \prob{Y=\rmn|X}, \\
        &\prob{\corr{Y}=\rmn|X} 
        := 
        \prob{\corr{Y}=\rmn|Y=\rmp,X} \; \prob{Y=\rmp|X} + \prob{\corr{Y}=\rmn|Y=\rmn,X} \; \prob{Y=\rmn|X}
    \end{aligned}
}
define the contaminated class probability functions.
Taking the contamination targets, the class probability functions, as the base distributions
\eqarr{
    B :=
    \mmatrix{
        \prob{Y=\rmp|X} \\
        \prob{Y=\rmn|X}
    } \label{eq:base_CCN_binary} \nonumber
}
and denoting the label-generating distributions
as
\eqarr{
    \corr{P} := 
    \mmatrix{
        \prob{\corr{Y}=\rmp|X} \\
        \prob{\corr{Y}=\rmn|X}
    }, \label{eq:MCCN_2}
}
we compare the matrix form of (\ref{eq:ccn_mechanism})
\eqarr{
    \mmatrix{
        \prob{\corr{Y}=\rmp|X} \\
        \prob{\corr{Y}=\rmn|X}
    }
    =
    \mmatrix{
        \prob{\corr{Y}=\rmp|Y=\rmp,X} & \prob{\corr{Y}=\rmp|Y=\rmn,X} \\
        \prob{\corr{Y}=\rmn|Y=\rmp,X} & \prob{\corr{Y}=\rmn|Y=\rmn,X}
    }
    \mmatrix{
        \prob{Y=\rmp|X} \\
        \prob{Y=\rmn|X}
    } \nonumber
}
with $\corr{P} = \Mcorr{corr} B$ (\ref{eq:recipe1}) to realize the contamination matrix $\Mcorr{corr}$ as 
\eqarr{
    \Mcorr{CCN} := 
    \mmatrix{
        \prob{\corr{Y}=\rmp|Y=\rmp,X} & \prob{\corr{Y}=\rmp|Y=\rmn,X} \\
        \prob{\corr{Y}=\rmn|Y=\rmp,X} & \prob{\corr{Y}=\rmn|Y=\rmn,X}
    }. \label{eq:MCCN}
}
in the CCN setting.

\subsubsection{Generalized CCN}
\label{sec:GE1_gCCN}
The concept of contaminating a \emph{single} label can be extended to generating a \emph{compound} label in the multiclass classification setting. 
Let $2^{\mathcal{Y}}$ be the power set of $\mathcal{Y} = [K]$. 
Define $\mathcal{S} := 2^{\mathcal{Y}} \backslash \left\{\emptyset, \mathcal{Y} \right\}$ as the observable space of compound labels\footnote{The removal of $\emptyset$ and $\mathcal{Y}$ is that they neither fit the concepts of complimentary- or partial-labels.}.
Since a compound label $S \in \mathcal{S}$ consists of an arbitrary number of class indices, one can view $S$ as a set generated by class probabilities $\prob{Y=k|X}$. 
Therefore, generalizing the CCN formulation (\ref{eq:ccn_mechanism}), we define the label-generating process of a compound label $S$ as 
\eqarr{
    \prob{S|X} = \sum_{k=1}^{K}\prob{S|Y=k,X}\prob{Y=k|X}, \nonumber 
}
where the role of $\prob{S|Y,X}$ is the probability of converting a single label $Y$ to a compound label $S \in \mathcal{S}$.
Moreover, in CCN, the distribution $\prob{X}$ is not contaminated.
Thus, by multiplying $\prob{X}$ on both sides, we obtain the data-generating distribution
\eqarr{
    \prob{S, X} = \sum_{k=1}^{K}\prob{S|Y=k,X}\prob{Y=k, X}, \label{eq:P_S|YX_general_2}
}

Viewing $\prob{S|Y,X}$ as a contamination probability, we arrange $\prob{S=s|Y=k,X}$ into a matrix in the following lemma to formulate the contamination matrix for the multiclass CCN setting.
\begin{lemma}
\label{lma:formulate_gCCN}
    Denote the data-generating distributions as
    \eqarr{
        \corr{P} 
        := 
        \mmatrix{
            \prob{S=s_{1}, X} \\
            \vdots \\
            \prob{S=s_{|\mathcal{S}|}, X}
        } \label{eq:label_dist_2}
    }
    and the base distributions as 
    \eqarr{
        B :=
        P = 
        \mmatrix{
            \prob{Y=1, X} \\
            \vdots \\
            \prob{Y=K, X}
        }. \label{eq:base_dist_CCN_2}
    }
    Then, $\corr{P} = \Mcorr{gCCN} B$ is equivalent to the formulation (\ref{eq:P_S|YX_general_2}), with
    \eqarr{
        \Mcorr{gCCN} :=
        \mmatrix{
            \prob{S=s_{1}|Y=1,X} & \prob{S=s_{1}|Y=2,X} & \cdots & \prob{S=s_{1}|Y=K,X} \\
            \prob{S=s_{2}|Y=1,X} & \prob{S=s_{2}|Y=2,X} & \cdots & \prob{S=s_{2}|Y=K,X} \\
            \vdots & \vdots & \ddots & \vdots \\
            \prob{S=s_{|\mathcal{S}|}|Y=1,X} & \prob{S=s_{|\mathcal{S}|}|Y=2,X} & \cdots & \prob{S=s_{|\mathcal{S}|}|Y=K,X}
        } \label{eq:MGeneral_2}
    }
    being the contamination matrix generalized from (\ref{eq:MCCN}) for the multiclass CCN setting.
\end{lemma}
\begin{proof}
    For each $j \in [|\mathcal{S}|]$, we have 
    \eqarr{
        \left( \Mcorr{gCCN} B \right)_j
        =
        \sum_{k=1}^{K} \prob{S=s_{j}|Y=k,X} \prob{Y=k, X}
        =
        \sum_{k=1}^{K} \prob{S=s_{j}, Y=k, X}
        =
        \prob{S=s_{j}, X}
        =
        \corr{P}_j, \nonumber
    }
    corresponding to (\ref{eq:P_S|YX_general_2}) with $S=s_{j}$.
    Note that (\ref{eq:MGeneral_2}) generalizes (\ref{eq:MCCN}) by extending the labeling setting from $\left( \corr{Y} \in \{ \rmp, \rmn \}, Y \in \{ \rmp, \rmn \} \right)$ to $\left( S \in \{ s_{1}, \cdots, s_{2^K-2} \}, Y \in \{ 1, \cdots, K \} \right)$.
\end{proof}
Comparing $\corr{P} = \Mcorr{gCCN} B$ with the formulation framework $\corr{P} = \Mcorr{corr} B$ (\ref{eq:recipe1}), we have the reduction path 
$$
\Mcorr{corr}  \rightarrow \Mcorr{CCN} \leadsto \Mcorr{gCCN}.
$$

\explain
Similar to $\Mcorr{UU}$ (\ref{eq:MUU}), which induces multiple contamination matrices as special cases of the MCD model, $\Mcorr{gCCN}$ also derives several contamination matrices formulating partial- or complementary-label settings, as we will show in the rest of this subsection.



\subsubsection{Proper Partial-Label (PPL) Learning
\label{sec:GE1_PPL}
\texorpdfstring{\citep{partial_21_PPL/Wu/LS/23}}{Lg}}
For a given example $(y,x)$ and a compound label $s \in \mathcal{S}$, we call $s$ a partial-label of $x$ if $y \in s$. Statistically speaking, we assume $\prob{Y \in S | S, X} = 1$.
Formally, according to Definition 1 of \cite{partial_21_PPL/Wu/LS/23}, if the contamination probability can be defined as 
\eqarr{
    \prob{S|Y,X} := C(S,X)\ind{Y \in S}, \label{eq:P_S|YX_PPL}
}
via a function $C: \mathcal{S}\times\mathcal{X} \rightarrow \mathbb{R}$, we call such a partial-label scenario proper.

Since the discussion above only involves specifying $\prob{S|Y,X}$, we replace the entries of $\Mcorr{gCCN}$ (\ref{eq:MGeneral_2}) according to (\ref{eq:P_S|YX_PPL}) to construct $\Mcorr{PPL}$:
\eqarr{
    \mmatrix{
        C(s_{1},X)\ind{Y=1 \in s_{1}} & C(s_{1},X)\ind{Y=2 \in s_{1}} & \cdots & C(s_{1},X)\ind{Y=K \in s_{1}} \\
        C(s_{2},X)\ind{Y=1 \in s_{2}} & C(s_{2},X)\ind{Y=2 \in s_{2}} & \cdots & C(s_{2},X)\ind{Y=K \in s_{2}} \\
        \vdots & \vdots & \ddots & \vdots \\
        C(s_{|\mathcal{S}|},X)\ind{Y=1 \in s_{|\mathcal{S}|}} & C(s_{|\mathcal{S}|},X)\ind{Y=2 \in s_{|\mathcal{S}|}} & \cdots & C(s_{|\mathcal{S}|},X)\ind{Y=K \in s_{|\mathcal{S}|}}
    }. \label{eq:MPPL}
}
The following lemma justifies $\Mcorr{PPL}$ as the corruption matrix for PPL learning.
\begin{lemma}
\label{lma:formulate_PPL}
    Given $B$ (\ref{eq:base_dist_CCN_2}) and $\Mcorr{PPL}$ (\ref{eq:MPPL}), $\mmatrix{\Mcorr{PPL} B}_{j}$ equals $\prob{S=s_{j}, X}$ in (\ref{eq:formulate_PPL}) for each $j \in [|\mathcal{S}|]$.
    Thus, denoting $\prob{S=s_{j}, X}$ as $\mmatrix{\corr{P}}_j$, the data-generating process of PPL can be formulated as $\corr{P} = \Mcorr{PPL} B$.
\end{lemma}
\begin{proof}
    For each $j$, 
    \eqarr{
        \left( \Mcorr{PPL} B \right)_j
        =
        \sum_{k=1}^{K} C(S=s_{j}, X)\ind{Y=k \in s_{j}} \prob{Y=k, X}
        =
        C(S=s_{j}, X) \sum_{k \in s_{j}} \prob{Y=k, X} \nonumber
    }
    corresponds to $\prob{S=s_{j}, X}$ in (\ref{eq:formulate_PPL}).
    By definition, $\prob{S=s_{j}, X} = \mmatrix{\corr{P}}_j$ establishes $\mmatrix{\Mcorr{PPL} B}_j = \mmatrix{\corr{P}}_j$.
    Hence, we have the matrix formulation $\corr{P} = \Mcorr{PPL} B$ for PPL learning.
\end{proof}
The entry replacement that converts (\ref{eq:MGeneral_2}) to (\ref{eq:MPPL}) through (\ref{eq:P_S|YX_PPL}) also gives the reduction path
$$
\Mcorr{corr} \rightarrow \Mcorr{gCCN} \rightarrow \Mcorr{PPL}.
$$

\subsubsection{Provably Consistent Partial-Label (PCPL) Learning \texorpdfstring{\citep{partial_20_PCPL/Feng/LHXNGAS/20}}{Lg}} 
\label{sec:GE1_PCPL}
In PCPL, the probability of each partial-label is assumed to be sampled uniformly from all feasible partial-labels.
Since there are $2^{K-1} - 1$ feasible partial-labels for every $y$, the label-generating probability $\prob{S=s|Y=y,X}$ is $\frac{1}{2^{K-1}-1}$ if $y \in s$\footnote{There are $2^{\mathcal{Y}\backslash\{y\}} \backslash \{\mathcal{Y}\backslash\{y\}\} = 2^{K-1} - 1$ combinations whose union with $\{y\}$ are partial-labels of $y$.}.
It corresponds to assign $C(S,X) = \frac{1}{2^{K-1}-1}$ in (\ref{eq:P_S|YX_PPL}). 
Hence, we obtain
\eqarr{
    C(S,X)\ind{Y\in S} 
    := \frac{1}{2^{K-1}-1} \ind{Y\in S}, \label{eq:entry_value_MPCPL}
}
which reduces the label-generating process of PPL to that of PCPL and recovers (5) of \cite{partial_20_PCPL/Feng/LHXNGAS/20}.

Then, replacing entries in (\ref{eq:MPPL}) via (\ref{eq:entry_value_MPCPL}), we obtain the contamination matrix of PCPL learning
\eqarr{
    \Mcorr{PCPL}
    :=
    \frac{1}{2^{K-1}-1}
    \mmatrix{
        \ind{Y=1 \in s_{1}} & \ind{Y=2 \in s_{1}} & \cdots & \ind{Y=K \in s_{1}} \\
        \ind{Y=1 \in s_{2}} & \ind{Y=2 \in s_{2}} & \cdots & \ind{Y=K \in s_{2}} \\
        \vdots & \vdots & \ddots & \vdots \\
        \ind{Y=1 \in s_{|\mathcal{S}|}} & \ind{Y=2 \in s_{|\mathcal{S}|}} & \cdots & \ind{Y=K \in s_{|\mathcal{S}|}}
    } \label{eq:MPCPL}
}
and the reduction path
$$
\Mcorr{corr} \rightarrow \Mcorr{gCCN} \rightarrow \Mcorr{PPL} \rightarrow \Mcorr{PCPL}.
$$
$\Mcorr{PCPL}$ characterizing the data-generating process of PCPL is justified by the following lemma, whose proof follows the same steps as that for Lemma~\ref{lma:formulate_PPL}.
\begin{lemma}
\label{lma:formulate_PCPL}
    Given $B$ (\ref{eq:base_dist_CCN_2}) and $\Mcorr{PCPL}$ (\ref{eq:MPCPL}), $\mmatrix{\Mcorr{PCPL} B}_{j}$ equals $\prob{S=s_{j}, X}$ in (\ref{eq:formulate_PCPL}) for each $j \in [|\mathcal{S}|]$.
    Thus, denoting $\prob{S=s_{j}, X}$ as $\mmatrix{\corr{P}}_j$, the data-generating process of PCPL can be formulated as $\corr{P} = \Mcorr{PCPL} B$.
\end{lemma}

\subsubsection{Multi-Complementary-Label (MCL) Learning \texorpdfstring{\citep{comp_20_MCL/Feng/KHNAS/20}}{Lg}}
\label{sec:GE1_MCL}
Recall the discussions in Sections~\ref{sec:review_CL} and \ref{sec:review_MCL} that a complementary-label contains the exclusion information of a true label.
Notice that for any partial-label $s$ (containing the true label of $x$), there is a corresponding $\corr{s} := \mathcal{Y} \backslash s$.
The definition of partial-label implies $\corr{s}$ containing multiple class indices must not contain the true label of $x$; hence, $\corr{s}$ is called a multi-complementary-label of $x$.

The complementary relationship between $\corr{s}$ and $s$ enables us to formulate the contamination matrix characterizing MCL formulation via the next lemma.
Let us abbreviate $|\mathcal{S}|$ as $N$, $\prob{|S|=d}$ as $q_{|S|}$, and $\prob{|\corr{S}|=d}$ as $\corr{q}_{|\corr{S}|}$.
\begin{lemma}
\label{lma:matrix_formulation_MCL}
    Let $\Mcorr{MCL}$ be
    \eqarr{
        \mmatrix{
            \frac{\corr{q}_{|\corr{s}_{1}|}}{{K-1 \choose |\corr{s}_{1}|}} \ind{Y=1 \notin \corr{s}_{1}} & \frac{\corr{q}_{|\corr{s}_{1}|}}{{K-1 \choose |\corr{s}_{1}|}} \ind{Y=2 \notin \corr{s}_{1}} & \cdots & \frac{\corr{q}_{|\corr{s}_{1}|}}{{K-1 \choose |\corr{s}_{1}|}} \ind{Y=K \notin \corr{s}_{1}} \\
            \frac{\corr{q}_{|\corr{s}_{2}|}}{{K-1 \choose |\corr{s}_{2}|}} \ind{Y=1 \notin \corr{s}_{2}} & \frac{\corr{q}_{|\corr{s}_{2}|}}{{K-1 \choose |\corr{s}_{2}|}} \ind{Y=2 \notin \corr{s}_{2}} & \cdots & \frac{\corr{q}_{|\corr{s}_{2}|}}{{K-1 \choose |\corr{s}_{2}|}} \ind{Y=K \notin \corr{s}_{2}} \\
            \vdots & \vdots & \ddots & \vdots \\
            \frac{\corr{q}_{\left|\corr{s}_{N}\right|}}{{K-1 \choose \left|\corr{s}_{N}\right|}} \ind{Y=1 \notin \corr{s}_{N}} & \frac{\corr{q}_{\left|\corr{s}_{N}\right|}}{{K-1 \choose \left|\corr{s}_{N}\right|}} \ind{Y=2 \notin \corr{s}_{N}} & \cdots & \frac{\corr{q}_{\left|\corr{s}_{N}\right|}}{{K-1 \choose \left|\corr{s}_{N}\right|}} \ind{Y=K \notin \corr{s}_{N}}
        }. \label{eq:MMCL_detail}
    }
    Then, for each $j \in [N]$, $\mmatrix{\Mcorr{MCL} B}_{j}$ equals $\prob{\corr{S}=\corr{s}_{j}, X}$ in (\ref{eq:formulate_MCL}), where $B$ is defined in (\ref{eq:base_dist_CCN_2}) and $\Mcorr{MCL}$ is given by assigning each $(s,k)$ entry of $\Mcorr{PPL}$ (\ref{eq:MPPL}) with
    \eqarr{
        C(s,X)\ind{Y=k \in s} := \frac{q_{|s|}}{ {K-1 \choose |s|-1} }\ind{Y=k \in s}. \nonumber
    }
    Moreover, denoting
    \eqarr{ 
        \corr{P}
        :=
        \mmatrix{
            \prob{\corr{S}=\corr{s}_{1}, X} \\
            \vdots \\
            \prob{\corr{S}=\corr{s}_{N}, X}
        }, \label{eq:label_dist_MCL}
    }
    the data-generating process of MCL can be formulated as $\corr{P} = \Mcorr{MCL} B$.
\end{lemma}
\blockComment{
    
}
\begin{proof}
    For any $\corr{s} = \mathcal{Y}\backslash s$, the complementary relationship implies $\ind{Y\in s} = \ind{Y\notin \corr{s}}$, $q_{|s|} = \corr{q}_{|\corr{s}|}$, and ${K-1 \choose |s|-1} = {K-1 \choose |\corr{s}|}$.
    Therefore,
    \eqarr{
        C(s,X)\ind{Y\in s}
        =
        \frac{q_{|s|}}{{K-1 \choose |s|-1}} \ind{Y\in s} 
        =
        \frac{\corr{q}_{|\corr{s}|}}{{K-1 \choose |\corr{s}|}} \ind{Y\notin \corr{s}}. \label{eq:PPL_to_MCL2}
    }
    We obtain $\Mcorr{MCL}$ (\ref{eq:MMCL_detail}) by replacing the entry values in $\Mcorr{PPL}$ (\ref{eq:MPPL}) accordingly.

    Next, we show that $\corr{P} = \Mcorr{MCL} B$ is equivalent to (\ref{eq:formulate_MCL}).
    For each $j \in [N]$,
    \eqarr{
        \mmatrix{
            \Mcorr{MCL} B
        }_j
        &=& 
        \sum_{Y} \frac{\corr{q}_{|\corr{s}_j|}}{{K-1 \choose |\corr{s}_j|}} \ind{Y\notin \corr{S}=\corr{s}_j} \prob{Y,X} \nonumber \\
        &=& 
        \sum_{Y} \frac{\sum_{d=1}^{K-1} \ind{|\corr{s}_j| = d} \corr{q}_{d}}{{K-1 \choose |\corr{s}_j|}} \ind{Y\notin \corr{S}=\corr{s}_j} \prob{Y,X} \nonumber \\
        &=& 
        \sum_{d=1}^{K-1} \corr{q}_{d} \sum_{Y} \frac{1}{{K-1 \choose |\corr{s}_j|}} \ind{Y\notin \corr{S}=\corr{s}_j} \ind{|\corr{s}_j| = d}  \prob{Y,X}. \nonumber \label{eq:recover_4_MCL_b}
    }
    On the other hand, the MCL formulation (\ref{eq:formulate_MCL}) 
    \eqarr{
        \prob{\corr{S}, X}
        &=& 
        \begin{cases}
            \sum_{d=1}^{K-1} \prob{|\corr{S}| = d} \cdot \frac{1}{{K-1 \choose |\corr{S}|}} \sum_{Y\notin \corr{S}} \prob{Y,X}, & \text{if}\ |\corr{S}|=d, \\
            0, & \text{otherwise}
        \end{cases} \nonumber \\
        &=& 
        \sum_{d=1}^{K-1} \prob{|\corr{S}| = d} \cdot \frac{1}{{K-1 \choose |\corr{S}|}} \sum_{Y\notin \corr{S}} \prob{Y,X} \ind{|\corr{S}|=d} \label{eq:formulate_MCL_2}
    }
    implies
    \eqarr{
        \prob{\corr{S}=\corr{s}_j, X}
        &=&
        \sum_{d=1}^{K-1} \prob{|\corr{s}_j| = d} \cdot \frac{1}{{K-1 \choose |\corr{s}_j|}} \sum_{Y\notin \corr{s}_j} \prob{Y,X} \ind{|\corr{s}_j|=d} \nonumber \\
        &=&
        \sum_{d=1}^{K-1} \corr{q}_{d} \sum_{Y} \frac{1}{{K-1 \choose |\corr{s}_j|}} \ind{Y\notin \corr{S}=\corr{s}_j} \ind{|\corr{s}_j| = d}  \prob{Y,X} 
        =
        \mmatrix{
            \Mcorr{MCL} B
        }_j. \nonumber
    }
\end{proof}
The construction of $\Mcorr{MCL}$ (\ref{eq:MMCL_detail}) implies the reduction path
$$
\Mcorr{corr} \rightarrow \Mcorr{gCCN} \rightarrow \Mcorr{PPL} \rightarrow \Mcorr{MCL}.
$$

In addition, comparing the decomposition
$\prob{\corr{S}, X} = \sum_{d=1}^{K-1} \prob{|\corr{S}| = d} \prob{\corr{S}, X||\corr{S}|=d},$
with (\ref{eq:formulate_MCL_2}), we obtain
$$
    \prob{\corr{S}, X||\corr{S}|=d} = \frac{1}{{K-1 \choose |\corr{S}|}} \sum_{Y\notin \corr{S}} \prob{Y,X} \ind{|\corr{S}|=d},
$$
corresponding to (4) formulated in \cite{comp_20_MCL/Feng/KHNAS/20}.
The interpretation of the equation is that the size of the outcome of the random variable $\corr{S}$ should match condition $|\corr{S}| = d$.
That is, if the outcome size satisfies the condition $|\corr{S}|=d$, the probability of seeing $\corr{S}$ is 
$$
    \prob{\corr{S}, X||\corr{S}|=d} = \frac{1}{{K-1 \choose |\corr{S}|}} \sum_{Y\notin \corr{S}} \prob{Y,X},
$$
and, on the contrary, if it fails, the probability is nullified $$\prob{\corr{S}, X||\corr{S}|=d} = 0.$$


\subsubsection{Complementary-Label (CL) Learning \texorpdfstring{\citep{comp_18/Ishida/NMS/19}}{Lg}}
\label{sec:GE1_CL}
As a special case of MCL (Section~\ref{sec:review_MCL}), we first assign for each $\corr{s} \in \mathcal{S}$, $\corr{q}_{|\corr{s}|} = 1$ if $|\corr{s}| = 1$ and $\corr{q}_{|\corr{s}|} = 0$ if $|\corr{s}| > 1$.
Obviously, MCL with size $d=1$ must be $\{1\}, \dots, \{K\}$. 
Dropping all-zero rows and renaming $\corr{s}_k := \{k\}$ for $k \in [K]$, we obtain from (\ref{eq:MMCL_detail}) the contamination matrix of CL learning
\eqarr{
    \Mcorr{CL} 
    &:=& 
    \mmatrix{
        \frac{1}{K-1}\ind{Y=1 \notin \{1\}} & \frac{1}{K-1}\ind{Y=2 \notin \{1\}} & \cdots & \frac{1}{K-1}\ind{Y=K \notin \{1\}} \\
        \frac{1}{K-1}\ind{Y=1 \notin \{2\}} & \frac{1}{K-1}\ind{Y=2 \notin \{2\}} & \cdots & \frac{1}{K-1}\ind{Y=K \notin \{2\}} \\
        \vdots & \vdots & \ddots & \vdots \\
        \frac{1}{K-1}\ind{Y=1 \notin \{K\}} & \frac{1}{K-1}\ind{Y=2 \notin \{K\}} & \cdots & \frac{1}{K-1}\ind{Y=K \notin \{K\}}
    } \nonumber \\
    &=&
    \frac{1}{K-1}
    \mmatrix{
        0 & 1 & \cdots & 1 \\
        1 & 0 & \cdots & 1 \\
        \vdots & \vdots & \ddots & \vdots \\
        1 & 1 & \cdots & 0 
    } \label{eq:MCL}
}
and the reduction path
$$
\Mcorr{corr} \rightarrow \Mcorr{gCCN} \rightarrow \Mcorr{PPL} \rightarrow \Mcorr{MCL} \rightarrow \Mcorr{CL}.
$$
Furthermore, it is easy to verify that given $B$ (\ref{eq:base_dist_CCN_2}), for any $j \in [K]$, letting $\corr{S} = j$ be a singleton gives
\eqarr{
    \mmatrix{\Mcorr{CL} B}_j = \sum_{Y\neq \corr{S}=j} \frac{1}{K-1} \prob{Y, X} = \prob{\corr{S}=j, X}, \nonumber
}
which corresponds to formulation (\ref{eq:formulate_CL}).
Hence, we have the following.
\begin{lemma}
\label{lma:matrix_formulation_CL}
    $\Mcorr{CL}$ (\ref{eq:MCL}) is the contamination matrix characterizing the data-generating distribution $\prob{\corr{S}, X}$ (\ref{eq:formulate_CL}) of CL learning.
\end{lemma}

\subsection{Confidence-based Scenarios}
\label{sec:formulations_confs}
At first sight, there seems to be no connection between ``contamination'' and single-class classification \citep{scconf_21/Cao/FSXANS/21}.
However, the following derivation
\begin{eqnarray}
    \frac{\prob{Y=y_\mrm{s}|X}}{\prob{Y=j|X}} \cdot \prob{Y=j,X}
    &=&
    \frac{\prob{Y=y_\mrm{s}|X}}{\prob{Y=j|X}} \cdot \prob{Y=j|X} \, \prob{X} 
    =
    \prob{Y=y_\mrm{s}|X} \, \prob{X} \nonumber \\ 
    &=& 
    \prob{Y=y_\mrm{s}, X} \label{eq:conf_rewright}
\end{eqnarray}
reveals a way to \emph{contaminate} a clean joint probability $\prob{Y=j,X}$ to the joint probability of a specific class $y_\mrm{s}$ via confidence weighting $\frac{\prob{Y=y_\mrm{s}|X}}{\prob{Y=j|X}}$.
As we will see in the rest of this subsection, the confidence weights are the key elements in formulating the contamination matrices for the confidence-based WSL settings. 

\subsubsection{Subset Confidence (Sub-Conf) Learning \texorpdfstring{\citep{scconf_21/Cao/FSXANS/21}}{Lg}}
\label{sec:GE1_Sub-Conf}
Let $\mathcal{Y}_\mrm{s} \subset [K]$ be a subset of classes.
Viewing $\mathcal{Y}_\mrm{s}$ as a ``super-class'', such that every instance $x$ of $(y, x)$ will be labeled $\mathcal{Y}_\mrm{s}$ if $y \in \mathcal{Y}_\mrm{s}$, we can define its class prior as $\prob{Y\in\mathcal{Y}_\mrm{s}} = \pi_{\mathcal{Y}_\mrm{s}} := \sum_{y\in\mathcal{Y}_\mrm{s}} \pi_y$ and its class probability function as $\prob{Y\in\mathcal{Y}_\mrm{s}|X} := \sum_{y\in\mathcal{Y}_\mrm{s}}\prob{Y=y|X}$.
Reusing the argument (\ref{eq:conf_rewright}), 
\begin{eqnarray}
    \frac{\prob{Y\in\mathcal{Y}_\mrm{s}|X}}{\prob{Y=j|X}} \cdot \prob{Y=j,X} 
    &=& 
    \frac{\prob{Y\in\mathcal{Y}_\mrm{s}|X}}{\prob{Y=j|X}} \cdot \prob{Y=j|X} \, \prob{X} \nonumber \\
    &=&
    \prob{Y\in\mathcal{Y}_\mrm{s},X} \nonumber
\end{eqnarray}
shows that no matter what joint distribution $\prob{Y=j, X}$ to begin with, the confidence weight $\frac{\prob{Y\in\mathcal{Y}_\mrm{s}|X}}{\prob{Y=j|X}}$ twists that joint distribution so that every observed data appears to be sampled from the same super-class distribution $\prob{\mathcal{Y}_\mrm{s}, X}$.
The following lemma leverages the observation to specify the contamination matrix $\Mcorr{Sub}$ characterizing Sub-Conf learning.
\begin{lemma}
\label{lma:formulate_Sub-Conf}
    Denote the base distributions as 
    \eqarr{
        B
        :=
        P = 
        \mmatrix{
            \prob{Y=1, X} \\
            \vdots \\
            \prob{Y=K, X}
        }. \label{eq:base_dist_Conf-based}
    }
    Inserting the confidence weights into the identity matrix, we define
    \eqarr{
        \Mcorr{Sub} 
        :=
        \mmatrix{
            \frac{\prob{Y\in\mathcal{Y}_\mrm{s}|X}}{\prob{Y=1|X}} & \cdots & 0 \\
            \vdots & \ddots & \vdots \\
            0 & \cdots & \frac{\prob{Y\in\mathcal{Y}_\mrm{s}|X}}{\prob{Y=K|X}}
        }. \label{eq:MSub}
    }
    Then, for any $j \in [K]$, $\mmatrix{\Mcorr{Sub} B}_{j}$ is equivalent to $\prob{X|Y\in\mathcal{Y}_\mrm{s}}$ in (\ref{eq:formulate_Sub-conf}).
    Moreover, denoting
    \eqarr{
        \corr{P}
        :=
        \mmatrix{
            \prob{Y\in\mathcal{Y}_\mrm{s}, X} \\
            \vdots \\
            \prob{Y\in\mathcal{Y}_\mrm{s}, X}
        }, \nonumber
    }
    the data-generating process of Sub-Conf can be formulated as $\corr{P} = \Mcorr{Sub} B$.
\end{lemma}
\begin{proof}
    For each $j \in [K]$, $\mmatrix{\Mcorr{Bub} B}_{j} = \prob{Y\in\mathcal{Y}_\mrm{s}, X}$ since
    $$
    \frac{\prob{Y\in\mathcal{Y}_\mrm{s}|X}}{\prob{Y=j|X}} \cdot \prob{Y=j,X} = \prob{Y\in\mathcal{Y}_\mrm{s},X}. 
    $$
    Thus, by definition, $\corr{P} = \Mcorr{Sub} B$.
    It further implies all observed instances are labeled with the same super-class $\mathcal{Y}_\mrm{s}$, meaning we can drop the observed labels, and the observed examples $\{x_i\}_{i=1}^{n}$ is equivalent to a set of i.i.d.\ samples from $\prob{X|Y\in\mathcal{Y}_\mrm{s}}$ (\ref{eq:formulate_Sub-conf}).
\end{proof}
Comparing $\corr{P} = \Mcorr{Sub} B$ with the formulation framework $\corr{P} = \Mcorr{corr} B$ (\ref{eq:recipe1}), we observe that in Sub-Conf learning, $\Mcorr{corr}$ is realized as $\Mcorr{Sub}$: 
\eqarr{
    \Mcorr{corr} 
    \rightarrow 
    \Mcorr{Sub}. \nonumber
}

\subsubsection{Single-Class Confidence (SC-Conf) Learning \texorpdfstring{\citep{scconf_21/Cao/FSXANS/21}}{Lg}}
\label{sec:GE1_SCConf}
We compare the formulation of SC-Conf (\ref{eq:formulate_SC-conf}) with Sub-Conf (\ref{eq:formulate_Sub-conf}) and observe that SC-Conf is a special case of Sub-Conf when $\mathcal{Y}_\mrm{s} = \{ y_\mrm{s} \}$ being a singleton.
Thus, we straightforwardly obtain the matrix formulation of SC-Conf from Lemma~\ref{lma:formulate_Sub-Conf}:
\begin{lemma}
\label{lma:formulate_SC-Conf}
    Let the base distributions $B$ be defined by (\ref{eq:base_dist_Conf-based}).
    Define 
    \eqarr{
        \Mcorr{SC} 
        :=
        \mmatrix{
            \frac{\prob{Y=y_\mrm{s}|X}}{\prob{Y=1|X}} & \cdots & 0 \\
            \vdots & \ddots & \vdots \\
            0 & \cdots & \frac{\prob{Y=y_\mrm{s}|X}}{\prob{Y=K|X}}
        } \label{eq:MSC} 
    }
    by substituting $\mathcal{Y}_\mrm{s}$ in Lemma~\ref{lma:formulate_Sub-Conf} with $y_\mrm{s}$.
    Then, for any $j \in [K]$, $\mmatrix{\Mcorr{SC} B}_{j}$ is equivalent to $\prob{X|Y=y_\mrm{s}}$ in (\ref{eq:formulate_SC-conf}). 
     Moreover, denoting
    \eqarr{
        \corr{P}
        :=
        \mmatrix{
            \prob{Y=y_\mrm{s}, X} \\
            \vdots \\
            \prob{Y=y_\mrm{s}, X}
        }, \nonumber
    }
    the data-generating process of SC-Conf can be formulated as $\corr{P} = \Mcorr{SC} B$.
\end{lemma}
Since SC-Conf is a special case of Sub-Conf, we have the reduction path
$$
    \Mcorr{corr} 
    \rightarrow \Mcorr{Sub}  
    \rightarrow \Mcorr{SC}.
$$

\subsubsection{Positive-confidence (Pconf) Learning \texorpdfstring{\citep{pconf_17/Ishida/NS/18}}{Lg}}
\label{sec:GE1_Pconf}
Comparing (\ref{eq:formulate_Pconf}) with (\ref{eq:formulate_SC-conf}), we see that Pconf is a special case of SC-Conf when $K=2$ and $y_\mrm{s} = \rmp$ since $r_{\rmn}(X) = 1- r_{\rmp}(X)$.
A further modification to Lemma~\ref{lma:formulate_SC-Conf} we obtain the contamination matrix $\Mcorr{Pconf}$ characterizing Pconf learning.
\begin{lemma}
\label{lma:formulate_Pconf}
    Let $B := P = \mmatrix{\prob{Y=\rmp, X} \\ \prob{Y=\rmn, X}}$.
    Define 
    \eqarr{
        \Mcorr{Pconf} 
        :=
        \mmatrix{
            \frac{\prob{Y=\rmp|X}}{\prob{Y=\rmp|X}} & 0 \\
            0 & \frac{\prob{Y=\rmp|X}}{\prob{Y=\rmn|X}}
        }. \label{eq:MPconf}
    }
    Then, each entry of $\Mcorr{Pconf} B$ is equivalent to $\prob{\mrm{P}}$ in (\ref{eq:formulate_Pconf}).
    Furthermore, $\Mcorr{Pconf}$ characterizes the data-generating process of Pconf since $\corr{P} = \Mcorr{Pconf} B$, where $\corr{P} := \mmatrix{\prob{Y=\rmp, X} \\ \prob{Y=\rmp, X}}$.
\end{lemma}
The entry replacement that converts (\ref{eq:MSC}) to  (\ref{eq:MPconf}) implies the reduction path
$$
    \Mcorr{corr} 
    \rightarrow \Mcorr{Sub}  
    \rightarrow \Mcorr{SC}
    \rightarrow \Mcorr{Pconf}.
$$

\subsubsection{Soft-Label Learning \texorpdfstring{\citep{soft_22/Ishida/YCNS/22}}{Lg}}
\label{sec:GE1_Soft}
The difference between the soft-label and the previous confidence-based settings (Sub-Conf, SC-Conf, and Pconf) is how $x$ is sampled.
The sample distributions condition on the label information in the previous settings, while that in soft-label is $\prob{X}$.
Reusing argument (\ref{eq:conf_rewright}), the equation 
\eqarr{
    \frac{1}{\prob{Y=j|X}} \cdot \prob{Y=j,X} = \prob{X} \nonumber
}
reveals how to convert $\prob{Y=j,X}$ to $\prob{X}$.
Therefore, filling the $j$-th diagonal entry of the identity matrix with $\frac{1}{\prob{Y=j|X}}$, we obtain the contamination matrix $\Mcorr{Soft}$ for soft-label learning:
\begin{lemma}
\label{lma:formulate_soft}
    Let the base distributions $B$ be defined by (\ref{eq:base_dist_Conf-based}).
    Denote
    \eqarr{
        \corr{P}
        :=
        \mmatrix{
            \prob{X} \\
            \vdots \\
            \prob{X}
        }. \label{eq:data_dist_soft}
    }
    Define
    \eqarr{
        \Mcorr{Soft}
        :=
        \mmatrix{
            \frac{1}{\prob{Y=1|X}} & \cdots & 0 \\
            \vdots & \ddots & \vdots \\
            0 & \cdots & \frac{1}{\prob{Y=K|X}}
        }. \label{eq:MSoft}
    }
    Then, $\corr{P} = \Mcorr{Soft} B$ formulates the data-generating process in (\ref{eq:formulate_Soft}).
\end{lemma}
\compare
Unlike SC-Conf and Pconf, which are special cases of Sub-Conf with $\mathcal{Y}_\mrm{s}$ taking only one label, the generation process of a soft-label can be viewed as assigning $\mathcal{Y}_\mrm{s} := [K]$.
Considering the entire label space results in $\prob{Y\in [K]|X} = 1$; it coincides with the meaning of $\prob{X}$ that samples $x$ regardless of the labels.
Although technically the soft-label setting is not a special case of Sub-Conf (recalling the $\mathcal{Y}_\mrm{s} \subset [K]$ assumption from Section~\ref{sec:GE2_Sub-Conf}), $\Mcorr{Soft}$ (\ref{eq:MSoft}) is reduced from $\Mcorr{Sub}$ (\ref{eq:MSub}) by realizing $\prob{Y\in\mathcal{Y}_{\mrm{s}}|X}$ as $\prob{Y\in [K]|X} = 1$.
Therefore, we obtain the following reduction path
$$
\Mcorr{corr} 
\rightarrow 
\Mcorr{Sub}
\rightarrow 
\Mcorr{Soft}.
$$


\section{Risk Rewrite via Decontamination}
\label{sec:riskRewrite}
We have demonstrated the capability of the proposed formulation component (\ref{eq:recipe1}) in the last section. 
This section shows how the proposed framework provides a unified methodology for solving the risk rewrite problem.
Specifically, given each contamination matrix described in Section~\ref{sec:matrixFormulations}, we show how to construct the corrected losses
(\ref{eq:recipe5}) to perform the risk rewrite via (\ref{eq:recipe2a}).
We then recover each rewrite to the corresponding form reported in the literature to justify its feasibility.
Because this paper focuses on a unified methodology for rewriting the classification risk instead of the designs of practical training objectives, we assume the required parameters are given or can be estimated accurately from the observed data.

\subsection{MCD Scenarios}
\label{sec:risk_rewrite_MCDs}
We apply the framework to conduct the risk rewrites for WSLs formulated in Section~\ref{sec:formulations_mcds}, whose summary is in Table~\ref{tab:MCD_matrices_summary}.
A general approach is to show that the inversion method discussed in Theorem~\ref{thm:inv_method} provides the decontamination matrix $\Mcorr{corr}^{\dagger}$ required in (\ref{eq:recipe5}).

\subsubsection{Unlabeled-Unlabeled (UU) Learning}
\label{sec:GE2_UU}
We justify the proposed framework for UU learning via the following steps.

\paragraph{Step 1: Corrected Loss Design and Risk Rewrite.}~\\
\indent
Recall that (\ref{eq:verify_MUU}) connects the data-generating distributions 
$\corr{P} = 
\mmatrix{
    \prob{\mrm{U}_1} \\
    \prob{\mrm{U}_2}
}$ 
and the base distributions 
$B = 
\mmatrix{
    \prob{X|Y=\rmp} \\
    \prob{X|Y=\rmn}
}$ 
and instantiates (\ref{eq:recipe1}) as $\corr{P} = \Mcorr{UU} B$.
To further link $\corr{P}$ with the risk-defining distributions 
$ P =
\mmatrix{
    \prob{Y=\rmp, X} \\
    \prob{Y=\rmn, X}
},$ 
we still need $\Mcorr{trsf}$ satisfying $B = \Mcorr{trsf} P$.
Introducing the prior matrix 
\eqarr{
    \Pi = 
    \mmatrix{
        \pi_\rmp & 0 \\
        0 & \pi_\rmn
    }, \nonumber
}
we see that $\Mcorr{trsf} = \Pi^{-1}$ fulfills the need:
\eqarr{
    \Mcorr{trsf} P
    =
    \mmatrix{
        \pi_\rmp^{-1} & 0 \\
        0 & \pi_\rmn^{-1}
    }
    \mmatrix{
        \prob{Y=\rmp, X} \\
        \prob{Y=\rmn, X}
    }
    =
    \mmatrix{
        \frac{\prob{Y=\rmp, X}}{\prob{Y=\rmp}} \\
        \frac{\prob{Y=\rmn, X}}{\prob{Y=\rmn}}
    }
    =
    \mmatrix{
        \prob{X|Y=\rmp} \\
        \prob{X|Y=\rmn}
    }
    =
    B. \nonumber
}
Hence, $\corr{P} = \Mcorr{corr} \Mcorr{trsf} P$ (\ref{eq:recipe7}) is realized as
\eqarr{
    \mmatrix{
        \prob{\mrm{U}_1} \\
        \prob{\mrm{U}_2}
    }
    =
    \Mcorr{UU}
    \Pi^{-1}
    \mmatrix{
        \prob{Y=\rmp, X} \\
        \prob{Y=\rmn, X}
    } \label{eq:corr_P_UU}
}
in UU learning.

Next, we apply Theorem~\ref{thm:inv_method} to construct the decontamination matrix $\Mcorr{corr}^{\dagger}$ needed in (\ref{eq:recipe6}).
We denote the modified loss at the $\corr{k}$-th entry of $\corr{L}$ as $\corr{\ell}_{\corr{k}} := \ell_{\corr{Y}=\corr{k}}(g(X))$, where $\corr{k} \in \corr{\mathcal{Y}}$ is a class of the observed data\footnote{The definition is in contrast to the original loss $\ell_{k} := \ell_{Y=k}(g(X))$.}.
\begin{corollary}
\label{thm:UU_M_inv}
    Let $\corr{P} = \Mcorr{UU} \Pi^{-1} P$ (\ref{eq:corr_P_UU}), and assume $\Mcorr{UU}$ is invertible.
    Then, defining the decontamination matrix for UU learning as
    \eqarr{
        \Mcorr{UU}^{\dagger}
        :=
        \Pi \Mcorr{UU}^{-1} \nonumber
    }
    gives rise to
    $
        \Mcorr{UU}^{\dagger} \corr{P}
        =
        P.
    $
\end{corollary}
\begin{proof}
    Suggested by Theorem~\ref{thm:inv_method}, the inverse matrix $\Pi\Mcorr{UU}^{-1}$ cancels out the contamination brought by $\Mcorr{UU}\Pi^{-1}$ in (\ref{eq:corr_P_UU}).
    Assigning $\Mcorr{UU}^{\dagger} = \Pi \Mcorr{UU}^{-1}$ and repeating the proof of Theorem~\ref{thm:inv_method}, we have 
    \eqarr{
        \Mcorr{UU}^{\dagger} \corr{P}
        =
        \Pi\Mcorr{UU}^{-1} \corr{P}
        =
        \Pi\Mcorr{UU}^{-1} \Mcorr{UU}\Pi^{-1} P
        =
        P \nonumber
    }
    that completes the proof.  
\end{proof}

With $\Mcorr{UU}^{\dagger}$ in hand, we proceed to devise the corrected losses $\corr{L}$ to achieve the risk rewrite for UU learning.
The following theorem proves rewrite (\ref{eq:review_rewrite_UU}) in Section~\ref{sec:review_UU}.
\begin{theorem}
\label{thm:UU_rewrite_corrected_losses}\label{thm:rewrite_MCD_UU}
Let $\mcdp, \mcdn > 0$ and $\mcdp + \mcdn \neq 1$.
Given $\Mcorr{UU}^{\dagger} = \Pi \Mcorr{UU}^{-1} \nonumber$ defined in Corollary~\ref{thm:UU_M_inv}, the vector of corrected losses suggested by (\ref{eq:recipe5})
\eqarr{
    \corr{L}^{\top}
    :=
    L^{\top} \Mcorr{UU}^{\dagger} 
    =
    \mmatrix{
        \corr{\ell}_{\mrm{U}_1} & \corr{\ell}_{\mrm{U}_2}
    } \nonumber
}
with
\eqarr{
    \corr{\ell}_{\mrm{U}_1} &=& \frac{(1-\mcdn)\pi_\rmp}{1-\mcdp-\mcdn} \ell_\rmp + \frac{-\mcdn \pi_\rmn}{1-\mcdp-\mcdn} \ell_\rmn, \nonumber \\
    \corr{\ell}_{\mrm{U}_2} &=& \frac{-\mcdp\pi_\rmp}{1-\mcdp-\mcdn} \ell_\rmp + \frac{(1-\mcdp)\pi_\rmn}{1-\mcdp-\mcdn} \ell_\rmn \label{eq:calibrated_loss_UU}
}
achieves the following risk rewrite:
\eqarr{
    R(g)
    =
    \expt{\mrm{U}_1}{\corr{\ell}_{\mrm{U}_1}} + \expt{\mrm{U}_2}{\corr{\ell}_{\mrm{U}_2}}. \label{eq:uu_risk_rewrite_a}
}
\end{theorem}
\begin{proof}
Since $\mcdp + \mcdn \neq 1$, 
\eqarr{
    \Mcorr{UU}^{-1}
    =
    \mmatrix{
        1-\mcdp & \mcdp \\
        \mcdn & 1-\mcdn
    }^{-1}
    =
    \mmatrix{
        \frac{1-\mcdn}{1-\mcdp-\mcdn} & \frac{-\mcdp}{1-\mcdp-\mcdn} \\
        \frac{-\mcdn}{1-\mcdp-\mcdn} & \frac{1-\mcdp}{1-\mcdp-\mcdn}
    } \nonumber
}
exists. The following derivation 
\eqarr{
    \mmatrix{
        \corr{\ell}_{\mrm{U}_1} & \corr{\ell}_{\mrm{U}_2}
    }
    &=&
    L^{\top} \Pi\Mcorr{UU}^{-1} \label{eq:subtle_diff_UU} \\
    &=&
    \mmatrix{
        \ell_\rmp & \ell_\rmn
    }
    \mmatrix{
        \pi_\rmp & 0 \\
        0 & \pi_\rmn
    } 
    \mmatrix{
        \frac{1-\mcdn}{1-\mcdp-\mcdn} & \frac{-\mcdp}{1-\mcdp-\mcdn} \\
        \frac{-\mcdn}{1-\mcdp-\mcdn} & \frac{1-\mcdp}{1-\mcdp-\mcdn}
    } \nonumber \\
    &=&
    \mmatrix{
        \ell_\rmp & \ell_\rmn
    }
    \mmatrix{
        \frac{(1-\mcdn)\pi_\rmp}{1-\mcdp-\mcdn} & \frac{-\mcdp\pi_\rmp}{1-\mcdp-\mcdn} \\
        \frac{-\mcdn \pi_\rmn}{1-\mcdp-\mcdn} & \frac{(1-\mcdp)\pi_\rmn}{1-\mcdp-\mcdn}
    } \nonumber    
}
gives (\ref{eq:calibrated_loss_UU}).

Next, with essential components $\corr{P}$ and $\corr{L}^{\top}$ in hand, applying (\ref{eq:recipe2a}), we obtain
\eqarr{
    R(g)
    &=&
    \int_{\mathcal{X}} \corr{L}^{\top} \corr{P} \, \dx \label{eq:uu_risk_rewrite} \\
    &=&
    \int_{\mathcal{X}} \left( \prob{\mrm{U}_1} \, \corr{\ell}_{\mrm{U}_1} + \prob{\mrm{U}_2} \, \corr{\ell}_{\mrm{U}_2} \right) \dx \nonumber \\
    &=&
    \expt{\mrm{U}_1}{\corr{\ell}_{\mrm{U}_1}} + \expt{\mrm{U}_2}{\corr{\ell}_{\mrm{U}_2}}, \nonumber    
}
where the first equality holds since
\eqarr{
    \corr{L}^{\top} \corr{P} 
    =
    L^{\top} \Mcorr{corr}^{\dagger} \corr{P}
    =
    L^{\top} P. \nonumber
}
\end{proof}

\explain
In (\ref{eq:subtle_diff_UU}), we do not need to specify the instance in $\ell_\rmp$ and $\ell_\rmn$ to be $x^{\rmu_1}$ or $x^{\rmu_2}$ since the equality holds for any instance $x$.
We only need to distinguish $x^{\rmu_1}$ from $x^{\rmu_2}$ when the corrected losses multiply the data distributions.
In particular, the most detailed form of rewrite (\ref{eq:uu_risk_rewrite}) aligning (\ref{eq:formulate_UU}) is
\eqarr{
    R(g)
    &=&
    \expt{\mrm{U}_1}{\corr{\ell}_{\mrm{U}_1}} + \expt{\mrm{U}_2}{\corr{\ell}_{\mrm{U}_2}} \nonumber \\
    &=&
    \expt{x^{\rmu_1} \sim \prob{\mrm{U}_1}}{\frac{(1-\mcdn)\pi_\rmp}{1-\mcdp-\mcdn} \ell_\rmp(X^{\rmu_1}) + \frac{-\mcdn \pi_\rmn}{1-\mcdp-\mcdn} \ell_\rmn(X^{\rmu_1})} \nonumber \\
    && + \;\; 
    \expt{x^{\rmu_2} \sim \prob{\mrm{U}_2}}{\frac{-\mcdp\pi_\rmp}{1-\mcdp-\mcdn} \ell_\rmp(X^{\rmu_2}) + \frac{(1-\mcdp)\pi_\rmn}{1-\mcdp-\mcdn} \ell_\rmn(X^{\rmu_2})}. \nonumber
}
The freedom from specifying $x$ in (\ref{eq:subtle_diff_UU}) eliminates the notational burden of distinguishing $\ell_Y(X^{\rmu_1})$ from $\ell_Y(X^{\rmu_2})$, allowing us to exploit the advantage of matrix multiplication while constructing the corrected losses.
The freedom also enables separated treatments for the data distributions (e.g., formulating $\corr{P} = \Mcorr{UU}\Pi^{-1} P$) and the corrected losses (e.g., devising $\corr{L}^{\top} = L^{\top} \Mcorr{UU}^{\dagger}$).

\paragraph{Step 2: Recovering the previous result(s).}~\\
\indent
Lastly, we verify the feasibility of our rewrite by showing that our rewrite corresponds to an existing result.
By parameter substitution, we replace $\mcdp$ with $1-\theta$, $\mcdn$ with $\theta'$, $\pi_\rmn$ with $1-\pi_\rmp$, $\ell_\rmp$ with $\ell(g(X))$, and $\ell_\rmn$ with $\ell(-g(X))$.
Then, (\ref{eq:calibrated_loss_UU}) becomes
\eqarr{
    \frac{(1-\theta')\pi_\rmp}{\theta-\theta'}\ell(g(X)) + \frac{-\theta'(1-\pi_\rmp)}{\theta-\theta'}\ell(-g(X)) &=& \bar{\ell}_{+}(g(X)), \nonumber \\
    \frac{\theta(1-\pi_\rmp)}{\theta-\theta'}\ell(-g(X)) + \frac{-(1-\theta)\pi_\rmp}{\theta-\theta'}\ell(g(X)) &=& \bar{\ell}_{-}(-g(X)), \nonumber
}
recovering the corrected loss functions (8) and the constants reported in Theorem 4 of \cite{uu_18/Lu/NMS/19}.

\subsubsection{Positive-Unlabeled (PU) Learning}
\label{sec:GE2_PU}
Recall that all WSLs discussed in Section~\ref{sec:formulations_mcds} share the same base distributions $B$ (\ref{eq:base_dist_MCD}).
Further, as shown in Table~\ref{tab:MCD_matrices_summary}, the contamination matrix of every WSL scenario beneath UU learning except $\Mcorr{Sconf}$ is a child of $\Mcorr{UU}$ on the reduction graph.
It means $\corr{P} = \Mcorr{UU} \Pi^{-1} P$ (\ref{eq:corr_P_UU}) is a general form for every child scenario in Table~\ref{tab:MCD_matrices_summary} (with different realizations of $\mcdp$ and $\mcdn$). 
Hence, we can reuse Theorem~\ref{thm:UU_rewrite_corrected_losses} to conduct the risk rewrite for every child scenario on the reduction graph.
PU learning is the first of such examples.

\paragraph{Step 1: Corrected Loss Design and Risk Rewrite.}~\\
\indent By the following corollary, we prove rewrite (\ref{eq:review_rewrite_PU}) in Section~\ref{sec:review_PU}.
\begin{corollary}
\label{thm:rewrite_MCD_PU}
    For PU learning, the classification risk can be rewritten as 
    \eqarr{
        R(g)
        =
        \expt{\mrm{P}}{\corr{\ell}_{\mrm{P}}} + \expt{\mrm{U}}{\corr{\ell}_{\mrm{U}}}, \label{eq:risk_rewrite_PU}
    }
    where 
    \eqarr{
        \corr{\ell}_{\mrm{P}} &=& \pi_\rmp \ell_\rmp - \pi_\rmp \ell_\rmn,  \nonumber \\
        \corr{\ell}_{\mrm{U}} &=& \ell_\rmn \nonumber.
    }
\end{corollary}
\begin{proof}
    According to Table~\ref{tab:MCD_matrices_summary}, $\Mcorr{PU}$ is a child of $\Mcorr{UU}$ on the reduction graph.
    Thus, replacing the subscripts $\{\mrm{U}_1, \mrm{U}_2\}$ of data-generating distributions $\corr{P} = \mmatrix{\prob{\mrm{U}_1} \\ \prob{\mrm{U}_2}}$ and the corrected losses $\corr{L} = \mmatrix{\corr{\ell}_{\mrm{U}_1} \\ \corr{\ell}_{\mrm{U}_2}}$ with $\{\mrm{P}, \mrm{U}\}$ and assigning $\mcdp=0$ and $\mcdn = \pi_\rmp$ as what we choose in Section~\ref{sec:GE1_PU}, we call Theorem~\ref{thm:UU_rewrite_corrected_losses} to conduct the risk rewrite:
    We obtain $\corr{\ell}_{\rm{P}}$ and $\corr{\ell}_{\mrm{U}}$ by plugging $\mcdp = 0$ and $\mcdn = \pi_\rmp$ into (\ref{eq:calibrated_loss_UU}).
    Then, repeating the proof steps in (\ref{eq:uu_risk_rewrite}), we achieve (\ref{eq:risk_rewrite_PU}).
\end{proof}

\paragraph{Step 2: Recovering the previous result(s).}~\\
\indent
Since $\prob{\mrm{P}}$ is $\prob{X|Y=\rmp}$ and $\prob{\mrm{U}}$ is $\prob{X}$, we further obtain 
\eqarr{
    R(g) 
    &=& \expt{\mrm{P}}{\corr{\ell}_{\mrm{P}}} + \expt{\mrm{U}}{\corr{\ell}_{\mrm{U}}} \nonumber \\
    &=& \expt{\mrm{P}}{\pi_\rmp \ell_\rmp - \pi_\rmp \ell_\rmn} + \expt{\mrm{U}}{\ell_\rmn} \nonumber \\
    &=& \pi_\rmp\expt{X|Y=\rmp}{\ell_\rmp} - \pi_\rmp\expt{X|Y=\rmp}{\ell_\rmn} + \expt{X}{\ell_\rmn} \nonumber 
}
from (\ref{eq:risk_rewrite_PU}), which corresponds to the risk estimators (2) in \cite{pu_17/Kiryo/NPS/17} and (3) in \cite{pu_15/Plessis/NS/15}.

Moreover, with an additional symmetric assumption of $\ell_\rmp + \ell_\rmn = 1$, one further obtains
\begin{eqnarray}
    R(g)
    &=&
    \pi_\rmp\expt{X|Y=\rmp}{\ell_\rmp} - \pi_\rmp\expt{X|Y=\rmp}{\ell_\rmn} + \expt{X}{\ell_\rmn} \nonumber \\
    &=&
    \pi_\rmp\expt{X|Y=\rmp}{\ell_\rmp} - \pi_\rmp\expt{X|Y=\rmp}{1-\ell_\rmp} + \expt{X}{\ell_\rmn} \nonumber \\
    &=&
    \pi_\rmp\expt{X|Y=\rmp}{\ell_\rmp} - \pi_\rmp\expt{X|Y=\rmp}{1} + \pi_\rmp\expt{X|Y=\rmp}{\ell_\rmp} + \expt{X}{\ell_\rmn} \nonumber \\
    &=&
    2\pi_\rmp\expt{X|Y=\rmp}{\ell_\rmp} - \pi_\rmp + \expt{X}{\ell_\rmn}. \nonumber
\end{eqnarray}
This expression recovers several risk rewrites such as (4) of \cite{pu_17/Kiryo/NPS/17}, (3) of \cite{pu_16/Niu/PSMS/16}, (2) of \cite{pu_15/Plessis/NS/15}\footnote{As the 0-1 loss is symmetric.}, and (3) of \cite{pu_14/Plessis/NS/14}.

\subsubsection{Similar-Unlabeled (SU) Learning}
\label{sec:GE2_SU}
According to Table~\ref{tab:MCD_matrices_summary}, $\Mcorr{SU}$ is a child of $\Mcorr{UU}$ on the reduction graph.
Thus, we can follow the same steps illustrated in Section~\ref{sec:GE2_PU} to justify the proposed framework.

\paragraph{Step 1: Corrected Loss Design and Risk Rewrite.}~\\
\indent
The following corollary combines (\ref{eq:calibrated_loss_UU}) and (\ref{eq:uu_risk_rewrite_a}) to conduct the risk rewrite.
\begin{corollary}
\label{thm:rewrite_SU}
    Assume $\pi_\rmp \neq 1/2$.
    For SU learning, the classification risk can be rewritten as
    \eqarr{
        R(g)
        =
        \expt{\tilde{\mrm{S}}}{\corr{\ell}_{\tilde{\mrm{S}}}} + \expt{\mrm{U}}{\corr{\ell}_{\mrm{U}}}, \label{eq:risk_rewrite_SU} \nonumber
    }
    where
    \eqarr{
        \begin{aligned}
        \label{eq:calibrated_losses_SU}
            &\corr{\ell}_{\tilde{\mathrm{S}}} 
            = 
            \frac{\pi_\rmp^2+\pi_\rmn^2}{2\pi_\rmp -1} \ell_\rmp - \frac{\pi_\rmp^2+\pi_\rmn^2}{2\pi_\rmp -1} \ell_\rmn,
            \\
            &\corr{\ell}_{\mathrm{U}} 
            = 
            -\frac{\pi_\rmn}{2\pi_\rmp -1} \ell_\rmp + \frac{\pi_\rmp}{2\pi_\rmp -1} \ell_\rmn.
        \end{aligned}
    }
\end{corollary}
\begin{proof}
By Table~\ref{tab:MCD_matrices_summary}, $\Mcorr{SU}$ is a child of $\Mcorr{UU}$.
Substituting the subscripts $\{\mrm{U}_1, \mrm{U}_2\}$ with subscripts $\{\tilde{\mrm{S}}, \mrm{U}\}$ and choosing $\mcdp = \frac{\pi_\rmn^2}{\pi_\rmp^2+\pi_\rmn^2}$ and $\mcdn = \pi_\rmp$ as we did in Section~\ref{sec:GE1_SU}, we construct the corrected losses by plugging the assigned values into (\ref{eq:calibrated_loss_UU}). 
We note that $\pi_\rmp \neq 1/2$ ensures the choices of $\mcdp$ and $\mcdn$ above satisfy the $\mcdp + \mcdn \neq 1$ assumption discussed in Section~\ref{sec:GE1_UU}.
Then, we obtain the rewrite by repeating the derivation for (\ref{eq:uu_risk_rewrite_a}). 
\end{proof}

\paragraph{Step 2: Recovering the previous result(s).}~\\
\indent
To recover Theorem 1 of \cite{su_18/Bao/NS/18}, we first need to restore $\expt{\mrm{S}}{\cdot}$ from $\expt{\tilde{\mrm{S}}}{\cdot}$ in Corollary~\ref{thm:rewrite_SU}.
The following lemma provides a means for us to do so.
\begin{lemma}
\label{lma:recover_SU_1}
Given $B$ (\ref{eq:base_dist_MCD}) and following the SU learning notations, we have $\corr{P} = \corr{P}'$,
where
\eqarr{
    \corr{P}'
    &=&
    \Mcorr{SU}' B, \nonumber \\
    \Mcorr{SU}'
    &:=&
    \mmatrix{
        \frac{\pi_\rmp^2 \int_{x'\in\mathcal{X}} \prob{x'|Y=\rmp} \dx' }{\pi_\rmp^2+\pi_\rmn^2} & \frac{\pi_\rmn^2 \int_{x'\in\mathcal{X}} \prob{x'|Y=\rmn} \dx'}{\pi_\rmp^2+\pi_\rmn^2} \\
        \pi_\rmp & \pi_\rmn
    }. \nonumber
}
\end{lemma}
\begin{proof}
Since $\int_{x'\in\mathcal{X}} \prob{x'|Y=\rmp} \dx' = 1$ and $\int_{x'\in\mathcal{X}} \prob{x'|Y=\rmn} \dx' = 1$, we have $\Mcorr{SU}' = \Mcorr{SU}$, and hence $\corr{P}' = \Mcorr{SU}' B = \Mcorr{SU} B = \corr{P}$.
\end{proof}
Lemma~\ref{lma:recover_SU_1} allows us to slightly revise the derivation (\ref{eq:uu_risk_rewrite}) as follows:
\eqarr{
    R(g)
    &=&
    \int_{x\in\mathcal{X}} \corr{L}^{\top} \corr{P} \, \dx 
    =
    \int_{x\in\mathcal{X}} \corr{L}^{\top} \corr{P}' \, \dx \nonumber \\
    &=&
    \int_{x\in\mathcal{X}} 
    \mmatrix{
        \corr{\ell}_{\tilde{\mrm{S}}} & \corr{\ell}_{\mrm{U}}
    }
    \mmatrix{
        \frac{\pi_\rmp^2 \int_{x'\in\mathcal{X}} \prob{x'|Y=\rmp} \dx' }{\pi_\rmp^2+\pi_\rmn^2} & \frac{\pi_\rmn^2 \int_{x'\in\mathcal{X}} \prob{x'|Y=\rmn} \dx'}{\pi_\rmp^2+\pi_\rmn^2} \\
        \pi_\rmp & \pi_\rmn
    }
    \mmatrix{
        {\prob{x|Y=\rmp}} \\ 
        {\prob{x|Y=\rmn}}
    }
    \, \dx \nonumber \\
    &\stackrel{\text{(a)}}{=}&
    \int_{x\in\mathcal{X}} \int_{x'\in\mathcal{X}}  
    \prob{\mrm{S}} \corr{\ell}_{\tilde{\mrm{S}}}
    \, \dx' \dx
    +
    \int_{x \in \mathcal{X}} \prob{\mrm{U}} \, \corr{\ell}_{\mrm{U}} \, \dx \nonumber \\
    &=&
    \expt{\mrm{S}}{\corr{\ell}_{\tilde{\mrm{S}}}} + \expt{\mrm{U}}{\corr{\ell}_{\mrm{U}}}, \nonumber
}
where equality (a) follows from the SU formulation (\ref{eq:formulate_SU}).

Then, denoting 
\eqarr{
    \mathcal{L}(X) 
    &:=& 
    \frac{1}{\pi_\rmp - \pi_\rmn} \ell_{\rmp}(X) - \frac{1}{\pi_\rmp - \pi_\rmn} \ell_{\rmn}(X), \label{eq:big_L} \\
    \mathcal{L}_{-}(X) 
    &:=& 
    - \frac{\pi_\rmn}{\pi_\rmp -\pi_\rmn} \ell_{\rmp}(X) + \frac{\pi_\rmp}{\pi_\rmp - \pi_\rmn} \ell_{\rmn}(X) \label{eq:big_L_minus}
}
and continuing with (\ref{eq:calibrated_losses_SU}), we obtain
\eqarr{
    \expt{\mrm{S}}{\corr{\ell}_{\tilde{\mrm{S}}}} 
    &=&
    \left( \pi_\rmp^2+\pi_\rmn^2 \right) \expt{\mrm{S}}{\frac{1}{2\pi_\rmp -1} \left( \ell_\rmp - \ell_\rmn \right)} \nonumber \\
    &=&
    \left( \pi_\rmp^2+\pi_\rmn^2 \right) \expt{\mrm{S}}{\mathcal{L}(X)} \nonumber \\
    &\stackrel{\text{(b)}}{=}& 
    \left( \pi_\rmp^2+\pi_\rmn^2 \right) \expt{\mrm{S}}{\frac{\mathcal{L}(X) + \mathcal{L}(X')}{2}} \nonumber
}
and
\eqarr{
    \expt{\mrm{U}}{\corr{\ell}_{\mrm{U}}} 
    &=&
    \expt{\mrm{U}}{-\frac{\pi_\rmn}{2\pi_\rmp -1} \ell_\rmp + \frac{\pi_\rmp}{2\pi_\rmp -1} \ell_\rmn} \nonumber \\
    &=&
    \expt{\mrm{U}}{\mathcal{L}_{-}(f(X))} \label{eq:unlabeled_rewrite_SU}
}
that prove rewrite (\ref{eq:review_rewrite_SU}) in Section~\ref{sec:review_SU} and recover Theorem 1 of \cite{su_18/Bao/NS/18} by matching notations\footnote{
The matching to the notations of \cite{su_18/Bao/NS/18} is as follows: 
$\pi_\rmp$ is $\pi_+$, 
$\pi_\rmn$ is $\pi_-$, 
$\pi_\rmp^2+\pi_\rmn^2$ is $\pi_{\mrm{S}}$, 
$\prob{\mrm{S}}$ is $p_{\mrm{S}}$,
$\prob{\mrm{U}}$ is $p$,
$\ell_\rmp$ is $\ell(f(X),+1)$, 
$\ell_\rmn$ is $\ell(f(X),-1)$, 
$\mathcal{L}(X)$ by definition is $\frac{1}{2\pi_{+} -1} \left( \ell(f(X),+1) - \ell(f(X),-1) \right)$, and 
$\mathcal{L}_{-}(f(X))$ by definition is $-\frac{\pi_{-}}{2\pi_{+} -1} \ell(f(X),+1) + \frac{\pi_{+}}{2\pi_{+} -1} \ell(f(X),-1)$.
}.
The following lemma justifies equality (b).
\begin{lemma}
\label{lma:symmetric_similar}
Let $(x,x') \sim \prob{\mrm{S}}$ defined by (\ref{eq:formulate_SU}).
Then,
$
\expt{\mrm{S}}{\frac{\mathcal{L}(X)}{2}}
=
\expt{\mrm{S}}{\frac{\mathcal{L}(X')}{2}}.
$
\end{lemma}
\begin{proof}
For clarity, we simplify $\prob{\mrm{S}}$ as $c_1 \prob{X|Y=\rmp} \prob{X'|Y=\rmp} + c_2 \prob{X|Y=\rmn} \prob{X'|Y=\rmn}$, with $c_1 = \frac{\pi_\rmp^2}{\pi_\rmp^2+\pi_\rmn^2}$ and $c_2 = \frac{\pi_\rmn^2}{\pi_\rmp^2+\pi_\rmn^2}$.
The lemma follows from
\eqarr{
    && \expt{\mrm{S}}{\mathcal{L}(X)} \nonumber \\
    && =
    \int_{x\in\mathcal{X}} \int_{x'\in\mathcal{X}}
    \prob{\mrm{S}} \mathcal{L}(x)
    \; \dx' \; \dx \nonumber \\
    && =
    \int_{x\in\mathcal{X}} \int_{x'\in\mathcal{X}}
    \left( c_1 \prob{x|Y=\rmp} \prob{x'|Y=\rmp} + c_2 \prob{x|Y=\rmn} \prob{x'|Y=\rmn} \right) \mathcal{L}(x)
    \; \dx' \; \dx \nonumber \\
    && = 
    c_1 \int_{x\in\mathcal{X}} \prob{x|Y=\rmp} \mathcal{L}(x) \; \dx \int_{x'\in\mathcal{X}} \prob{x'|Y=\rmp} \; \dx'
    +
    c_2 \int_{x\in\mathcal{X}} \prob{x|Y=\rmn} \mathcal{L}(x) \; \dx \int_{x'\in\mathcal{X}} \prob{x'|Y=\rmn} \; \dx' \nonumber \\
    && =
    c_1 \int_{x\in\mathcal{X}} \prob{x|Y=\rmp} \mathcal{L}(x) \; \dx 
    +
    c_2 \int_{x\in\mathcal{X}} \prob{x|Y=\rmn} \mathcal{L}(x) \; \dx, \nonumber
}
and similarly,
\eqarr{
    \expt{\mrm{S}}{\mathcal{L}(X')}
    =
    c_1 \int_{x'\in\mathcal{X}} \prob{x'|Y=\rmp} \mathcal{L}(x') \; \dx' 
    +
    c_2 \int_{x'\in\mathcal{X}} \prob{x'|Y=\rmn} \mathcal{L}(x') \; \dx'. \nonumber
}
\end{proof}

\explain 
The analysis demonstrates the flexibility of the proposed framework in which a slight modification of $\Mcorr{SU}$ recovers the pairwise distribution $\prob{\mrm{S}}$ required for $\expt{\mrm{S}}{\cdot}$.
Moreover, the technique developed here significantly reduces the proof in Appendix B of \cite{su_18/Bao/NS/18}.
Later in Section~\ref{sec:GE2_SDU}, we apply the same trick to recover Theorem 1 of \cite{sdu_19/Shimada/BSS/21} for SDU learning.

\explain
We remark that the result recovered in this paper is merely Theorem 1 of \cite{su_18/Bao/NS/18} but not the last expression in (5) of \cite{su_18/Bao/NS/18}, which later was implemented as the objective (10) for optimization.
It is because, pointed out by \cite{23_UUU}, the additional assumption $\prob{\mrm{S}}^{(x,x')} = \prob{\tilde{\mrm{S}}}^{(x)} \prob{\tilde{\mrm{S}}}^{(x')}$ required for achieving (5) of \cite{su_18/Bao/NS/18} is impractical.
We note that the remedy proposed by \cite{23_UUU} can be analyzed by the proposed framework, but we omit it due to the amount of overlap with the analyses in Sections~\ref{sec:GE2_UU} and \ref{sec:GE2_SU}.

\subsubsection{Pairwise Comparison (Pcomp) Learning}
\label{sec:GE2_Pcomp}
We follow the steps illustrated in Section~\ref{sec:GE2_PU} to justify the proposed framework since, by Table~\ref{tab:MCD_matrices_summary}, $\Mcorr{Pcomp}$ is reduced from $\Mcorr{UU}$.

\paragraph{Step 1: Corrected Loss Design and Risk Rewrite.}~\\
\indent
The following corollary combines (\ref{eq:calibrated_loss_UU}) and (\ref{eq:uu_risk_rewrite_a}) to achieve rewrite (\ref{eq:review_rewrite_Pcomp}) in Section~\ref{sec:review_Pcomp}.
\begin{corollary}
\label{thm:rewrite_Pcomp}
For Pcomp learning, the classification risk can be rewritten as
\eqarr{
    R(g)
    =
    \expt{\mrm{Sup}}{\corr{\ell}_{\mrm{Sup}}} + \expt{\mrm{Inf}}{\corr{\ell}_{\mrm{Inf}}}, \label{eq:risk_rewrite_Pcomp} \nonumber
}
where
\eqarr{
    \corr{\ell}_{\mrm{Sup}} &=& \ell_\rmp - \pi_\rmp \ell_\rmn, \nonumber \\
    \corr{\ell}_{\mrm{Inf}} &=& -\pi_\rmn \ell_\rmp + \ell_\rmn. \label{eq:calibrated_losses_Pcomp}
}
\end{corollary}
\begin{proof}
Since $\Mcorr{UU}$ reduces to $\Mcorr{Pcomp}$ with $\mcdp = \frac{\pi_\rmn^2}{\pi_\rmp + \pi_\rmn^2}$ and $\mcdn = \frac{\pi_\rmp^2}{\pi_\rmp^2 + \pi_\rmn}$, we replace the subscripts $\{\mrm{U}_1, \mrm{U}_2\}$ with $\{\mrm{Sup}, \mrm{Inf}\}$ and instantiate (\ref{eq:calibrated_loss_UU}) with $\mcdp$ and $\mcdn$ to obtain the corrected losses $\corr{\ell}_{\mrm{Sup}}$ and $\corr{\ell}_{\mrm{Inf}}$.
Then, repeating the same steps of proving (\ref{eq:uu_risk_rewrite_a}), we have the corollary.
\end{proof}

\paragraph{Step 2: Recovering the previous result(s).}~\\
\indent
It is straightforward to recover Theorem 3 of \cite{pcomp_20/Feng/SLS/21} by matching notations\footnote{
The matching is as follows:
$\prob{\mrm{Sup}}$ is $\tilde{p}_{+}(x)$,
$\prob{\mrm{Inf}}$ is $\tilde{p}_{-}(x)$,
$\ell_\rmp$ is $\ell(f(x),+1)$, and
$\ell_\rmn$ is $\ell(f(x),-1)$.
}.
Since $x$ is a variable and can be substituted by $x'$, we express Corollary~\ref{thm:rewrite_Pcomp} as 
\eqarr{
    R(g)
    = 
    \expt{x\sim\prob{\mrm{Sup}}}{\ell_\rmp(x) - \pi_\rmp \ell_\rmn(x)} 
    + 
    \expt{x'\sim\prob{\mrm{Inf}}}{\ell_\rmn(x') - \pi_\rmn \ell_\rmp(x')},
}
recovering (5) of \cite{pcomp_20/Feng/SLS/21}.

\subsubsection{Similar-dissimilar-unlabeled (SDU) Learning}
\label{sec:GE2_SDU}
We justify the applicability of the proposed framework for DU and SD separately.
Firstly, we start with DU learning, which is similar to SU learning in the sense that pairwise information is provided. 
From Lemmas~\ref{lma:formulate_SU} and \ref{lma:formulate_DU}, we see that the pairwise distributions are treated similarly. 
Thus, following the same steps in Section~\ref{sec:GE2_SU}, we conduct the risk rewrite for DU learning.

\paragraph{Step 1: Corrected Loss Design and Risk Rewrite for DU Learning.}~\\
\indent
The following corollary is a variant of Corollary~\ref{thm:rewrite_SU}.
\begin{corollary}
\label{thm:rewrite_DU}
    Assume $\pi_\rmp \neq 1/2$.
    For DU learning, the classification risk can be rewritten as
    \eqarr{
        R(g)
        =
        \expt{\tilde{\mrm{D}}}{\corr{\ell}_{\tilde{\mrm{D}}}} + \expt{\mrm{U}}{\corr{\ell}_{\mrm{U}}},
    }
    where
    \eqarr{
        \corr{\ell}_{\tilde{\mrm{D}}} &=& 2\pi_\rmp\pi_\rmn \left( \frac{1}{\pi_\rmn-\pi_\rmp} \ell_\rmp - \frac{1}{\pi_\rmn-\pi_\rmp} \ell_\rmn \right), \nonumber \\
        \corr{\ell}_{\mrm{U}} &=&  -\frac{\pi_\rmp}{\pi_\rmn-\pi_\rmp} \ell_\rmp + \frac{\pi_\rmn}{\pi_\rmn-\pi_\rmp} \ell_\rmn. \label{eq:calibrated_losses_DU}
    }
\end{corollary}
\begin{proof}
One can prove Corollary~\ref{thm:rewrite_DU} by repeating the proof for Corollary~\ref{thm:rewrite_SU} with $\{\mrm{U}_1, \mrm{U}_2\}$ replaced by $\{ \tilde{\mrm{D}}, \mrm{U}\}$, $\mcdp = 1/2$, and $\mcdn = \pi_\rmp$.
Note that $\pi_\rmp \neq 1/2$ implies that the $\mcdp$ and $\mcdn$ assignments are feasible.
\end{proof}

\paragraph{Step 2: Recovering the previous result(s) for DU Learning.}~\\
\indent
We reuse the trick in Lemma~\ref{lma:recover_SU_1} for restoring the pairwise distribution $\prob{\mrm{S}}$ to restore $\prob{\mrm{D}}$ needed here, allowing us to recover the rewrite (15) in Theorem 1 of \cite{sdu_19/Shimada/BSS/21} and the first result in Theorem 7.3 of \cite{wsl_sugibook/Sugiyama/BILSG/22}. 
The derivation resembles that of SU learning.
We start with the next lemma, revised from Lemma~\ref{lma:recover_SU_1}.
\begin{lemma}
\label{lma:recover_DU_1}
    Given $B$ (\ref{eq:base_dist_MCD}) and following the DU learning notations, we have $\corr{P} = \corr{P}'$, where
    \eqarr{
        \corr{P}'
        &=&
        \Mcorr{DU}' B, \nonumber \\
        \Mcorr{DU}'
        &:=&
        \mmatrix{
            \frac{\int_{x'\in\mathcal{X}} \prob{x'|Y=\rmn} \dx'}{2} & \frac{\int_{x'\in\mathcal{X}} \prob{x'|Y=\rmp} \dx'}{2} \\
            \pi_\rmp & \pi_\rmn
        }. \nonumber
    }
\end{lemma}
\begin{proof}
    Following the same argument in Lemma~\ref{lma:recover_SU_1}, we have $\Mcorr{DU}' = \Mcorr{DU}$ and hence the Lemma. 
\end{proof}
We apply Lemma~\ref{lma:recover_DU_1} to slightly revise the derivation of (\ref{eq:uu_risk_rewrite}) as follows:
\eqarr{
    R(g)
    &=&
    \int_{x\in\mathcal{X}} \corr{L}^{\top} \corr{P} \, \dx 
    =
    \int_{x\in\mathcal{X}} \corr{L}^{\top} \corr{P}' \, \dx \nonumber \\
    &=&
    \int_{x\in\mathcal{X}} 
    \mmatrix{
        \corr{\ell}_{\tilde{\mrm{D}}} & \corr{\ell}_{\mrm{U}}
    }
    \mmatrix{
        \frac{\int_{x'\in\mathcal{X}} \prob{x'|Y=\rmn} \dx'}{2} & \frac{\int_{x'\in\mathcal{X}} \prob{x'|Y=\rmp} \dx'}{2} \\
        \pi_\rmp & \pi_\rmn
    }
    \mmatrix{
        \prob{x|Y=\rmp} \\ 
        \prob{x|Y=\rmn}
    }
    \, \dx \nonumber \\
    &=&
    \int_{x\in\mathcal{X}} \int_{x'\in\mathcal{X}} \prob{\mrm{D}} \, \corr{\ell}_{\tilde{\mrm{D}}} \, \dx' \dx
    + \int_{x\in\mathcal{X}} \prob{\mrm{U}} \, \corr{\ell}_{\mrm{U}} \, \dx \nonumber \\
    &=&
    \expt{\mrm{D}}{\corr{\ell}_{\tilde{\mrm{D}}}} + \expt{\mrm{U}}{\corr{\ell}_{\mrm{U}}}, \nonumber
}
where the second to last equality follows from the DU formulation (\ref{eq:formulate_DU}).
Denoting
\eqarr{
    \mathcal{L}_{+}(X) := \frac{\pi_\rmp}{\pi_\rmp -\pi_\rmn} \ell_{\rmp}(X) - \frac{\pi_\rmn}{\pi_\rmp - \pi_\rmn} \ell_{\rmn}(X), \label{eq:big_L_plus}
}
recalling $\mathcal{L}(X)$ from (\ref{eq:big_L}), and continuing with (\ref{eq:calibrated_losses_DU}), we have
\eqarr{
    \expt{\mrm{D}}{\corr{\ell}_{\tilde{\mrm{D}}}}
    &=&
    2\pi_\rmp\pi_\rmn \expt{\mrm{D}}{\frac{1}{\pi_\rmn-\pi_\rmp} \ell_\rmp - \frac{1}{\pi_\rmn-\pi_\rmp} \ell_\rmn} \nonumber \\
    &=&
    2\pi_\rmp\pi_\rmn \expt{\mrm{D}}{-\mathcal{L}(X)} \nonumber \\
    &\stackrel{\text{(a)}}{=}&
    2\pi_\rmp\pi_\rmn \expt{\mrm{D}}{-\frac{\mathcal{L}(X)+\mathcal{L}(X')}{2}} \nonumber
}
and
\eqarr{
    \expt{\mrm{U}}{\corr{\ell}_{\mrm{U}}}
    &=&
    \expt{\mrm{U}}{-\frac{\pi_\rmp}{\pi_\rmn-\pi_\rmp} \ell_\rmp + \frac{\pi_\rmn}{\pi_\rmn-\pi_\rmp} \ell_\rmn} \nonumber \\
    &=&
    \expt{\mrm{U}}{\mathcal{L}_{+}(X)} \label{eq:unlabeled_rewrite_DU}
}
that prove rewrite (\ref{eq:review_rewrite_DU}) in Section~\ref{sec:review_DU}.
By matching notations, we recover (15) in Theorem 1 of \cite{sdu_19/Shimada/BSS/21}
\footnote{
The matching to the notations of \cite{sdu_19/Shimada/BSS/21} is as follows: 
$\pi_\rmp$ is $\pi_+$, 
$\pi_\rmn$ is $\pi_-$, 
$\pi_\rmp^2+\pi_\rmn^2$ is $\pi_{\mrm{S}}$, 
$2\pi_\rmp\pi_\rmn$ is $\pi_{\mrm{D}}$, 
$\prob{\mrm{S}}$ is $p_{\mrm{S}}(x,x')$,
$\prob{\mrm{D}}$ is $p_{\mrm{D}}(x,x')$,
$\prob{\mrm{U}}$ is $p_{\mrm{U}}(x)$,
$\ell_\rmp$ is $\ell(f(X),+1)$, 
$\ell_\rmn$ is $\ell(f(X),-1)$,
$\mathcal{L}(X)$ is $\tilde{\mathcal{L}}(f(X))$,
$\mathcal{L}_{+}(X)$ is $\mathcal{L}(f(X),+1)$, and
$\mathcal{L}_{-}(X)$ is $\mathcal{L}(f(X),-1)$.
}
.
Equality (a) follows from the next lemma.
\begin{lemma}
\label{lma:symmetric_dissimilar}
    Let $(x, x') \sim \prob{\mrm{D}}$ defined in (\ref{eq:formulate_DU}).
    Then,
    $
    \expt{\mrm{D}}{\frac{\mathcal{L}(X)}{2}}
    =
    \expt{\mrm{D}}{\frac{\mathcal{L}(X')}{2}}.
    $
\end{lemma}
\begin{proof}
Recall 
$
\prob{\mrm{D}}=
\frac{1}{2} ( \prob{x|Y=\rmp} \prob{x'|Y=\rmn} + \prob{x|Y=\rmn} \prob{x'|Y=\rmp}).
$
Following the similar argument in Lemma~\ref{lma:symmetric_similar},
\eqarr{
    \expt{\mrm{D}}{\frac{\mathcal{L}(X)}{2}}
    &=&
    \int_{x\in\mathcal{X}} \int_{x'\in\mathcal{X}}
    \left( \prob{x|Y=\rmp} \prob{x'|Y=\rmn} + \prob{x|Y=\rmn} \prob{x'|Y=\rmp} \right) \frac{\mathcal{L}(x)}{4}
    \; \dx' \dx \nonumber \\
    &=&
    \int_{x\in\mathcal{X}}
    \left( \prob{x|Y=\rmp} + \prob{x|Y=\rmn} \right) \frac{\mathcal{L}(x)}{4} \; \dx \nonumber
}
and
\eqarr{
    \expt{\mrm{D}}{\frac{\mathcal{L}(X')}{2}}
    &=&
    \int_{x'\in\mathcal{X}}
    \left( \prob{x'|Y=\rmn} + \prob{x'|Y=\rmp} \right) \frac{\mathcal{L}(x')}{4} \; \dx' \nonumber
}
prove the lemma.
\end{proof}

Secondly, we consider the rewrite of SD learning.
To do so, we apply the knowledge acquired from SU and DU learning (Corollaries~\ref{thm:rewrite_SU} and \ref{thm:rewrite_DU}).
\paragraph{Step 1: Corrected Loss Design and Risk Rewrite for SD Learning.}~\\
We provide another variant of Corollary~\ref{thm:rewrite_SU} to conduct the risk rewrite.
\begin{corollary}
\label{thm:rewrite_SD}
    Assume $\pi_\rmp \neq 1/2$.
    For SD learning, the classification risk can be rewritten as
    \eqarr{
        R(g)
        =
        \expt{\tilde{\mrm{S}}}{\corr{\ell}_{\tilde{\mrm{S}}}} + \expt{\tilde{\mrm{D}}}{\corr{\ell}_{\tilde{\mrm{D}}}},
    }
    where
    \eqarr{
        \corr{\ell}_{\tilde{\mrm{S}}} &=& \left( \pi_\rmp^2 + \pi_\rmn^2 \right) \left( \frac{\pi_\rmp}{\pi_\rmp-\pi_\rmn} \ell_\rmp - \frac{\pi_\rmn}{\pi_\rmp-\pi_\rmn} \ell_\rmn \right), \nonumber \\
        \corr{\ell}_{\tilde{\mrm{D}}} &=& 2\pi_\rmp\pi_\rmn \left( -\frac{\pi_\rmn}{\pi_\rmp-\pi_\rmn} \ell_\rmp + \frac{\pi_\rmp}{\pi_\rmp-\pi_\rmn} \ell_\rmn \right). \label{eq:calibrated_losses_SD}
    }
\end{corollary}
\begin{proof}
One can prove Corollary~\ref{thm:rewrite_SD} by repeating the proof for Corollary~\ref{thm:rewrite_SU} with $\{\mrm{U}_1, \mrm{U}_2\}$ replaced by $\{ \tilde{\mrm{S}} , \tilde{\mrm{D}} \}$, $\mcdp = \frac{\pi_\rmn^2}{\pi_\rmp^2+\pi_\rmn^2}$, and $\mcdn = 1/2$.
Note that $\pi_\rmp \neq 1/2$ implies that the $\mcdp$ and $\mcdn$ assignments are feasible.
\end{proof}

\paragraph{Step 2: Recovering the previous result(s) for SD Learning.}~\\
\indent
We apply the same strategy as in Lemma~\ref{lma:recover_DU_1} to obtain the needed $\prob{\mrm{S}}$ and $\prob{\mrm{D}}$. 
We begin with the next lemma, revised from Lemma~\ref{lma:recover_SU_1}, to recover (16) in Theorem 1 of \cite{sdu_19/Shimada/BSS/21} and the second result in Theorem 7.3 of \cite{wsl_sugibook/Sugiyama/BILSG/22}.
\begin{lemma}
\label{lma:recover_SD_1}
    Given $B$ (\ref{eq:base_dist_MCD}) and following the SD learning notations, we have $\corr{P} = \corr{P}'$, where
    \eqarr{
        \corr{P}'
        &=&
        \Mcorr{SD}' B, \nonumber \\
        \Mcorr{SD}'
        &:=&
        \mmatrix{
            \frac{\pi_\rmp^2\int_{x'\in\mathcal{X}} \prob{x'|Y=\rmp}\dx'}{\pi_\rmp^2+\pi_\rmn^2} & \frac{\pi_\rmn^2\int_{x'\in\mathcal{X}} \prob{x'|Y=\rmn}\dx'}{\pi_\rmp^2+\pi_\rmn^2} \\
            \frac{\int_{x'\in\mathcal{X}} \prob{x'|Y=\rmn}\dx'}{2} & \frac{\int_{x'\in\mathcal{X}} \prob{x'|Y=\rmp}\dx'}{2}
        }. \nonumber
    }
\end{lemma}
\begin{proof}
    Following the same argument in Lemma~\ref{lma:recover_SU_1}, we have $\Mcorr{SD}' = \Mcorr{SD}$ and hence the Lemma. 
\end{proof}
We apply Lemma~\ref{lma:recover_SD_1} to slightly revise the derivation of (\ref{eq:uu_risk_rewrite}) as follows:
\eqarr{
    R(g)
    &=&
    \int_{x\in\mathcal{X}} \corr{L}^{\top} \corr{P} \, \dx 
    =
    \int_{x\in\mathcal{X}} \corr{L}^{\top} \corr{P}' \, \dx \nonumber \\
    &=&
    \int_{x\in\mathcal{X}} 
    \mmatrix{
        \corr{\ell}_{\tilde{\mrm{S}}} & \corr{\ell}_{\tilde{\mrm{D}}}
    }
    \mmatrix{
        \frac{\pi_\rmp^2\int_{x'\in\mathcal{X}} \prob{x'|Y=\rmp}\dx'}{\pi_\rmp^2+\pi_\rmn^2} & \frac{\pi_\rmn^2\int_{x'\in\mathcal{X}} \prob{x'|Y=\rmn}\dx'}{\pi_\rmp^2+\pi_\rmn^2} \\
        \frac{\int_{x'\in\mathcal{X}} \prob{x'|Y=\rmn}\dx'}{2} & \frac{\int_{x'\in\mathcal{X}} \prob{x'|Y=\rmp}\dx'}{2}
    }
    \mmatrix{
        \prob{x|Y=\rmp} \\ 
        \prob{x|Y=\rmn}
    }
    \, \dx \nonumber \\
    &=&
    \int_{x\in\mathcal{X}} \int_{x'\in\mathcal{X}} \prob{\mrm{S}} \, \corr{\ell}_{\tilde{\mrm{S}}} \, \dx' \dx
    + \int_{x\in\mathcal{X}} \int_{x'\in\mathcal{X}'} \prob{\mrm{D}} \, \corr{\ell}_{\tilde{\mrm{D}}} \, \dx' \dx \nonumber \\
    &=&
    \expt{\mrm{S}}{\corr{\ell}_{\tilde{\mrm{S}}}} + \expt{\mrm{D}}{\corr{\ell}_{\tilde{\mrm{D}}}}, \nonumber
}
where the second to last equality follows from the SD formulation (\ref{eq:formulate_SD}).
Recalling $\mathcal{L}_{+}(X)$ (\ref{eq:big_L_plus}) and $\mathcal{L}_{-}(X)$ (\ref{eq:big_L_minus}) and continuing with (\ref{eq:calibrated_losses_SD}),
\eqarr{
    \expt{\mrm{S}}{\corr{\ell}_{\tilde{\mrm{S}}}}
    &=&
    \left( \pi_\rmp^2 + \pi_\rmn^2 \right) \expt{\mrm{S}}{\frac{\pi_\rmp}{\pi_\rmp-\pi_\rmn} \ell_{\rmp} - \frac{\pi_\rmn}{\pi_\rmp-\pi_\rmn} \ell_{\rmn} } \nonumber \\
    &=&
    \left( \pi_\rmp^2 + \pi_\rmn^2 \right) \expt{\mrm{S}}{\mathcal{L}_{+}(X)} \nonumber \\
    &\stackrel{\text{(b)}}{=}&
    \left( \pi_\rmp^2 + \pi_\rmn^2 \right) \expt{\mrm{S}}{\frac{\mathcal{L}_{+}(X)+\mathcal{L}_{+}(X')}{2}} \label{eq:recover_SD_1st_term} \nonumber
}
and
\eqarr{
    \expt{\mrm{D}}{\corr{\ell}_{\tilde{\mrm{D}}}}
    &=&
    2\pi_\rmp\pi_\rmn \expt{\mrm{D}}{-\frac{\pi_\rmn}{\pi_\rmp-\pi_\rmn} \ell_\rmp + \frac{\pi_\rmp}{\pi_\rmp-\pi_\rmn} \ell_\rmn} \nonumber \\
    &=&
    2\pi_\rmp\pi_\rmn \expt{\mrm{D}}{\mathcal{L}_{-}(X)} \nonumber \\
    &\stackrel{\text{(c)}}{=}&
    2\pi_\rmp\pi_\rmn \expt{\mrm{D}}{\frac{\mathcal{L}_{-}(X)+\mathcal{L}_{-}(X')}{2}} \nonumber
}
prove rewrite (\ref{eq:review_rewrite_SD}) in Section~\ref{sec:review_SD}.
We also recover (16) in Theorem 1 of \cite{sdu_19/Shimada/BSS/21} via matching notations. 
The required matches can be found in the paragraph before Lemma~\ref{lma:symmetric_dissimilar}.
The equality (b) holds by applying Lemma~\ref{lma:symmetric_similar} with $\mathcal{L}(X)$ replaced by $\mathcal{L}_{+}(X)$, and (c) follows from Lemma~\ref{lma:symmetric_dissimilar} with $\mathcal{L}(X)$ replaced by $\mathcal{L}_{-}(X)$.

\compare
An intriguing observation worth mentioning is that the losses $\mathcal{L}_{+}(X)$ and $\mathcal{L}_{-}(X)$ applied to decontaminate the unlabeled data in SU and DU learning ((\ref{eq:unlabeled_rewrite_SU}) and (\ref{eq:unlabeled_rewrite_DU})) are now used to decontaminate the similar and the dissimilar data in SD learning, respectively. 
One can also quickly draw the same conclusion from Table~\ref{tab:binary_WSL_rewrites}.
Knowing the reason behind this observation would help to transfer one corrected loss developed in one scenario to another weakly supervised scenario.

\subsubsection{Similarity-Confidence (Sconf) Learning}
\label{sec:GE2_Sconf}
Since $\Mcorr{Sconf}$ (\ref{eq:MSconf}) is not a child of $\Mcorr{UU}$ (\ref{eq:MUU}) on the reduction graph, a direct application of Theorem~\ref{thm:UU_rewrite_corrected_losses} is infeasible.
Nevertheless, we demonstrate how our framework is applied to rewrite the classification risk for Sconf learning.
We make a small adjustment to the framework that instead of showing $\corr{L}^{\top} \corr{P} = L^{\top} M^{\dagger} \corr{P} = L^{\top} P$, we hope that for an arbitrary loss vector $\mathcal{L}$,
\eqarr{
    \int_{x'\in\mathcal{X}} \mathcal{L}^{\top} \corr{P} \; \dx'
    =
    \mathcal{L}^{\top} \tilde{M}_{\mrm{Sconf}} P. \label{eq:key_Sconf_2}
}
The idea behind this approach is to accommodate $x'$ sampled from $\prob{X'}$ (\ref{eq:formulate_Sconf}). 
Suppose, informally, we have the equation above. 
Then, the right-hand side of (\ref{eq:key_Sconf_2}) will produce $L^{\top} P$ if we can compute a decontamination matrix $\tilde{M}_{\mrm{Sconf}}^{\dagger}$ and assign $\mathcal{L}^{\top} = \corr{L}^{\top} := L^{\top} \tilde{M}_{\mrm{Sconf}}^{\dagger}$.
In this way, the left-hand side of (\ref{eq:key_Sconf_2}) will become $\int_{x'} \corr{L}^{\top} \corr{P} \dx'$.
Therefore, integrating over $x$ on both sides, we obtain the key equation
\eqarr{
    \int_{x\in\mathcal{X}}\int_{x'\in\mathcal{X}} \corr{L}^{\top} \corr{P} \; \dx' \dx
    =
    \int_{x\in\mathcal{X}} L^{\top} P \; \dx \nonumber
}
for risk rewrite.

\paragraph{Step 1: Corrected Loss Design and Risk Rewrite.}~\\
\indent
Let us follow the notations in Section~\ref{sec:GE1_Sconf}.
We begin with two technical lemmas and defer their proofs at the end of this sub-subsection.
The first technical lemma shows how to achieve (\ref{eq:key_Sconf_2}).
\begin{lemma}
\label{lma:key_Sconf}
    Assume the formulation (\ref{eq:Sconf_corruption_mechanism}) is given.
    Suppose a vector of corrected losses $\mathcal{L}_{x}^{\top}$ of the form $\mmatrix{\tilde{\ell}_1(x) & \tilde{\ell}_2(x)}$ depends only on $x$.
    Then, we have
    \eqarr{
        \int_{x'\in\mathcal{X}} \mathcal{L}_{x}^{\top} \corr{P} \; \dx'
        =
        \mathcal{L}_{x}^{\top} \tilde{M}_{\mrm{Sconf}} P, \label{eq:key_Sconf}
    }
    where
    \eqarr{
        \tilde{M}_{\mrm{Sconf}}
        =
        \mmatrix{
            \int_{x'} \frac{ \pi_\rmp^2\prob{x'|\rmp} - \pi_\rmn^2\prob{x'|\rmn} }{r-\pi_\rmn} \dx' & \int_{x'} \frac{ \pi_\rmn^2\prob{x'|\rmn} - \pi_\rmn^2\prob{x'|\rmp} }{r-\pi_\rmn} \dx' \\
            \int_{x'} \frac{ \pi_\rmp^2\prob{x'|\rmn} - \pi_\rmp^2\prob{x'|\rmp} }{\pi_\rmp - r} \dx' & \int_{x'} \frac{ \pi_\rmp^2\prob{x'|\rmp} - \pi_\rmn^2\prob{x'|\rmn} }{\pi_\rmp - r} \dx'
        }. \nonumber
    }
\end{lemma}
The second technical lemma computes the decontamination matrix.
\begin{lemma}
\label{lma:inv_MSconf}
    Let 
    \eqarr{
        \tilde{M}_{\mrm{Sconf}}^{\dagger}
        :=
        \mmatrix{
            \frac{r- \pi_\rmn}{\pi_\rmp - \pi_\rmn} & 0 \\
            0 & \frac{\pi_\rmp - r}{\pi_\rmp - \pi_\rmn}
        }. \nonumber
    }
    Then,
    \eqarr{
        \tilde{M}_{\mrm{Sconf}}^{\dagger} \tilde{M}_{\mrm{Sconf}} = I. \nonumber
    }
\end{lemma}
Then, we follow the informal sketch above to instantiate $\mathcal{L}_{x}^{\top}$ as 
\eqarr{
    \corr{L}^{\top} 
    := 
    L^{-1} \tilde{M}_{\mrm{Sconf}}^{\dagger}
    =
    \mmatrix{
        \frac{r- \pi_\rmn}{\pi_\rmp - \pi_\rmn} \ell_{\rmp}(X) &
        \frac{\pi_\rmp - r}{\pi_\rmp - \pi_\rmn} \ell_{\rmn}(X)
    }. \nonumber
}

Putting $\tilde{M}_{\mrm{Sconf}}^{\dagger}$, $\corr{L}$, and (\ref{eq:key_Sconf}) together, we have the following rewrite.
\begin{theorem}
\label{thm:rewrite_Sconf}
    Assume $\pi_\rmp \neq 1/2$.
    The classification risk of Sconf learning can be expressed by
    \eqarr{
        R(g)
        =
        \expt{X,X'}{\frac{r- \pi_\rmn}{\pi_\rmp - \pi_\rmn} \ell_{\rmp}(X) + \frac{\pi_\rmp - r}{\pi_\rmp - \pi_\rmn} \ell_{\rmn}(X)}. \label{eq:rewrite_Sconf_1}
    }
\end{theorem}
\begin{proof}
    Integrating both sides of (\ref{eq:key_Sconf}) over $x$ and applying Lemma~\ref{lma:inv_MSconf}, we obtain
    \eqarr{
        \int_{x\in\mathcal{X}} \int_{x'\in\mathcal{X}} \corr{L}^{\top} \corr{P} \; \dx' \dx
        &=&
        \int_{x\in\mathcal{X}} \corr{L}^{\top} \tilde{M}_{\mrm{Sconf}} P \; \dx \nonumber \\
        &=&
        \int_{x\in\mathcal{X}} L^{\top} \tilde{M}_{\mrm{Sconf}}^{-1} \tilde{M}_{\mrm{Sconf}} P \; \dx
        =
        R(g). \nonumber
    }
    On the other hand, substituting $\corr{L}$ with
    $
    \mmatrix{
        \frac{r- \pi_\rmn}{\pi_\rmp - \pi_\rmn} \ell_{\rmp}(X) &
        \frac{\pi_\rmp - r}{\pi_\rmp - \pi_\rmn} \ell_{\rmn}(X)
    }
    $ 
    and $\corr{P}$ with 
    $
    \mmatrix{
        \prob{X} \prob{X'} \\
        \prob{X} \prob{X'}
    }
    $,
    \eqarr{
        \int_{x\in\mathcal{X}} \int_{x'\in\mathcal{X}} \corr{L}^{\top} \corr{P} \; \dx' \dx
        &=&
        \int_{x\in\mathcal{X}} \int_{x'\in\mathcal{X}} \prob{x}\prob{x'} \left( \frac{r- \pi_\rmn}{\pi_\rmp - \pi_\rmn} \ell_{\rmp}(x) + \frac{\pi_\rmp - r}{\pi_\rmp - \pi_\rmn} \ell_{\rmn}(x) \right) \; \dx' \dx \nonumber \\
        &=&
        \expt{X,X'}{\frac{r- \pi_\rmn}{\pi_\rmp - \pi_\rmn} \ell_{\rmp}(X) + \frac{\pi_\rmp - r}{\pi_\rmp - \pi_\rmn} \ell_{\rmn}(X) } \nonumber
    }
    completes the proof of the theorem.
\end{proof}

\paragraph{Step 2: Recovering the previous result(s).}~\\
\indent
From the above derivation, we have achieved the first half of the rewrite in (\ref{eq:review_rewrite_Sconf}).
Notice that (\ref{eq:Pconf_corruption1}) can be rephrased as 
\eqarr{
    \left( \frac{r-\pi_\rmn}{\pi_\rmp} \right) \prob{X} \prob{X'}
    = 
    \left( \pi_\rmp^2\prob{X|\rmp} - \pi_\rmn^2\prob{X|\rmn} \right) \prob{X'|\rmp} 
    + 
    \left( \pi_\rmn^2\prob{X|\rmn} - \pi_\rmn^2\prob{X|\rmp} \right) \prob{X'|\rmn} \nonumber
}
and that (\ref{eq:Pconf_corruption2}) can be rephrased as
\eqarr{
    \left( \frac{\pi_\rmp - r}{\pi_\rmn} \right) \prob{X} \prob{X'}
    =  
    \left( \pi_\rmp^2\prob{X|\rmn} - \pi_\rmp^2\prob{X|\rmp} \right) \prob{X'|\rmp} 
    + 
    \left( \pi_\rmp^2\prob{X|\rmp} - \pi_\rmn^2\prob{X|\rmn} \right) \prob{X'|\rmn}. \nonumber
}
Thus, when $\pi_\rmp \neq 1/2$, we can repeat the proof steps in Lemma~\ref{lma:MSconf} to rephrase (\ref{eq:Sconf_corruption_mechanism}) as 
\eqarr{
    \mmatrix{
        \prob{X}\prob{X'} \\
        \prob{X}\prob{X'}
    }
    =
    \mmatrix{
        \frac{\pi_\rmp \left( \pi_\rmp^2\prob{X|\rmp} - \pi_\rmn^2\prob{X|\rmn} \right)}{r-\pi_\rmn} & \frac{\pi_\rmp \left( \pi_\rmn^2\prob{X|\rmn} - \pi_\rmn^2\prob{X|\rmp} \right)}{r-\pi_\rmn} \\
        \frac{\pi_\rmn \left( \pi_\rmp^2\prob{X|\rmn} - \pi_\rmp^2\prob{X|\rmp} \right)}{\pi_\rmp - r} & \frac{\pi_\rmn \left( \pi_\rmp^2\prob{X|\rmp} - \pi_\rmn^2\prob{X|\rmn} \right)}{\pi_\rmp - r}
    }
    \mmatrix{
        \prob{X'|\rmp} \\
        \prob{X'|\rmn}
    }. \nonumber
}
Comparing the equation above with $\corr{P} = \Mcorr{Sconf} B$, we see that it is still feasible to formulate $\corr{P}$ with $X$ and $X'$ in $\Mcorr{Sconf}$ and $B$ of (\ref{eq:Sconf_corruption_mechanism}) swapped.
Then, repeating the same argument in \textbf{Step 1} with $x$ and $x'$ swapped, we obtain 
\eqarr{
    R(g)
    =
    \expt{X',X}{\frac{r- \pi_\rmn}{\pi_\rmp - \pi_\rmn} \ell_{\rmp}(X') + \frac{\pi_\rmp - r}{\pi_\rmp - \pi_\rmn} \ell_{\rmn}(X')}. \label{eq:rewrite_Sconf_2}
}
Therefore, the following combines (\ref{eq:rewrite_Sconf_1}) and (\ref{eq:rewrite_Sconf_2}) to obtain
\eqarr{
    R(g) 
    &=& 
    \frac{1}{2} (R(g) + R(g)) \nonumber \\
    &=&
    \frac{1}{2} \expt{X,X'}{\frac{r- \pi_\rmn}{\pi_\rmp - \pi_\rmn} \ell_{\rmp}(X) + \frac{\pi_\rmp - r}{\pi_\rmp - \pi_\rmn} \ell_{\rmn}(X)} + 
    \frac{1}{2} \expt{X,X'}{\frac{r- \pi_\rmn}{\pi_\rmp - \pi_\rmn} \ell_{\rmp}(X') + \frac{\pi_\rmp - r}{\pi_\rmp - \pi_\rmn} \ell_{\rmn}(X')} \nonumber \\
    &=&
    \expt{X,X'}{\frac{r- \pi_\rmn}{\pi_\rmp - \pi_\rmn} \frac{\ell_{\rmp}(X) + \ell_{\rmp}(X')}{2} + \frac{\pi_\rmp - r}{\pi_\rmp - \pi_\rmn} \frac{\ell_{\rmn}(X) + \ell_{\rmn}(X')}{2}} \nonumber
}
that recovers rewrite (\ref{eq:review_rewrite_Sconf}) in Section~\ref{sec:review_Sconf}.
By matching notations, we recover Theorem 3 of \cite{sconf_21/Cao/FXANS/21}
\footnote{
The matching to the notations of \cite{sconf_21/Cao/FXANS/21} is as follows: 
$\pi_\rmp$ is $\pi_+$, 
$\pi_\rmn$ is $\pi_-$, 
$r$ is $s$, 
$\ell_\rmp(X)$ is $\ell(g(X),+1)$, and 
$\ell_\rmn(X)$ is $\ell(g(X),-1)$.
}.

Now we switch to the deferred proofs.
\begin{proof}\textbf{of Lemma~\ref{lma:key_Sconf}.}
    Applying the definitions of $\mathcal{L}_{x}$ and (\ref{eq:Sconf_corruption_mechanism}), we derive
    \eqarr{
        \int_{x'\in\mathcal{X}} \mathcal{L}_{x}^{\top} \corr{P} \; \dx'
        &=&
        \int_{x'} 
        \mathcal{L}_{x}^{\top}
        \mmatrix{
            \frac{\pi_\rmp \left( \pi_\rmp^2\prob{x'|\rmp} - \pi_\rmn^2\prob{x'|\rmn} \right)}{r-\pi_\rmn} & \frac{\pi_\rmp \left( \pi_\rmn^2\prob{x'|\rmn} - \pi_\rmn^2\prob{x'|\rmp} \right)}{r-\pi_\rmn} \\
            \frac{\pi_\rmn \left( \pi_\rmp^2\prob{x'|\rmn} - \pi_\rmp^2\prob{x'|\rmp} \right)}{\pi_\rmp - r} & \frac{\pi_\rmn \left( \pi_\rmp^2\prob{x'|\rmp} - \pi_\rmn^2\prob{x'|\rmn} \right)}{\pi_\rmp - r}
        }
        \mmatrix{
            \prob{X|\rmp} \\
            \prob{X|\rmn}
        }
        \; \dx' \nonumber \\
        &=&
        \mathcal{L}_{x}^{\top}
        \mmatrix{
            \int_{x'} \frac{ \pi_\rmp^2\prob{x'|\rmp} - \pi_\rmn^2\prob{x'|\rmn} }{r-\pi_\rmn} \dx' & \int_{x'} \frac{ \pi_\rmn^2\prob{x'|\rmn} - \pi_\rmn^2\prob{x'|\rmp} }{r-\pi_\rmn} \dx' \\
            \int_{x'} \frac{ \pi_\rmp^2\prob{x'|\rmn} - \pi_\rmp^2\prob{x'|\rmp} }{\pi_\rmp - r} \dx' & \int_{x'} \frac{ \pi_\rmp^2\prob{x'|\rmp} - \pi_\rmn^2\prob{x'|\rmn} }{\pi_\rmp - r} \dx'
        }
        \mmatrix{
            \pi_\rmp\prob{X|\rmp} \\
            \pi_\rmn\prob{X|\rmn}
        }. \nonumber
    }
    Comparing the equality above with (\ref{eq:key_Sconf}), we have
    \eqarr{
        \tilde{M}_{\mrm{Sconf}}
        =
        \mmatrix{
            \int_{x'} \frac{ \pi_\rmp^2\prob{x'|\rmp} - \pi_\rmn^2\prob{x'|\rmn} }{r-\pi_\rmn} \dx' & \int_{x'} \frac{ \pi_\rmn^2\prob{x'|\rmn} - \pi_\rmn^2\prob{x'|\rmp} }{r-\pi_\rmn} \dx' \\
            \int_{x'} \frac{ \pi_\rmp^2\prob{x'|\rmn} - \pi_\rmp^2\prob{x'|\rmp} }{\pi_\rmp - r} \dx' & \int_{x'} \frac{ \pi_\rmp^2\prob{x'|\rmp} - \pi_\rmn^2\prob{x'|\rmn} }{\pi_\rmp - r} \dx'
        } \nonumber
    }
    that completes the proof.
\end{proof}
\begin{proof}\textbf{of Lemma~\ref{lma:inv_MSconf}.}
    We prove the lemma by examining each entry of $\tilde{M}_{\mrm{Sconf}}^{\dagger} \tilde{M}_{\mrm{Sconf}}$.
    The value of $(1,1)$ entry is 
    \eqarr{
        \frac{r- \pi_\rmn}{\pi_\rmp - \pi_\rmn} \int_{x'} \frac{ \pi_\rmp^2\prob{x'|\rmp} - \pi_\rmn^2\prob{x'|\rmn} }{r-\pi_\rmn} \dx'
        &=&
        \frac{1}{\pi_\rmp - \pi_\rmn} 
        \left( \pi_\rmp^2 \int_{x'} \prob{x'|\rmp} \dx' - \pi_\rmn^2 \int_{x'} \prob{x'|\rmn} \dx' \right) \nonumber \\
        &=&
        \frac{\pi_\rmp^2 - \pi_\rmn^2}{\pi_\rmp - \pi_\rmn}
        = 1. \nonumber
    }
    The $(2,2)$ entry has value
    \eqarr{
        \frac{\pi_\rmp - r}{\pi_\rmp - \pi_\rmn} \int_{x'} \frac{ \pi_\rmp^2\prob{x'|\rmp} - \pi_\rmn^2\prob{x'|\rmn} }{\pi_\rmp - r} \dx'
        &=&
        \frac{1}{\pi_\rmp - \pi_\rmn} 
        \left( \pi_\rmp^2 \int_{x'} \prob{x'|\rmp} \dx' - \pi_\rmn^2 \int_{x'} \prob{x'|\rmn} \dx' \right) \nonumber \\
        &=&
        \frac{\pi_\rmp^2 - \pi_\rmn^2}{\pi_\rmp - \pi_\rmn} 
        = 1. \nonumber
    }
    The $(1,2)$ entry and the $(2,1)$ entry are zeros since $\int_{x'} \left( \pi_\rmn^2\prob{x'|\rmn} - \pi_\rmn^2\prob{x'|\rmp} \right) \dx' = 0$ and $\int_{x'} \left( \pi_\rmp^2\prob{x'|\rmn} - \pi_\rmp^2\prob{x'|\rmp} \right) \dx' = 0$
\end{proof}


\subsection{CCN Scenarios}
\label{sec:risk_rewrite_CCNs}
The proposed framework is now applied to conduct the risk rewrites for WSLs discussed in Section~\ref{sec:formulations_ccns} and summarized in Table~\ref{tab:CCN_matrices_summary}. 
Counterintuitively, we demonstrate that finding an inverse matrix (e.g., Theorem~\ref{thm:inv_method}) is not the only way to solve the risk rewrite problem. 
Introduced in Theorem~\ref{thm:marginal_chain}, the new technique exploited in this subsection, {marginal chain}, calculates the decontamination matrix for (\ref{eq:recipe6}) via applying the conditional probability formula twice during a chain of matrix multiplications.

\subsubsection{Generalized CCN}
\label{sec:GE2_GCCN}
Same as what we have illustrated in the MCD scenarios, having a properly designed $\Mcorr{corr}^{\dagger}$ satisfying (\ref{eq:recipe6}) is crucial for constructing the corrected losses $\corr{L}$ for a CCN instance.
We next discuss how to achieve $\corr{L}^{\top}\corr{P} = L^{\top}P$ (\ref{eq:recipe2a}) for the generalized CCN setting given the contamination matrix $\Mcorr{gCCN}$ (\ref{eq:MGeneral_2}).
Derived equations will be applied to solve the risk rewrite problem for WSLs discussed in Section~\ref{sec:formulations_ccns}.

\paragraph{Step 1: Corrected Loss Design.}~\\
\indent
Let us follow the notations in Theorem~\ref{thm:marginal_chain} and Section~\ref{sec:GE1_gCCN}.
Note that for generalized CCN, $\corr{P} = \Mcorr{gCCN} B$ and $B = P$ from Lemma~\ref{lma:formulate_gCCN}.
Thus, we have $\corr{P} = \Mcorr{gCCN} P$ for free (i.e., do not need to handle $\Mcorr{trsf}$ discussed in Section~\ref{sec:GE2_UU}).
Noticing $\Mcorr{gCCN}$ (\ref{eq:MGeneral_2}) equals $M$ (\ref{eq:M_MargianlChain}), a direct application of Theorem~\ref{thm:marginal_chain} gives the decontamination matrix 
\eqarr{
    \Mcorr{gCCN}^{\dagger}
    =
    \mmatrix{
        \prob{Y=1|S=s_{1},X} & \prob{Y=1|S=s_{2},X} & \cdots & \prob{Y=1|S=s_{|\mathcal{S}|},X} \\
        \prob{Y=2|S=s_{1},X} & \prob{Y=2|S=s_{2},X} & \cdots & \prob{Y=2|S=s_{|\mathcal{S}|},X} \\
        \vdots & \vdots & \ddots & \vdots \\
        \prob{Y=K|S=s_{1},X} & \prob{Y=K|S=s_{2},X} & \cdots & \prob{Y=K|S=s_{|\mathcal{S}|},X}
    } \label{eq:MGeneral_inv_b}
}
for the generalized CCN setting satisfying $\Mcorr{gCCN}^{\dagger} \corr{P} = P$.
Then, instantiating (\ref{eq:recipe5}), we obtain the corrected losses
$\corr{L}^{\top} := L^{\top} \Mcorr{gCCN}^{\dagger}$, where
the $k$-th entry of $L$ is $\ell_{Y=k}$ with $k \in [K]$ and the $j$-th entry of $\corr{L}$ is $\corr{\ell}_{S=s_{j}}$ with $j \in [|\mathcal{S}|]$.

Despite Theorem~\ref{thm:marginal_chain}'s simplicity, the construction of $\Mcorr{gCCN}^{\dagger}$ is somewhat surprising.
$\Mcorr{gCCN}^{\dagger}$, to our best knowledge, contributes to a first loss correction result relaxing the invertibility constraint.
Unlike $\Mcorr{UU}^{\dagger}$ (Corollary~\ref{thm:UU_M_inv}), which needs to compute an inverse matrix, one can construct $\Mcorr{gCCN}^{\dagger}$ by calculating each entry $\prob{Y|S,X}$ in (\ref{eq:MGeneral_inv_b}), to which, we point out an efficient way in Section~\ref{sec:GE2_PPL}.

\paragraph{Step 2: Classification Risk Rewrite.}~\\
\indent
The following theorem applies the proposed framework to obtain an intermediate form of risk rewrite.
\begin{theorem}
\label{thm:rewrite_gCCN}\label{thm:rewrite_CCN_gCCN}
    Denote $\corr{L}^{\top} := L \Mcorr{gCCN}^{\dagger}$.
    Then, $\corr{L}^{\top} \corr{P} = L^{\top} P$ and 
    \eqarr{
        R(g) 
        = 
        \int_{\mathcal{X}} L^{\top} P \dx 
        =
        \int_{\mathcal{X}} \corr{L}^{\top} \corr{P} \dx. \label{eq:risk_rewrite_GCCN_intermediate}
    }
\end{theorem}
\begin{proof}
    Since $\Mcorr{gCCN}^{\dagger}$ is given by Theorem~\ref{thm:marginal_chain}, $\Mcorr{gCCN}^{\dagger} \corr{P} = P$.
    Thus, following the framework (\ref{eq:recipe2a}), we have $\corr{L}^{\top} \corr{P} = L^{\top} \Mcorr{gCCN}^{\dagger} \corr{P} = L^{\top} P$ implying (\ref{eq:risk_rewrite_GCCN_intermediate}).
\end{proof}
\connect
Theorem~\ref{thm:rewrite_gCCN} will be applied to derive the respective rewrites for WSLs discussed in Section~\ref{sec:formulations_ccns} in the rest of this subsection.
In particular, we explain how to realize $\Mcorr{gCCN}^{\dagger}$ (\ref{eq:MGeneral_inv_b}) for a given CCN scenario. 
Then, the risk rewrite (\ref{eq:risk_rewrite_GCCN_intermediate}) automatically carries over for the scenario considered, and the respective $\corr{L}$ specifies the corrected losses in the rewrite.

\subsubsection{Proper Partial-Label (PPL) Learning}
\label{sec:GE2_PPL}
$\Mcorr{gCCN}^{\dagger}$ (\ref{eq:MGeneral_inv_b}) provides an abstraction for us to construct the corrected losses $\corr{L}$.
Next, we focus on deriving the actual form of $\prob{Y|S,X}$ in $\Mcorr{gCCN}^{\dagger}$ to explicitly express $\corr{\ell}_{S}$ for PPL.
\paragraph{Step 1: Corrected Loss Design and Risk Rewrite.}~\\
\indent
Let us follow the notations in Theorem~\ref{thm:marginal_chain} and Section~\ref{sec:GE1_PPL}.
The following lemma specifies the form of $\prob{Y|S,X}$ to instantiate $\Mcorr{gCCN}^{\dagger}$.
\begin{lemma}
\label{lma:inv_MPPL}
    $\Mcorr{PPL}^{\dagger}$ corresponds to realizing $\Mcorr{gCCN}^{\dagger}$ (\ref{eq:MGeneral_inv_b}) with
    \eqarr{
        \prob{Y=i|S=s_j,X} 
        := \frac{\prob{Y=i|X} \ind{Y=i \in s_j}}{\sum_{a \in s_j} \prob{Y=a|X}}. \label{eq:PPL_type1}
    }
\end{lemma}
\begin{proof}
    Recall that the decontamination matrix of $\Mcorr{gCCN}$ (\ref{eq:MGeneral_2}) is $\Mcorr{gCCN}^{\dagger}$ (\ref{eq:MGeneral_inv_b}) and $\Mcorr{PPL}$ is a reduction of $\Mcorr{gCCN}$ via $\prob{S|Y,X} = C(S,X)\ind{Y \in S}$ (\ref{eq:P_S|YX_PPL}).
    Thus, to find out the $(i,j)$ entry of $\Mcorr{PPL}^{\dagger}$, we need to find out the form of $\prob{Y=i|S=s_j,X}$ subject to (\ref{eq:P_S|YX_PPL}).

    Applying Theorem 1 of \cite{partial_21_PPL/Wu/LS/23} directly gives
    \eqarr{
        \prob{Y=i|S=s_j,X} 
        = \frac{\prob{Y=i|X} \ind{Y=i \in s_j}}{\sum_{a \in s_j} \prob{Y=a|X}}, \nonumber
    }
    which completes the proof. 
    For completeness, we provide a derivation as follows.
    Since $\prob{S|X} = \sum_{a\in S} \prob{S,Y=a|X}$ (recall the assumption $\prob{Y \in S | S, X} = 1$ in Section~\ref{sec:GE1_PPL}) and $\prob{S,Y|X} = \prob{S|Y,X}\prob{Y|X}$,
    \eqarr{
        \prob{Y|S,X}
        =
        \frac{\prob{S,Y|X}}{\prob{S|X}}
        &=&
        \frac{\prob{S|Y,X}\prob{Y|X}}{\sum_{a\in S} \prob{S|Y=a,X}\prob{Y=a|X}} \nonumber \\
        &=&
        \frac{C(S,X)\ind{Y \in S} \prob{Y|X}}{\sum_{a\in S} C(S,X)\ind{Y=a \in S} \prob{Y=a|X}} \nonumber \\
        &=&
        \frac{\prob{Y|X}\ind{Y \in S}}{\sum_{a\in S} \prob{Y=a|X}} \nonumber
    }
    achieves (\ref{eq:PPL_type1}).
\end{proof}
Then, we construct the corrected losses according to (\ref{eq:PPL_type1}) and continue (\ref{eq:risk_rewrite_GCCN_intermediate}) to obtain the risk rewrite (\ref{eq:review_rewrite_PPL}) in Section~\ref{sec:review_PPL} for PPL.
\begin{corollary}
\label{thm:rewrite_PPL}
    Define the corrected losses 
    $\corr{L}^{\top} := L^{\top} \Mcorr{PPL}^{\dagger}$.
    Then, for PPL learning, the classification risk can be rewritten as 
    $$
        R(g) = \expt{S,X}{\corr{\ell}_S}, 
    $$
    where
    \eqarr{
        \corr{\ell}_{S} = \sum_{i \in S} \frac{\prob{Y=i|X}}{\sum_{a \in S} \prob{Y=a|X}} \ell_{Y=i}. \label{eq:calibrated_loss_PPL}
    }
\end{corollary}
\begin{proof}
    Given (\ref{eq:PPL_type1}), the $j$-th entry of $\corr{L}^{\top}$ is of the form
    \eqarr{
        \corr{\ell}_{S=s_j} 
        = \left( L^{\top} \Mcorr{PPL}^{\dagger} \right)_j 
        &=& \sum_{i=1}^{K} \frac{\prob{Y=i|X} \ind{Y=i \in s_j} }{\sum_{a \in s_j} \prob{Y=a|X}} \ell_{Y=i} \nonumber \\
        &=& \sum_{i \in s_j} \frac{\prob{Y=i|X}}{\sum_{a \in s_j} \prob{Y=a|X}} \ell_{Y=i}. \nonumber
    }
    Then, since $\Mcorr{PPL}^{\dagger}$ is a realization of $\Mcorr{gCCN}^{\dagger}$ according to Lemma~\ref{lma:inv_MPPL}, we continue (\ref{eq:risk_rewrite_GCCN_intermediate}) to express the risk as
    \eqarr{
        R(g) 
        = 
        \int_{x \in \mathcal{X}} \corr{L}^{\top} \corr{P} \dx 
        =
        \int_{x \in \mathcal{X}} \sum_{j=1}^{|\mathcal{S}|} \prob{S=s_j, x} \corr{\ell}_{S=s_j} \dx
        = 
        \expt{S,X}{\corr{\ell}_S}. \nonumber 
    }
\end{proof}

\paragraph{Step 2: Recovering the previous result(s).}~\\
\indent
We finish this part by pointing out Corollary~\ref{thm:rewrite_PPL} recovers Theorem 3 of \cite{partial_21_PPL/Wu/LS/23}.

\subsubsection{Provably Consistent Partial-Label (PCPL) Learning}
\label{sec:GE2_PCPL}
It is fairly straightforward to apply the proposed framework to rewrite the classification risk.
But it is more involved in recovering the existing result.

\paragraph{Step 1: Corrected Loss Design and Risk Rewrite.}~\\
\indent
From Section~\ref{sec:GE1_PCPL} we know that PCPL is a special case of PPL that only differs in the choice of $C(S, X)$.
Note that $\Mcorr{PCPL}^{\dagger} = \Mcorr{PPL}^{\dagger}$ since the proof of Theorem 1 in \cite{partial_21_PPL/Wu/LS/23} cancels $C(S, X)$ in the derivation (refer the proof of Lemma~\ref{lma:inv_MPPL} for detail).
Hence, following the notations in Section~\ref{sec:GE1_PCPL} and applying Corollary~\ref{thm:rewrite_PPL} directly, we obtain the risk rewrite for PCPL:
\begin{corollary}
\label{thm:rewrite_PCPL}
    Let $\Mcorr{PCPL}^{\dagger} = \Mcorr{PPL}^{\dagger}$.
    Define the corrected losses 
    $\corr{L}^{\top} := L^{\top} \Mcorr{PCPL}^{\dagger}$.
    Then, for PCPL learning, the classification risk can be rewritten as 
    $$
        R(g) = \expt{S,X}{\corr{\ell}_S}, 
    $$
    where
    \eqarr{
        \corr{\ell}_{S} = \sum_{i \in S} \frac{\prob{Y=i|X}}{\sum_{a \in S} \prob{Y=a|X}} \ell_{Y=i}. \label{eq:rewrite_PCPL}
    }
\end{corollary}

\paragraph{Step 2: Recovering the previous result(s).}~\\
\indent
In order to recover (8) of \cite{partial_20_PCPL/Feng/LHXNGAS/20}, we need to reorganize the sum in (\ref{eq:rewrite_PCPL}) by leveraging a unique property of a pair of partial-labels $(s,s')$ that complement each other.
The following technical lemma states the required property, whose proof is deferred to the end of this sub-subsection.
\begin{lemma}
\label{lma:recover_PCPL}
    Let $(s,s')$ be a pair of partial-labels satisfying $s = \mathcal{Y}\backslash s'$.
    Then, 
    \eqarr{
        \prob{S=s, X} \corr{\ell}_{S=s} + \prob{S=s', X} \corr{\ell}_{S=s'}
        =
        \prob{S=s, X} \sum_{i=1}^{K} \frac{\prob{Y=i|X} \ell_{Y=i}}{\sum_{a \in s} \prob{Y=a|X}}. \nonumber 
    }
\end{lemma}
Denote $s'_j := \mathcal{Y}\backslash s_j$ for every $s_j \in \mathcal{S}$.
Then, Lemma~\ref{lma:recover_PCPL} implies
\eqarr{
    \sum_{j=1}^{|\mathcal{S}|} 2 \prob{S=s_j, X} \corr{\ell}_{S=s_j}
    &=&
    \sum_{j=1}^{|\mathcal{S}|} \left( \prob{S=s_j, X} \corr{\ell}_{S=s_j} + \prob{S=s'_j, X} \corr{\ell}_{S=s'_j} \right) \nonumber \\
    &=&
    \sum_{j=1}^{|\mathcal{S}|} 
    \prob{S=s_j, X} 
    \sum_{i=1}^{K} \frac{\prob{Y=i|X} \ell_{Y=i}}{\sum_{a \in s_j} \prob{Y=a|X}}. \nonumber
}
Hence, continuing from Corollary~\ref{thm:rewrite_PCPL},
\eqarr{
    \expt{S,X}{\corr{\ell}_{S}} 
    &=&
    \int_{x\in\mathcal{X}} \sum_{j=1}^{|\mathcal{S}|} \prob{S=s_{j}, x} \corr{\ell}_{S=s_j} \dx \nonumber \\
    &=&
    \frac{1}{2} \int_{x\in\mathcal{X}} \sum_{j=1}^{|\mathcal{S}|} 
    \prob{S=s_j, x} 
    \sum_{i=1}^{K} \frac{\prob{Y=i|x} \ell_{Y=i}}{\sum_{a \in s_j} \prob{Y=a|x}} \dx \nonumber \\
    &=&
    \frac{1}{2} \expt{S,X}{\sum_{i=1}^{K} \frac{\prob{Y=i|X} }{\sum_{a \in S} \prob{Y=a|X}} \ell_{Y=i}} \nonumber
}
shows that the rewrite from the framework recovers (\ref{eq:review_rewrite_PCPL}) in Section~\ref{sec:review_PCPL}. 
By matching notations, we also recover (8) of \cite{partial_20_PCPL/Feng/LHXNGAS/20}\footnote{
The matching to the notations of \cite{partial_20_PCPL/Feng/LHXNGAS/20} is as follows:
$\prob{S,X}$ is $\tilde{p}(x,Y)$,
$\prob{Y=i|X}$ is $p(y=i|x)$, and
$\ell_{Y=i}$ is $\mathcal{L}(f(x),i)$.
}.

Now we return to the postponed proof.
\begin{proof}\textbf{of Lemma~\ref{lma:recover_PCPL}.}
    Given $\Mcorr{PCPL}$ (\ref{eq:MPCPL}), we apply $\corr{P} = \Mcorr{PCPL} B$ to obtain 
    \eqarr{
        \prob{S=s', X}
        =
        \frac{\sum_{k=1}^{K} \ind{Y=k\in s'} \prob{Y=k, X}}{2^{K-1}-1}. \nonumber
    }
    We also have
    \eqarr{
        \corr{\ell}_{S=s'}
        =
        \frac{\sum_{i\in s'} \prob{Y=i|X} \ell_{Y=i} }{\sum_{a \in s'} \prob{Y=a|X}} \nonumber
    }
    according to (\ref{eq:calibrated_loss_PPL}).
    Since
    \eqarr{
        \frac{\sum_{k=1}^{K} \ind{Y=k\in s'} \prob{Y=k, X}}{\sum_{a \in s'} \prob{Y=a|X}} 
        = 
        \prob{X} 
        =
        \frac{\sum_{k=1}^{K} \ind{Y=k\in s} \prob{Y=k, X}}{\sum_{a \in s} \prob{Y=a|X}}, \nonumber
    }
    \eqarr{
        \prob{S=s', X} \corr{\ell}_{S=s'}
        &=&
        \frac{\sum_{k=1}^{K} \ind{Y=k\in s'} \prob{Y=k, X}}{2^{K-1}-1} \frac{\sum_{i\in s'} \prob{Y=i|X} \ell_{Y=i} }{\sum_{a \in s'} \prob{Y=a|X}} \nonumber \\
        &=&
        \frac{\sum_{k=1}^{K} \ind{Y=k\in s} \prob{Y=k, X}}{2^{K-1}-1} \frac{\sum_{i\in s'} \prob{Y=i|X} \ell_{Y=i} }{\sum_{a \in s} \prob{Y=a|X}}. \nonumber
    }
    Thus,
    \eqarr{
        \prob{S=s, X} \corr{\ell}_{S=s} + \prob{S=s', X} \corr{\ell}_{S=s'} 
        &=& 
        \frac{\sum_{k=1}^{K} \ind{Y=k\in s} \prob{Y=k, X}}{2^{K-1}-1} \frac{\sum_{i\in s} \prob{Y=i|X} \ell_{Y=i} }{\sum_{a \in s} \prob{Y=a|X}} \nonumber \\
        &&+
        \frac{\sum_{k=1}^{K} \ind{Y=k\in s} \prob{Y=k, X}}{2^{K-1}-1} \frac{\sum_{i\in s'} \prob{Y=i|X} \ell_{Y=i} }{\sum_{a \in s} \prob{Y=a|X}} \nonumber \\
        &=&
        \prob{S=s, X} \frac{\sum_{i\in s} \prob{Y=i|X} \ell_{Y=i} + \sum_{i\in s'} \prob{Y=i|X} \ell_{Y=i}}{\sum_{a \in s} \prob{Y=a|X}} \nonumber \\
        &=&
        \prob{S=s, X} \sum_{i = 1}^{K} \frac{\prob{Y=i|X} \ell_{Y=i} }{\sum_{a \in s} \prob{Y=a|X}} \nonumber
    }
    proves the lemma.
\end{proof}


\subsubsection{Multi-Complementary-Label (MCL) Learning}
\label{sec:GE2_MCL}
\paragraph{Step 1: Corrected Loss Design and Risk Rewrite.}~\\
Let us follow the notations in Theorem~\ref{thm:marginal_chain} and Section~\ref{sec:GE1_MCL}.
As discussed in Section~\ref{sec:GE1_MCL}, MCL is a special case of PPL.
Thus, we substitute $C(S,X)\ind{Y\in S}$ in PPL with the values specified in (\ref{eq:PPL_to_MCL2}) and repeat the proof of Theorem 1 in \cite{partial_21_PPL/Wu/LS/23} to obtain a complementary version of (11) of \cite{partial_21_PPL/Wu/LS/23}:
\eqarr{
    \prob{Y=i|\corr{S}=\corr{s}_j,X} 
    = \frac{\prob{Y=i|X} \ind{Y=i \notin \corr{s}_j}}{\sum_{a \notin \corr{s}_j} \prob{Y=a|X}}. \label{eq:PPL_type1_to_MCL_margin}
} 
The construction of $\Mcorr{MCL}^{\dagger}$ is to replace each entry in $\Mcorr{gCCN}^{\dagger}$ (\ref{eq:MGeneral_inv_b}) via (\ref{eq:PPL_type1_to_MCL_margin}).
Note that $\prob{Y=i|S=s_{j}, X}$ is assigned as $\prob{Y=i|\corr{S}=\corr{s}_{j}, X}$ since the observed information is changed from a partial sense to a complementary sense, thus inducing the difference in notation.
Repeating the same steps for proving Corollary~\ref{thm:rewrite_PPL}, we have
\eqarr{
    \corr{\ell}_{\corr{S}=\corr{s}_j} 
    = \left( L^{\top} \Mcorr{MCL}^{\dagger} \right)_j 
    &=& \sum_{i=1}^{K} \frac{\prob{Y=i|X} \ind{Y=i \notin \corr{s}_j} }{\sum_{a \notin \corr{s}_j} \prob{Y=a|X}} \ell_{Y=i} \nonumber \\
    &=& \sum_{i \notin \corr{s}_j} \frac{\prob{Y=i|X}}{\sum_{a \notin \corr{s}_j} \prob{Y=a|X}} \ell_{Y=i} \nonumber
}
that leads to a counterpart of Corollary~\ref{thm:rewrite_PPL} for MCL:
\begin{corollary}
\label{thm:rewrite_MCL}
    Define the corrected losses
    $\corr{L}^{\top} := L^{\top} \Mcorr{MCL}^{\dagger}$.
    Then, for MCL learning, the classification risk can be rewritten as 
    $$
        R(g) = \expt{\corr{S},X}{\corr{\ell}_{\corr{S}}}, 
    $$
    where
    \eqarr{
        \corr{\ell}_{\corr{S}} 
        = 
        \sum_{i \notin \corr{S}} \frac{\prob{Y=i|X}}{\sum_{a \notin \corr{S}} \prob{Y=a|X}} \ell_{Y=i}. \label{eq:rewrite_MCL_marginal} 
    }
\end{corollary}

\paragraph{Step 2: Recovering the previous result(s).}~\\
\indent
Although legitimate, the risk rewrite (\ref{eq:rewrite_MCL_marginal}) following the marginal chain approach appears different from Theorem 3 of \cite{comp_20_MCL/Feng/KHNAS/20}, to which we resort to the inversion approach (Theorem~\ref{thm:inv_method}) that finds another decontamination matrix, termed $\Mcorr{MCL}^{-1}$, to recover.
As a preparation step, we denote $N_d$ as the number of multi-complementary-labels with size $d$ and group rows of $\Mcorr{MCL}$ (\ref{eq:MMCL_detail}) by the size of labels as follows.
\eqarr{
    \Mcorr{MCL} 
    =
    \mmatrix{
        \prob{|\corr{S}|=1} M_1 \\
        \prob{|\corr{S}|=2} M_2 \\
        \vdots \\
        \prob{|\corr{S}|=K-1} M_{K-1}
    }, \label{eq:Mcorr_MCL_regroup}
}
where for $d\in [K-1]$, each block is of the form\footnote{Comparing to (\ref{eq:MMCL_detail}) where we use one index to denote a total of $|\mathcal{S}|$ partial-labels, $M_d$ uses a pair of indices $d$ and $j$ to denote the $j$-th partial-label with size $d$. It is easy to verify that $\sum_{d=1}^{K-1} N_d = \sum_{d=1}^{K-1} {K-1 \choose d} = 2^K - 2 = |\mathcal{S}|$.}
\eqarr{
    M_d
    =
    \frac{1}{{K-1 \choose d}}
    \mmatrix{
        \ind{Y=1 \notin \corr{s}_{d,1}} & \ind{Y=2 \notin \corr{s}_{d,1}} & \cdots & \ind{Y=K \notin \corr{s}_{d,1}} \\
        \ind{Y=1 \notin \corr{s}_{d,2}} & \ind{Y=2 \notin \corr{s}_{d,2}} & \cdots & \ind{Y=K \notin \corr{s}_{d,2}} \\
        \vdots & \vdots & \ddots & \vdots \\
        \ind{Y=1 \notin \corr{s}_{d,N_d}} & \ind{Y=2 \notin \corr{s}_{d,N_d}} & \cdots & \ind{Y=K \notin \corr{s}_{d,N_d}}
    }. \label{eq:Md_MCL}
}
To maintain the equality $\corr{P} = \Mcorr{MCL} P$ established in Lemma~\ref{lma:matrix_formulation_MCL}, we also rearrange $\corr{P}$ (\ref{eq:label_dist_MCL}) as
\eqarr{
    \mmatrix{
        \prob{\corr{S} = \corr{s}_{1,1}, X} & \cdots
        \prob{\corr{S} = \corr{s}_{1,N_1}, X} & \cdots
        \prob{\corr{S} = \corr{s}_{K-1,1}, X} & \cdots
        \prob{\corr{S} = \corr{s}_{K-1,N_{K-1}}, X}
    }^{\top}. \label{eq:label_dist_MCL_2}
}
As a sanity check, we see that for any $d' \in [K-1]$ and $j' \in [N_d]$,
\eqarr{
    \mmatrix{
        \prob{|\corr{S}|=d'} M_{d'} P 
    }_{j'}
    &=&
    \prob{|\corr{S}|=d'} \cdot \frac{1}{{K-1 \choose d'}} \sum_{Y} \ind{Y \notin \corr{s}_{d',j'}} \prob{Y,X} \nonumber \\
    &=&
    \sum_{d=1}^{K-1} \prob{|\corr{s}_{d',j'}| = d} \cdot \frac{1}{{K-1 \choose d'}} \sum_{Y \notin \corr{s}_{d',j'}} \prob{Y,X} \ind{|\corr{s}_{d',j'}|=d} \nonumber \\
    &=&
    \prob{\corr{S} = \corr{s}_{d',j'}, X}. \label{eq:formulate_MCL_3}
}

The next lemma is crucial for us to devise the decontamination matrix $\Mcorr{MCL}^{-1}$ via the inversion approach.
We defer its proof to the later part of this sub-subsection.
\begin{lemma}
\label{lma:Md_inv}
    Let $i^{\star} \in \mathcal{Y}$ be fixed.
    Then, for every $d \in [K-1]$,
    \eqarr{
        \prob{Y=i^{\star}, X} 
        =
        \sum_{j=1}^{N_d} 
        \left( 
        1-\frac{K-1}{d} \ind{Y=i^{\star} \in \corr{S} = \corr{s}_{d,j}} 
        \right)  
        \prob{\corr{S}=\corr{s}_{d,j}, X| |\corr{S}|=d}. \nonumber 
    }
    Moreover, the inverse matrix $M_d^{-1}$ of $M_d$ (\ref{eq:Md_MCL}) is of the form
    \eqarr{
        \mmatrix{
            1-\frac{K-1}{d}\ind{Y=1\in \corr{s}_{d,1}} & 1-\frac{K-1}{d}\ind{Y=1\in \corr{s}_{d,2}} &
            \cdots & 1-\frac{K-1}{d}\ind{Y=1\in \corr{s}_{d,N_d}} \\
            1-\frac{K-1}{d}\ind{Y=2\in \corr{s}_{d,1}} & 1-\frac{K-1}{d}\ind{Y=2\in \corr{s}_{d,2}} &
            \cdots & 1-\frac{K-1}{d}\ind{Y=2\in \corr{s}_{d,N_d}} \\
            \vdots & \vdots & \ddots & \vdots \\
            1-\frac{K-1}{d}\ind{Y=K\in \corr{s}_{d,1}} & 1-\frac{K-1}{d}\ind{Y=K\in \corr{s}_{d,2}} &
            \cdots & 1-\frac{K-1}{d}\ind{Y=K\in \corr{s}_{d,N_d}} \\
        }. \nonumber \\ \label{eq:Md_inv_MCL}
    }
\end{lemma}
Applying the lemma, we construct 
\eqarr{
    \Mcorr{MCL}^{-1} 
    := 
    \mmatrix{
        M_1^{-1} & M_2^{-1} & \cdots & M_{K-1}^{-1} \nonumber
    }
}
and obtain $\Mcorr{MCL}^{-1} \corr{P} = P$ since $\corr{P} = \Mcorr{MCL} P$ (\ref{eq:formulate_MCL_3}) and 
\eqarr{
    \Mcorr{MCL}^{-1} \Mcorr{MCL}
    =
    \sum_{d=1}^{K-1} M_d^{-1} \prob{|\corr{S}|=d} M_d
    =
    \sum_{d=1}^{K-1} \prob{|\corr{S}|=d} M_d^{-1} M_d
    =
    \sum_{d=1}^{K-1} \prob{|\corr{S}|=d} I
    =
    I. \nonumber
}
\explain
We remark that $\Mcorr{MCL}^{-1}$ plays the same role as $\Mcorr{MCL}^{\dagger}$ realized by (\ref{eq:PPL_type1_to_MCL_margin}), as they both are decontamination matrices (designed to convert $\corr{P}$ back to $P$ and used to construct the corrected losses
$\corr{L}$).
Distinct symbols are merely used to reflect the difference that $\Mcorr{MCL}^{\dagger}$ results from the marginal chain method while $\Mcorr{MCL}^{-1}$ comes from the inversion approach.
Then, applying the framework (\ref{eq:recipe5}), $\corr{L}^{\top} := L^{\top} \Mcorr{MCL}^{-1}$ leads to 
$$
\corr{L}^{\top} \corr{P} 
= L^{\top} \Mcorr{MCL}^{-1} \corr{P}
= L^{\top} P.
$$ 

With the corrected losses $\corr{L}$ in hand, the following theorem provides the risk rewrite (\ref{eq:review_rewrite_MCL}) for MCL via the inversion approach and recovers Theorem 3 of \cite{comp_20_MCL/Feng/KHNAS/20}\footnote{
The matching to the notations of \cite{comp_20_MCL/Feng/KHNAS/20} is as follows:
$\prob{\corr{S},X| |\corr{S}|=d}$ is $\bar{p}(x,\bar{Y}|s=d)$,
$\prob{|\corr{S}|=d}$ is $p(s=d)$, and
$\corr{\ell}_{\corr{S}}$ is $\bar{\mathcal{L}}_{d}(f(x),\bar{Y})$.
}.
\begin{theorem}
\label{thm:rewrite_MCL_inv}
    For MCL learning, the classification risk can be expressed as follows.
    \eqarr{
        R(g) 
        = 
        \expt{\corr{S}, X}{\corr{\ell}_{\corr{S}}}
        =
        \sum_{d=1}^{K-1} \prob{|\corr{S}|=d} \expt{\corr{S},X||\corr{S}|=d}{\corr{\ell}_{\corr{S}}}, \nonumber
    }
    where
    \eqarr{
        \corr{\ell}_{\corr{S}}
        = 
        \sum_{i\notin\corr{S}} \ell_{Y=i}
        - \frac{K-1-|\corr{S}|}{|\corr{S}|} \sum_{\corr{s} \in \corr{S}} \ell_{Y=\corr{s}}. \nonumber
    }
\end{theorem}
\begin{proof}
    We first establish 
    \eqarr{
        R(g) 
        = 
        \int_{\mathcal{X}} \corr{L}^{\top} \corr{P} \dx 
        = 
        \expt{\corr{S},X}{\corr{\ell}_{\corr{S}}} \nonumber
    }
    since $\corr{L}^{\top} \corr{P} = L^{\top} P$, where $\corr{P}$ is specified in (\ref{eq:label_dist_MCL_2}) and $\corr{L}^{\top} = L^{\top} \Mcorr{MCL}^{-1}$ with $\mmatrix{\corr{L}^{\top}}_{\corr{S}} = \corr{\ell}_{\corr{S}}$.
    Also, recall that $\prob{\corr{S},X} = \sum_{d=1}^{K-1} \prob{|\corr{S}|=d} \prob{\corr{S},X||\corr{S}|=d}$ in Section~\ref{sec:GE1_MCL}.
    Thus, decomposing the probability by the size of $\corr{S}$, we have
    \eqarr{
        \expt{\corr{S},X}{\corr{\ell}_{\corr{S}}} = \sum_{d=1}^{K-1} \prob{|\corr{S}|=d} \expt{\corr{S},X||\corr{S}|=d}{\corr{\ell}_{\corr{S}}}. \nonumber
    }
    Lastly, the definition of $\Mcorr{MCL}^{-1}$ implies, when $\corr{S} = \corr{s}_{d,j}$, 
    \eqarr{
        \corr{\ell}_{\corr{S} = \corr{s}_{d,j}}
        =
        \mmatrix{L^{\top} M_d^{-1}}_{j}
        =
        \sum_{i=1}^K \ell_{Y=i} \left( 1 - \frac{K-1}{d} \ind{Y=i \in \corr{s}_{d,j}} \right)
        =
        \sum_{i=1}^K \ell_{Y=i} - \frac{K-1}{d} \sum_{i \in \corr{s}_{d,j}} \ell_{Y=i}. \nonumber
    }
    A simple reorganization and substituting $d$ with $|\corr{S}|$ shows 
    $$
        \corr{\ell}_{\corr{S}}
        =
        \sum_{i\notin\corr{S}} \ell_{Y=i} + \sum_{i\in\corr{S}} \ell_{Y=i} - \frac{K-1}{|\corr{S}|} \sum_{i \in \corr{S}} \ell_{Y=i}
        =
        \sum_{i\notin\corr{S}} \ell_{Y=i}
        - \frac{K-1-|\corr{S}|}{|\corr{S}|} \sum_{i \in \corr{S}} \ell_{Y=i}. \nonumber
    $$
\end{proof}

Now we return to the postponed proof.
\begin{proof}\textbf{of Lemma~\ref{lma:Md_inv}.}
We start with identifying $M_d^{-1}$.
Denote $\{ \corr{s}_{d,1}, \ldots, \corr{s}_{d,N_d} \}$, the set of multi-complementary-labels of size $d$, as $\corr{\mathcal{S}}_d$.
Let us focus on the sized-$d$ data-generating distribution
$$
\corr{P}_d = 
\mmatrix{
    \prob{\corr{S}=\corr{s}_{d,1}, X| |\corr{S}|=d} \\
    \vdots \\
    \prob{\corr{S}=\corr{s}_{d,N_d}, X| |\corr{S}|=d}
}.
$$
Note that $\corr{P}_d$ corresponds to extracting the entries from (\ref{eq:label_dist_MCL_2}) that generate sized-$d$ data and then dividing them by $\prob{|\corr{S}|=d}$. 
Thus, $\corr{P} = \Mcorr{MCL} P$ in Lemma~\ref{lma:matrix_formulation_MCL} implies $\corr{P}_d = M_d P$ and its $j$-th entry is expressed as
\eqarr{
    \prob{\corr{S}=\corr{s}_{d,j}, X| |\corr{S}|=d}
    =
    \frac{1}{{K-1 \choose d}} \sum_{i=1}^{K} \ind{Y=i \notin \corr{s}_{d,j}} \prob{Y=i, X}. \label{eq:Pd_entry}
}

The equality hints to us that if one manages to collect certain multi-complementary-labels $\corr{s}'$ to form an equation resembling $\sum_{\corr{s}'}\prob{\corr{s}', X| |\corr{s}'| = d} = c_3 \cdot \prob{Y=i, X}$ for some constant $c_3$, then a reciprocal operation $\frac{1}{c_3}$ recovers $\prob{Y=i, X}$ we need (recall we want to find $M_d^{-1}$ achieving $M_d^{-1} \corr{P}_d = P$).
To achieve such a goal, we fix on class $i^{\star}$ and collect elements in $\corr{\mathcal{S}}_d$ that do not contain $i^{\star}$ to form $\mathcal{E}_d^{i^{\star}} := \left\{ \corr{s}_{d,j} |\corr{s}_{d,j} \in \corr{\mathcal{S}}_d, i^{\star} \notin \corr{s}_{d,j} \right\}$ to connect $\prob{\corr{S}, X| |\corr{S}|=d}$ with $\prob{Y=i^{\star}, X}$ as follows.
Summing (\ref{eq:Pd_entry}) over all elements in $\mathcal{E}_d^{i^{\star}}$, we obtain
\eqarr{
    \sum_{\corr{s} \in \mathcal{E}_d^{i^{\star}}} \prob{\corr{S} = \corr{s}, X| |\corr{S}|=d}
    &=&
    \sum_{\corr{s} \in \mathcal{E}_d^{i^{\star}}} \frac{1}{{K-1 \choose d}} \sum_{i=1}^{K} \ind{Y=i \notin \corr{s}} \prob{Y=i, X} \nonumber \\
    &=&
    \frac{1}{{K-1 \choose d}} \left[ 
    {K-2 \choose d} \sum_{\substack{i=1 \\ i\neq i^{\star}}}^{K} \prob{Y=i, X} 
    +
    {K-1 \choose d} \prob{Y=i^{\star}, X} \right]. \nonumber
}
The last equality holds since there are ${K-2 \choose d}$ multi-complementary-labels $\corr{s} \in \corr{\mathcal{S}}_d$ such that $i \neq i^{\star}$ and neither of them is in $\corr{s}$, and there are ${K-1 \choose d}$ multi-complementary-labels $\corr{s} \in \corr{\mathcal{S}}_d$ such that $i = i^{\star}$ and $i$ is not in $\corr{s}$.
Then, we regroup the sums by pulling ${K-2 \choose d} \prob{Y=i^{\star}, X}$ out of ${K-1 \choose d} \prob{Y=i^{\star}, X}$ to combine with ${K-2 \choose d} \sum_{\substack{i=1 \\ i\neq i^{\star}}}^{K} \prob{Y=i, X}$. 
It leads to
\eqarr{
    \sum_{\corr{s} \in \mathcal{E}_d^{i^{\star}}} \prob{\corr{S} = \corr{s}, X| |\corr{S}|=d}
    &=&
    \frac{1}{{K-1 \choose d}} \left[ 
    {K-2 \choose d} \sum_{i=1}^{K} \prob{Y=i, X} 
    +
    {K-2 \choose d-1} \prob{Y=i^{\star}, X}\right] \nonumber \\
    &=&
    \frac{K-1-d}{K-1} \prob{X} + \frac{d}{K-1}\prob{Y=i^{\star}, X}. \label{eq:key_for_inverse_matrix}
}
Denoting $\corr{\mathcal{S}}_d \backslash \mathcal{E}_d^{i^{\star}} = \{ \corr{s}_{d,j} | \corr{s}_{d,j} \in \corr{\mathcal{S}}_d, i^{\star} \in \corr{s}_{d,j} \}$ as $\mathcal{I}_d^{i^{\star}}$ and rearranging terms in the above equation according to the reciprocal idea illustrated above, we have
\eqarr{
    \prob{Y=i^{\star}, X} 
    &=&
    \frac{K-1}{d}\left( \sum_{\corr{s} \in \mathcal{E}_d^{i^{\star}}} \prob{\corr{S} = \corr{s}, X| |\corr{S}|=d} - \frac{K-1-d}{K-1} \prob{X} \right) \label{eq:MCL_inverse_equation} \\
    &\stackrel{\text{(a)}}{=}&
    \frac{K-1}{d}\left( \prob{X} - \sum_{\corr{s} \in \mathcal{I}_d^{i^{\star}}} \prob{\corr{S} = \corr{s}, X| |\corr{S}|=d} - \frac{K-1-d}{K-1} \prob{X} \right) \nonumber \\
    &=&
    \prob{X} - \frac{K-1}{d} \sum_{\corr{s} \in \mathcal{I}_d^{i^{\star}}} \prob{\corr{S} = \corr{s}, X| |\corr{S}|=d}. \nonumber
}
Equality (a) holds since $|\corr{S}|$ and $X$ are independent \citep{comp_20_MCL/Feng/KHNAS/20}, which implies $$\prob{X} = \prob{X| |\corr{S}|=d} = \sum_{\corr{s}\in \corr{\mathcal{S}}_d} \prob{\corr{S} = \corr{s}, X| |\corr{S}| = d} = \sum_{\corr{s} \in \mathcal{E}_d^{i^{\star}}} \prob{\corr{S} = \corr{s}, X| |\corr{S}|=d} + \sum_{\corr{s} \in \mathcal{I}_d^{i^{\star}}} \prob{\corr{S} = \corr{s}, X| |\corr{S}|=d}.$$
Continuing the derivation, we have 
\eqarr{
    \prob{Y=i^{\star}, X}
    &=&
    \sum_{j=1}^{N_d} \prob{\corr{S}=\corr{s}_{d,j}, X| |\corr{S}|=d}
    - \sum_{j=1}^{N_d} \frac{K-1}{d} \ind{Y=i^{\star} \in \corr{s}_{d,j}} \prob{\corr{S}=\corr{s}_{d,j}, X| |\corr{S}|=d} \nonumber \\
    &=&
    \sum_{j=1}^{N_d} \left( 1-\frac{K-1}{d} \ind{Y=i^{\star} \in \corr{s}_{d,j}} \right)  \prob{\corr{S}=\corr{s}_{d,j}, X| |\corr{S}|=d}, \label{eq:M_inv_MCL_c} 
}
proving the first part of the lemma.

\explain
The derivation of turning (\ref{eq:key_for_inverse_matrix}) to (\ref{eq:MCL_inverse_equation}) is a reciprocal action.
Thus, if we view $1-\frac{K-1}{d} \ind{Y=i^{\star} \in \corr{s}_{d,j}}$ as the $(i^{\star}, j)$ entry of some matrix $M'$, (\ref{eq:M_inv_MCL_c}) can be interpreted as $\mmatrix{P}_{i^{\star}} = \mmatrix{M' \corr{P}_d}_{i^{\star}}$, suggesting $M' M_d = I$ since $\corr{P}_d = M_d P$.
We formalize this intuition in the next lemma.
\begin{lemma}
\label{lma:Md_inv_2}
    Let $M'$ be of the form (\ref{eq:Md_inv_MCL}), and recall $M_d$ is defined by (\ref{eq:Md_MCL}). Then, $M'M_d = I$, meaning $M' = M_d^{-1}$.
\end{lemma}
The above lemma finishes the proof of Lemma~\ref{lma:Md_inv}.
\end{proof}

\begin{proof}\textbf{of Lemma~\ref{lma:Md_inv_2}.}
    Let $d$ be fixed.
    Denoted by $A_{i,k}$, the $(i,k)$ entry of $M'M_d$, is the inner product of $i$-th row of $M'$ (\ref{eq:Md_inv_MCL}) and the $k$-th column of $M_d$ (\ref{eq:Md_MCL})
    \eqarr{
        A_{i,k} 
        = 
        \sum_{j=1}^{N_d} \left( 1-\frac{K-1}{d}\ind{Y=i\in \corr{s}_{d,j}} \right) \left( \frac{1}{{K-1 \choose d}}\ind{Y=k\notin \corr{s}_{d,j}} \right) 
        =
        \sum_{j=1}^{N_d} c_{i,k}. \nonumber
    }
    In the following, we will show that the calculation results in the identity matrix
    \eqarr{
        A_{i,k}
        =
        \begin{cases}
            1, & \text{if}\ i=k, \nonumber \\
            0, & \text{if}\ i\neq k, \nonumber
        \end{cases}
    }
    to complete the proof.
    
    When $i\neq k$, we have 4 possible cases:
    (i) Both $i$ and $k$ are in $\corr{s}_{d,j}$, (ii) Both of them are not in $\corr{s}_{d,j}$, (iii) $i\in \corr{s}_{d,j}$ and $k\notin \corr{s}_{d,j}$, and (iv) $i\notin \corr{s}_{d,j}$ and $k\in \corr{s}_{d,j}$.
    For cases (i) and (iv), the coefficients $c_{i,k}$ are 0 since $\ind{k\notin \corr{s}_{d,j}} = 0$ if $k \in \corr{s}_{d,j}$.
    For case (ii), the coefficient $c_{i,k}$ is $\frac{1}{{K-1 \choose d}}$.
    The number of such $\corr{s}_{d,j}$ is ${K-2 \choose d}$ since we are counting the ways of forming a set of size $d$ from $K-2$ elements.
    For case (iii), the coefficient $c_{i,k}$ is $\left(1-\frac{K-1}{d}\right) \frac{1}{{K-1 \choose d}}$.
    The number of such $\corr{s}_{d,j}$ is ${K-2 \choose d-1}$ since we are counting the ways of forming a set of size $d-1$ from $k-2$ elements.
    Thus, if $i\neq k$,
    \eqarr{
        A_{i,k}
        &=& 
        \frac{1}{{K-1 \choose d}} {K-2 \choose d} + \left(1-\frac{K-1}{d}\right) \frac{1}{{K-1 \choose d}} {K-2 \choose d-1} \nonumber \\
        &=& 
        \frac{{K-2 \choose d}}{{K-1 \choose d}} + \frac{{K-2 \choose d-1}}{{K-1 \choose d}} - \frac{\frac{K-1}{d}{K-2 \choose d-1}}{{K-1 \choose d}} = 0 \nonumber
    }
    since
    \eqarr{
        {K-2 \choose d} + {K-2 \choose d-1} = {K-1 \choose d} = \frac{K-1}{d}{K-2 \choose d-1}. \nonumber
    }
    
    When $i=k$, we have 2 possible cases:
    (i) Both $i$ and $k$ are in $\corr{s}_{d,j}$, (ii) Both are not in $\corr{s}_{d,j}$.
    For case (i), the coefficient $c_{i,k}$ is 0.
    For case (ii), the coefficient $c_{i,k}$ is $\frac{1}{{K-1 \choose d}}$, and the number of such $\corr{s}_{d,j}$ is ${K-1 \choose d}$, as we want to form a set of size $d$ from $K-1$ candidates.
    Therefore, if $i=k$,
    \eqarr{
        A_{i,k}
        = 
        \frac{1}{{K-1 \choose d}} {K-1 \choose d} = 1. \nonumber
    }
\end{proof}

\explain 
We want to elaborate more on the role of Theorem 1 of \cite{partial_21_PPL/Wu/LS/23} in the analyses in Section~\ref{sec:risk_rewrite_CCNs}.
Firstly, as shown in the proof of Lemma~\ref{lma:Md_inv}, it aids the execution of the inversion approach (Theorem~\ref{thm:inv_method}).
The properness $C(S,X)\ind{Y\in S}$ (\ref{eq:P_S|YX_PPL}) can be instantiated to define the entries of $M_d$ (\ref{eq:Md_MCL}), which in turn establishes the key equation (\ref{eq:key_for_inverse_matrix}) enabling us to identify the entries of $M_d^{-1}$ (\ref{eq:M_inv_MCL_c}). 
Composing $M_d^{-1}$, we obtain $\Mcorr{MCL}^{-1}$, a crucial element for applying our framework (\ref{eq:recipe5}).

\explain
Secondly, Theorem 1 of \cite{partial_21_PPL/Wu/LS/23} contributes to the marginal chain approach (Theorem~\ref{thm:marginal_chain}) as well.
The key equations (\ref{eq:PPL_type1}) and (\ref{eq:PPL_type1_to_MCL_margin}) realised from Theorem 1 of \cite{partial_21_PPL/Wu/LS/23} provide the entries of $\Mcorr{PPL}^{\dagger}$ (Lemma~\ref{lma:inv_MPPL}, Section~\ref{sec:GE2_PPL}), $\Mcorr{PCPL}^{\dagger}$ (Section~\ref{sec:GE2_PCPL}), and $\Mcorr{MCL}^{\dagger}$ (Section~\ref{sec:GE2_MCL}) when applying (\ref{eq:recipe5}).
Therefore, the combined advantage of our framework and Theorem 1 of \cite{partial_21_PPL/Wu/LS/23} provides CCN scenarios unified analyses whose key steps can also be rationally interpreted.
Moreover, as will be shown later, we compare the marginal chain and the inversion approaches via a CL example in Section~\ref{sec:GE2_CL}. 
A CL example is the simplest way to convey the differences between the two methods without burying the essence in complicated derivations.


\subsubsection{Complementary-Label (CL) Learning}
\label{sec:GE2_CL}
\paragraph{Step 1: Corrected Loss Design and Risk Rewrite.}~\\
Note that the parameters chosen for the construction of $\Mcorr{CL}$ (\ref{eq:MCL}) in Section~\ref{sec:GE1_CL} reduces $\Mcorr{MCL}$ (\ref{eq:Mcorr_MCL_regroup}) to be $M_1$ of (\ref{eq:Md_MCL}).
That is, assigning $\prob{|\corr{S}|=d}=1$ for $d=1$, $\prob{|\corr{S}|=d}=0$ for $d>1$, and $\corr{s}_{1,j}=\{j\}$ for all $j\in[K]$ in (\ref{eq:Mcorr_MCL_regroup}), we have
\eqarr{
    \Mcorr{MCL}
    \rightarrow
    M_1
    =
    \frac{1}{K-1}
    \mmatrix{
        0 & 1 & \cdots & 1 \\
        1 & 0 & \cdots & 1 \\
        \vdots & \vdots & \ddots & \vdots \\
        1 & 1 & \cdots & 0
    }
    =
    \Mcorr{CL}. \nonumber
}
Hence, the proof steps of Theorem~\ref{thm:rewrite_MCL_inv} carry over to CL learning. 
With a simple rearranging on
\eqarr{
    \corr{\ell}_{\corr{S}}
    &=& 
    \sum_{i\notin\corr{S}} \ell_{Y=i}
    - \frac{K-1-|\corr{S}|}{|\corr{S}|} \sum_{\corr{s} \in \corr{S}} \ell_{Y=\corr{s}} \nonumber \\ 
    &=&
    \sum_{i=1}^K \ell_{Y=i}
    - \frac{K-1}{|\corr{S}|} \sum_{\corr{s} \in \corr{S}} \ell_{Y=\corr{s}} \nonumber
}
and assigning $\corr{|S|}=1$, we arrive at (\ref{eq:review_rewrite_CL}):
\begin{corollary}
\label{thm:rewrite_CL}
    For CL learning, the classification risk can be expressed as 
    \eqarr{
        R(g) 
        = 
        \expt{\corr{S},X}{\corr{\ell}_{\corr{S}}} \nonumber
        = 
        \expt{\corr{S},X}
        {
        \sum_{i=1}^{K} \ell_{Y=i} - (K-1)\ell_{\corr{S}}
        }. \nonumber
    }
\end{corollary}

\paragraph{Step 2: Recovering the previous result(s).}~\\
\indent
The rewrite above recovers Theorem 1 of \cite{comp_18/Ishida/NMS/19} if we substitute $\corr{S}$ with $\bar{Y}$ and $\ell_{\corr{S}}$ with $\ell(\bar{Y}, g(X))$. 
Moreover, if we choose $d=1$ and $\corr{s}_{1,j}=\{j\}$ for all $j\in [K]$, the decontamination matrix provided by (\ref{eq:Md_inv_MCL}) becomes 
\eqarr{
    M_1^{-1}
    =
    \mmatrix{
        -(K-2) & 1 & \cdots & 1 \\
        1 & -(K-2) & \cdots & 1 \\
        \vdots & \vdots & \ddots & \vdots \\
        1 & 1 & \cdots & -(K-2)
    }, \label{eq:inv_MCL}
}
which translates the corrected losses $\corr{L}^{\top} = L^{\top} M_1^{-1}$ as 
$$
L^{\top} \left( -(K-2)\mathbf{I}_K + \mathbf{1}\mathbf{1}^{\top} \right),
$$
recovering (9) of \cite{comp_18/Ishida/NMS/19}.

\paragraph{Comparing inversion with marginal chain via an example.}~\\
\indent
We use a simple CL example to demonstrate the differences between the inversion (Theorem~\ref{thm:inv_method}) and the marginal chain (Theorem~\ref{thm:marginal_chain}) approaches and explain how the intuition of decontamination is implemented.
Here, we focus on comparing how a decontamination matrix $\Mcorr{corr}^{\dagger}$ achieves $\Mcorr{corr}^{\dagger} \corr{P} = P$ (\ref{eq:recipe6}) since when the equality is established, the downstream construction of the corrected losses 
and the risk rewrite follow the framework.
For this example, let us choose $K=4$ and simplify $\prob{Y=k, X}$ as $p_k$.
Applying (\ref{eq:MCL}), the contamination process defining the data-generating distributions is expressed as
\eqarr{
    \corr{P}
    =
    \Mcorr{CL} P
    =
    \frac{1}{3}
    \mmatrix{
        0 & 1 & 1 & 1 \\
        1 & 0 & 1 & 1 \\
        1 & 1 & 0 & 1 \\
        1 & 1 & 1 & 0
    }
    \mmatrix{
        p_1 \\ p_2 \\ p_3 \\ p_4
    } 
    =
    \mmatrix{
        \frac{p_2+p_3+p_4}{3} \\
        \frac{p_1+p_3+p_4}{3} \\
        \frac{p_1+p_2+p_4}{3} \\
        \frac{p_1+p_2+p_3}{3}
    }. \nonumber
}
Equation (\ref{eq:inv_MCL}), simplified from (\ref{eq:Md_inv_MCL}), provides the decontamination matrix from the inversion approach:
\eqarr{
    \Mcorr{CL}^{-1}
    =
    \mmatrix{
        -2 & 1 & 1 & 1 \\
        1 & -2 & 1 & 1 \\
        1 & 1 & -2 & 1 \\
        1 & 1 & 1 & -2 
    }. \nonumber
}
Then, the {inversion} approach (Theorem~\ref{thm:inv_method}) achieves the decontamination (\ref{eq:recipe6}) by showing
\eqarr{
    \Mcorr{CL}^{-1} \corr{P}
    &=&
    \frac{1}{3}
    \mmatrix{
        -2 & 1 & 1 & 1 \\
        1 & -2 & 1 & 1 \\
        1 & 1 & -2 & 1 \\
        1 & 1 & 1 & -2 
    }
    \mmatrix{
        0 & 1 & 1 & 1 \\
        1 & 0 & 1 & 1 \\
        1 & 1 & 0 & 1 \\
        1 & 1 & 1 & 0
    }
    \mmatrix{
        p_1 \\ p_2 \\ p_3 \\ p_4
    } \label{eq:CL_ex_inv} \\
    &=&
    \frac{1}{3}
    \mmatrix{
        3 & 0 & 0 & 0 \\
        0 & 3 & 0 & 0 \\
        0 & 0 & 3 & 0 \\
        0 & 0 & 0 & 3
    } 
    \mmatrix{
        p_1 \\ p_2 \\ p_3 \\ p_4
    }
    =
    \mmatrix{
        p_1 \\ p_2 \\ p_3 \\ p_4
    }
    = P. \nonumber
}

On the other hand, equation (\ref{eq:PPL_type1_to_MCL_margin}) produces the decontamination matrix from the marginal chain approach:
\eqarr{
    \Mcorr{CL}^{\dagger}
    =
    \mmatrix{
        \frac{0\cdot p_1}{p_2 + p_3 + p_4} & \frac{p_1}{p_1 + p_3 + p_4} & \frac{p_1}{p_1 + p_2 + p_4} & \frac{p_1}{p_1 + p_2 + p_3} \\
        \frac{p_2}{p_2 + p_3 + p_4} & \frac{0 \cdot p_2}{p_1 + p_3 + p_4} & \frac{p_2}{p_1 + p_2 + p_4} & \frac{p_2}{p_1 + p_2 + p_3} \\
        \frac{p_3}{p_2 + p_3 + p_4} &        \frac{p_3}{p_1 + p_3 + p_4} & \frac{0 \cdot p_3}{p_1 + p_2 + p_4} & \frac{p_3}{p_1 + p_2 + p_3} \\
        \frac{p_4}{p_2 + p_3 + p_4} & \frac{p_4}{p_1 + p_3 + p_4} & \frac{p_4}{p_1 + p_2 + p_4} & \frac{0 \cdot p_4}{p_1 + p_2 + p_3}
    }. \nonumber
}
Then, the {marginal chain} approach (Theorem~\ref{thm:marginal_chain}) achieves the decontamination (\ref{eq:recipe6}) by showing
\eqarr{
    \Mcorr{CL}^{\dagger} \corr{P}
    &=&
    \mmatrix{
        \frac{0\cdot p_1}{p_2 + p_3 + p_4} & \frac{p_1}{p_1 + p_3 + p_4} & \frac{p_1}{p_1 + p_2 + p_4} & \frac{p_1}{p_1 + p_2 + p_3} \\
        \frac{p_2}{p_2 + p_3 + p_4} & \frac{0 \cdot p_2}{p_1 + p_3 + p_4} & \frac{p_2}{p_1 + p_2 + p_4} & \frac{p_2}{p_1 + p_2 + p_3} \\
        \frac{p_3}{p_2 + p_3 + p_4} &        \frac{p_3}{p_1 + p_3 + p_4} & \frac{0 \cdot p_3}{p_1 + p_2 + p_4} & \frac{p_3}{p_1 + p_2 + p_3} \\
        \frac{p_4}{p_2 + p_3 + p_4} & \frac{p_4}{p_1 + p_3 + p_4} & \frac{p_4}{p_1 + p_2 + p_4} & \frac{0 \cdot p_4}{p_1 + p_2 + p_3}
    }
    \mmatrix{
        \frac{p_2+p_3+p_4}{3} \\
        \frac{p_1+p_3+p_4}{3} \\
        \frac{p_1+p_2+p_4}{3} \\
        \frac{p_1+p_2+p_3}{3}
    } \label{eq:CL_ex_mar} \\
    &=&
    \mmatrix{
        \frac{p_1+p_1+p_1}{3} \\
        \frac{p_2+p_2+p_2}{3} \\
        \frac{p_3+p_3+p_3}{3} \\
        \frac{p_4+p_4+p_4}{3} 
    } 
    =
    \mmatrix{
        p_1 \\ p_2 \\ p_3 \\ p_4
    }
    = P . \nonumber
}

\explain
Comparing (\ref{eq:CL_ex_inv}) and (\ref{eq:CL_ex_mar}), we see that the intuition of decontamination is realized differently. 
The inversion approach directly cancels out the effect of $\Mcorr{corr}$ without relying on any property of $P$.
In contrast, the marginal chain method leverages the fact that $P$ is a probability vector and carries out a procedure similar to importance reweighting to resolve the contamination.
\future
Both methods have respective merits, and we hope the comparison will inspire new thoughts leveraging certain properties of $P$ for the corrected loss 
design and the study of decontamination.

\subsection{Confidence-based Scenarios}
\label{sec:risk_rewrite_Confs}
The proposed framework is now applied to conduct the risk rewrites for WSLs discussed in Section~\ref{sec:formulations_confs} and summarized in Table~\ref{tab:conf_matrices_summary}.

\subsubsection{Subset Confidence (Sub-Conf) Learning}
\label{sec:GE2_Sub-Conf}
\paragraph{Step 1: Corrected Loss Design and Risk Rewrite.}~\\
Let us follow the notations in Section~\ref{sec:GE1_Sub-Conf}.
To cancel out the contamination caused by $\Mcorr{Sub}$ (\ref{eq:MSub}), we apply Theorem~\ref{thm:inv_method} to construct the decontamination matrix $\Mcorr{Sub}^{\dagger}$.
\begin{lemma}
\label{lma:MSub_inv}
    Assume $\prob{Y\in\mathcal{Y}_\mrm{s}|X} > 0$ for all possible outcomes of $X$.
    Define 
    \eqarr{
        \Mcorr{Sub}^{\dagger}
        :=
        \mmatrix{
            \frac{\prob{Y=1|X}}{\prob{Y\in\mathcal{Y}_\mrm{s}|X}} & \cdots & 0 \\
            \vdots & \ddots & \vdots \\
            0 & \cdots & \frac{\prob{Y=K|X}}{\prob{Y\in\mathcal{Y}_\mrm{s}|X}} \\
        }. \nonumber
    }
    Then, realizing (\ref{eq:recipe5}) as $\corr{L}^{\top} := L^{\top} \Mcorr{Sub}^{\dagger}$, we have $\corr{L}^{\top} \corr{P} = L^{\top} P$
\end{lemma}
\begin{proof}
    Since $\Mcorr{Sub}$ is invertible, we follow Theorem~\ref{thm:inv_method} to define $\Mcorr{Sub}^{\dagger} := \Mcorr{Sub}^{-1}$ so that 
    \eqarr{
        \Mcorr{Sub}^{\dagger} \corr{P} 
        = 
        \Mcorr{Sub}^{-1} \Mcorr{Sub} P
        = 
        P, \nonumber
    }
    where $\corr{P} = \Mcorr{Sub} P$ is given by Lemma~\ref{lma:formulate_Sub-Conf}.
    The equalities further imply 
    \eqarr{
        \corr{L}^{\top} \corr{P} = L^{\top} \Mcorr{Sub}^{\dagger} \corr{P} = L^{\top} P. \nonumber
    }
\end{proof}
Then, we achieve the rewrite (\ref{eq:review_rewrite_Sub-conf}) as follows.
\begin{theorem}
\label{thm:rewrite_Sub-Conf}\label{thm:rewrite_CONF_Sub-Conf}
    For Sub-Conf learning, the classification risk can be written as
    \eqarr{
        R(g)
        =
        \pi_{\mathcal{Y}_\mrm{s}} \expt{X|Y\in\mathcal{Y}_\mrm{s}}{\sum_{i=1}^{K} \frac{r_{i}(X)}{r_{\mathcal{Y}_\mrm{s}}(X)} \ell_{i}}. \nonumber
    }
\end{theorem}
\begin{proof}
    According to Lemma~\ref{lma:MSub_inv}, we have
    \eqarr{
        \mmatrix{\corr{L}^{\top}}_{i}
        =
        \mmatrix{L^{\top} \Mcorr{Sub}^{\dagger}}_{i}
        =
        \frac{\prob{Y=i|X}}{\prob{Y\in\mathcal{Y}_\mrm{s}|X}} \ell_{i} \nonumber
    }
    for each $i \in [K]$.
    Then, applying (\ref{eq:recipe2a}), we obtain
    \eqarr{
        R(g)
        &=&
        \int_{x\in\mathcal{X}} L^{\top} P \dx
        =
        \int_{x\in\mathcal{X}} \corr{L}^{\top} \corr{P} \dx 
        =
        \int_{x\in\mathcal{X}} \sum_{i=1}^{K} \frac{\prob{Y=i|X}}{\prob{Y\in\mathcal{Y}_\mrm{s}|X}} \ell_{i} \cdot \prob{Y\in\mathcal{Y}_\mrm{s}} \prob{X|Y\in\mathcal{Y}_\mrm{s}} \dx \nonumber \\
        &=&
        \prob{Y\in\mathcal{Y}_\mrm{s}} \expt{X|Y\in\mathcal{Y}_\mrm{s}}{\sum_{i=1}^{K} \frac{\prob{Y=i|X}}{\prob{Y\in\mathcal{Y}_\mrm{s}|X}} \ell_{i}} \nonumber \\
        &=&
        \pi_{\mathcal{Y}_\mrm{s}} \expt{X|Y\in\mathcal{Y}_\mrm{s}}{\sum_{i=1}^{K} \frac{r_{i}(X)}{r_{\mathcal{Y}_\mrm{s}}(X)} \ell_{i}} \nonumber
    }
    by following the notations in Section~\ref{sec:review_Sub-Conf}.
\end{proof}

\paragraph{Step 2: Recovering the previous result(s).}~\\
\indent
Notation matching gives
\eqarr{
    R(g) 
    = 
    \pi_{\mathcal{Y}_\mrm{s}}\expt{p(x|y\in\mathcal{Y}_\mrm{s})}{\sum_{y=1}^{K}\frac{r^{y}(x)}{r^{\mathcal{Y}_\mrm{s}}(x)}\ell(g(x),y)}, \nonumber
}
recovering Theorem 6 of \cite{scconf_21/Cao/FSXANS/21}\footnote{
The matching is as follows:
$\prob{X|Y\in\mathcal{Y}_\mrm{s}}$ is $p(x|y\in\mathcal{Y}_\mrm{s})$,
$r_{i}(X)$ is $r^{i}(X)$, 
$r_{\mathcal{Y}_\mrm{s}}(X)$ is $r^{\mathcal{Y}_\mrm{s}}(X)$, and
$\ell_{i}$ is $\ell(g(X),i)$.
}.

\subsubsection{Single-Class Confidence (SC-Conf) Learning}
\label{sec:GE2_SCConf}
\paragraph{Step 1: Corrected Loss Design and Risk Rewrite.}~\\
The SC-Conf derivation resembles that in Section~\ref{sec:GE2_Sub-Conf} since $\Mcorr{SC}$ is a child of $\Mcorr{Sub}$ on the reduction graph.
Thus, following the notations in Section~\ref{sec:GE1_SCConf}, assuming $\prob{Y=y_\mrm{s}|X} > 0$ for all possible outcomes of $X$, and replacing the set $\mathcal{Y}_\mrm{s}$ in $\Mcorr{Sub}^{\dagger}$ with a singleton $y_\mrm{s}$, we have
\eqarr{
    \Mcorr{SC}^{\dagger}
    :=
    \mmatrix{
        \frac{\prob{Y=1|X}}{\prob{Y=y_\mrm{s}|X}} & \cdots & 0 \\
        \vdots & \ddots & \vdots \\
        0 & \cdots & \frac{\prob{Y=K|X}}{\prob{Y=y_\mrm{s}|X}} \\
    } \nonumber
}
satisfying $\Mcorr{SC}^{\dagger} \corr{P} = P$. 
We also obtain $\corr{L}^{\top} = L^{\top} \Mcorr{SC}^{\dagger}$ and $\corr{L}^{\top} \corr{P} = L^{\top} P$ by inheriting the proof of Lemma~\ref{lma:MSub_inv}.
Then, a variant of Theorem~\ref{thm:rewrite_Sub-Conf} replacing
$
    \mmatrix{\corr{L}^{\top}}_{i}
    =
    \frac{\prob{Y=i|X}}{\prob{Y\in\mathcal{Y}_\mrm{s}|X}} \ell_{i}
$
with
\eqarr{
        \mmatrix{\corr{L}^{\top}}_{i}
        =
        \mmatrix{L^{\top} \Mcorr{SC}^{\dagger}}_{i}
        =
        \frac{\prob{Y=i|X}}{\prob{Y=y_\mrm{s}|X}} \ell_{i} 
        =
        \frac{r_{i}(X)}{r_{y_\mrm{s}}(X)} \ell_{i} \nonumber
    }
rewrites the classification risk and proves (\ref{eq:review_rewrite_SC-conf}) for SC-Conf learning:
\begin{corollary}
\label{thm:rewrite_SCConf}
    For SC-Conf learning, the classification risk can be written as
    \eqarr{
        R(g) 
        = 
        \pi_{y_\mrm{s}} \expt{X|Y=y_\mrm{s}}{\sum_{i=1}^{K} \frac{r_{i}(X)}{r_{y_\mrm{s}}(X)} \ell_{i}}. \nonumber
    }
\end{corollary}

\paragraph{Step 2: Recovering the previous result(s).}~\\
\indent
By matching notations, we obtain 
\eqarr{
    R(g) 
    = 
    \pi_{y_\mrm{s}}\expt{p(x|y_\mrm{s})}{\sum_{y=1}^{K}\frac{r^{y}(x)}{r^{y_\mrm{s}}(x)}\ell(g(x),y)}, \nonumber
}
recovering Theorem 1 of \cite{scconf_21/Cao/FSXANS/21}\footnote{
The matching is as follows:
$\prob{X|Y=y_\mrm{s}}$ is $p(x|y_\mrm{s})$,
$r_{i}(X)$ is $r^{i}(X)$, 
$r_{y_\mrm{s}}(X)$ is $r^{y_\mrm{s}}(X)$, and
$\ell_{i}$ is $\ell(g(X),i)$.
}.

\subsubsection{Positive-confidence (Pconf) Learning}
\label{sec:GE2_Pconf}
\paragraph{Step 1: Corrected Loss Design and Risk Rewrite.}~\\
Let us follow the notations in Section~\ref{sec:GE1_Pconf}.
Recall that $\Mcorr{Pconf}$ is a child of $\Mcorr{SC}$ on the reduction graph with $K = 2$ and $y_{\mrm{S}} = \rmp$.
Thus, assuming $\prob{Y=\rmp|X} > 0$ for all possible outcomes of $X$ and replacing $K$ and $y_{\mrm{s}}$ in Section~\ref{sec:GE2_SCConf} accordingly, we obtain the decontamination matrix
\eqarr{
    \Mcorr{Pconf}^{\dagger}
    :=
    \mmatrix{
        \frac{\prob{Y=\rmp|X}}{\prob{Y=\rmp|X}} & 0 \\
        0 & \frac{\prob{Y=\rmn|X}}{\prob{Y=\rmp|X}}
    } 
    =
    \mmatrix{
        1 & 0 \\
        0 & \frac{1-r(X)}{r(X)}
    }
    \nonumber
}
and the rewrite (\ref{eq:review_rewrite_Pconf}) reviewed in Section~\ref{sec:review_Pconf}.
\begin{corollary}
\label{thm:rewrite_Pconf}
    For Pconf learning, the classification risk can be written as
    \eqarr{
        R(g)
        =
        \pi_\rmp
        \expt{\mrm{P}}{
            \ell_\rmp + \frac{1-r(X)}{r(X)} \ell_\rmn
        }. \nonumber
    }
\end{corollary}

\paragraph{Step 2: Recovering the previous result(s).}~\\
\indent
By matching notations, we obtain 
\eqarr{
    R(g)
    =
    \pi_{+}\expt{+}{\ell(g(x)) + \frac{1-r(x)}{r(x)} \ell(-g(x))}, \nonumber
}
recovering Theorem 1 of \cite{pconf_17/Ishida/NS/18}\footnote{
The matching is as follows: 
$\pi_\rmp$ is $\pi_{+}$,
$\prob{X|Y=\rmp}$ is $p(x|y=+1)$,
$\prob{Y=\rmp|X}$ is $r(x)$, 
$\prob{Y=\rmn|X}$ is $1-r(x)$,
$\ell_{\rmp}$ is $\ell(g(x))$, and
$\ell_{\rmn}$ is $\ell(-g(x))$.
}.


\blockComment{
\section{Drafting Section}
\subsection{Drafting Subsection}

\subsubsection{\rred{(Drafting)} Similar-Unlabeled (SU) Learning \texorpdfstring{\citep{su_18/Bao/NS/18}}{Lg}}

\subsubsection{\rred{(Drafting)} Similar-Unlabeled (SU) Learning}

\paragraph{Step 1: Corrected Loss Design and Risk Rewrite.}~\\
\indent

\paragraph{Step 2: Recovering the previous result(s).}~\\
\indent
}

\subsubsection{Soft-Label Learning}
\label{sec:GE2_Soft}
\paragraph{Step 1: Corrected Loss Design and Risk Rewrite.}~\\
Since $\Mcorr{Soft}$ is shown to be a child of $\Mcorr{Sub}$ on the reduction graph in Section~\ref{sec:GE1_Soft}, the analysis for soft-label learning resembles the argument in Section~\ref{sec:GE2_Sub-Conf}.
We follow the notations in Section~\ref{sec:GE1_Soft}, substitute $\prob{Y\in\mathcal{Y}_\mrm{s}|X}$ in Lemma~\ref{lma:MSub_inv} with $\prob{Y\in[K]|X} = 1$, and fix $\corr{P}$ in Lemma~\ref{lma:MSub_inv} and Theorem~\ref{thm:rewrite_Sub-Conf} as (\ref{eq:data_dist_soft}) to obtain the decontamination matrix
\eqarr{
    \Mcorr{Soft}^{\dagger}
    :=
    \mmatrix{
        \prob{Y=1|X} & \cdots & 0 \\
        \vdots & \ddots & \vdots \\
        0 & \cdots & \prob{Y=K|X}
    } \nonumber
}
and achieve (\ref{eq:review_rewrite_Soft}) by the next corollary.
\begin{corollary}
\label{thm:rewrite_Soft}
    For soft-label learning, the classification risk can be written as
    \eqarr{
        R(g) 
        = 
        \expt{X}{\sum_{i=1}^{K} \prob{Y=i|X}\ell_{i}}
        =
        \expt{X}{\sum_{i=1}^{K} r_{i}(X)\ell_{i}}. \nonumber
    }
\end{corollary}


\paragraph{Step 2: Recovering the previous result(s).}~\\
\indent
\cite{soft_22/Ishida/YCNS/22} did not focus on the classification risk rewrite problem.
We can modify Corollary~\ref{thm:rewrite_Soft} to provide a risk rewrite for binary soft-label learning mentioned by \cite{soft_22/Ishida/YCNS/22}.
Taking $K=2$, we have
\eqarr{
    R(g) 
    = 
    \expt{X}{\prob{Y=\rmp|X}\ell_{\rmp} + \prob{Y=\rmn|X}\ell_{\rmn}}
    =
    \expt{X}{r(X)\ell_{\rmp} + (1-r(X))\ell_{\rmn}}. \nonumber
}

\section{Conclusion and Outlook}
\label{sec:future}
We set out with the questions wishing to determine if there is a common way to interpret the formation of weak supervision and search for a generic treatment to solve WSL, to understand the essence of WSL.
In response, we proposed a framework that unifies the formulations and analyses of a set of WSL scenarios to provide a common ground to connect, compare, and understand various weakly-supervised signals.
The formulation component of the proposed framework, viewing WSL from a contamination perspective, associates a WSL data-generating process with a base distribution vector multiplied by a contamination matrix.
By instantiating the contamination matrices of WSLs, we revealed a comprehensive reduction graph, Figure~\ref{fig:reduction_map_all}, connecting existing WSLs.
Each vertex contains a contamination matrix and the section index of the WSL scenario which the matrix characterizes.
Each edge represents the reduction relation of two WSLs.
We can see three major branches from the abstract $\Mcorr{corr}$, corresponding to Tables~\ref{tab:MCD_matrices_summary}, \ref{tab:CCN_matrices_summary}, and \ref{tab:conf_matrices_summary} we discussed in Section~\ref{sec:matrixFormulations}.
The analysis component of the proposed framework, tackling the problem from a decontamination viewpoint, working with the technical building blocks Theorems~\ref{thm:inv_method} and \ref{thm:marginal_chain} constitute a generic treatment to solve the risk rewrite problem.
Section~\ref{sec:riskRewrite} discussed in depth how the analysis component conducts risk rewrite and recovers existing results for WSLs.

\begin{figure}[h]
    \centering
    \includegraphics[width=\textwidth]{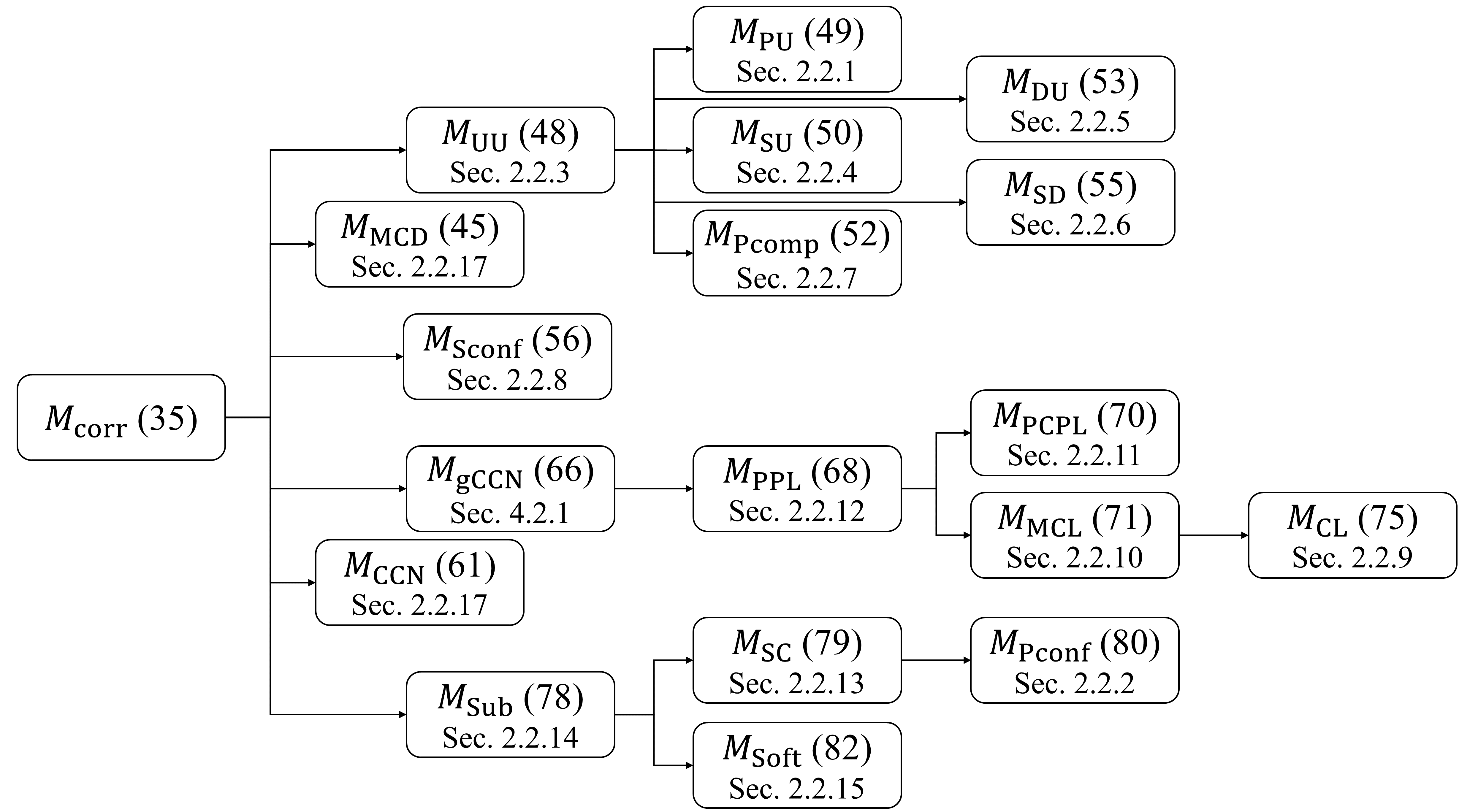}
    \caption{Depicting the reduction map from Tables~\ref{tab:MCD_matrices_summary}, \ref{tab:CCN_matrices_summary}, and \ref{tab:conf_matrices_summary}.}
    \label{fig:reduction_map_all}
\end{figure}

The application of the proposed framework results in a set of theorems. 
We summarize them in Table~\ref{tab:all_theorem_table}.
The Formulation column consists of the results of the formulation component (\ref{eq:recipe1}).
The Decontamination and the Corrected losses columns correspond to the results of the analysis component ((\ref{eq:recipe6}), (\ref{eq:recipe5}), and (\ref{eq:recipe2a})).
The Recovery column justifies the framework by recovering results from the literature.
Crucial results are marked red.

\begin{table}[H]
\centering
\caption{\label{tab:all_theorem_table} Theorem Structure.}
{\renewcommand{\arraystretch}{1.1}
\begin{tabular}[t]{ |c|c|c|c|c| } 
    \hline
    Model & Formulation & Decontamination & Corrected losses & Recovery \\
    & (Find $M$ s.t. & (Find $M^{\dagger}$ s.t. & (Rewrite via & \\
    & $\corr{P} = M B$.) & $P = M^{\dagger} \corr{P}$.) & $\corr{L}^{\top} = L^{\top} M^{\dagger}$ and $\corr{P}$.) & \\
    \hline
    Abstract & (\ref{eq:recipe1}) & (\ref{eq:recipe6}) & (\ref{eq:recipe5}) and (\ref{eq:recipe2a}) & \\
    model & & \rred{Theorem~\ref{thm:inv_method}} and \rred{Theorem~\ref{thm:marginal_chain}} & & \\
    \hline
    MCD & & & & \\
    \rred{UU} & \rred{Lemma~\ref{lma:formulate_UU}} & Corollary~\ref{thm:UU_M_inv} & \rred{Theorem~\ref{thm:UU_rewrite_corrected_losses}} & (Notation swap.) \\
    PU & Lemma~\ref{lma:formulate_PU} & (Immediate reduction.) & Corollary~\ref{thm:rewrite_MCD_PU} & (Notation swap.) \\
    SU & Lemma~\ref{lma:formulate_SU} & (Immediate reduction.) & Corollary~\ref{thm:rewrite_SU} & Lemmas~\ref{lma:recover_SU_1} and \ref{lma:symmetric_similar} \\
    Pcomp & Lemma~\ref{lma:formulate_Pcomp} & (Immediate reduction.) & Corollary~\ref{thm:rewrite_Pcomp} & (Notation swap.) \\
    DU & Lemma~\ref{lma:formulate_DU} & (Immediate reduction.) & Corollary~\ref{thm:rewrite_DU} & Lemmas~\ref{lma:recover_DU_1} and \ref{lma:symmetric_dissimilar} \\
    SD & Lemma~\ref{lma:formulate_SD} & (Immediate reduction.) & Corollary~\ref{thm:rewrite_SD} & Lemma~\ref{lma:recover_SD_1} \\
    Sconf & Lemma~\ref{lma:MSconf} & Lemmas~\ref{lma:key_Sconf} and \ref{lma:inv_MSconf} & Theorem~\ref{thm:rewrite_Sconf} & (Notation swap.) \\
    \hline
    CCN & & & & \\
    \rred{gCCN} & \rred{Lemma~\ref{lma:formulate_gCCN}} &  (\ref{eq:MGeneral_inv_b}) and  Theorem~\ref{thm:marginal_chain} & \rred{Theorem~\ref{thm:rewrite_gCCN}} & (Notation swap.) \\
    PPL & Lemma~\ref{lma:formulate_PPL} & Lemma~\ref{lma:inv_MPPL} & Corollary~\ref{thm:rewrite_PPL} & (Notation swap.) \\
    PCPL & Lemma~\ref{lma:formulate_PCPL} & (Immediate reduction.) & Corollary~\ref{thm:rewrite_PCPL} & Lemma~\ref{lma:recover_PCPL} \\
    MCL & Lemma~\ref{lma:matrix_formulation_MCL} & (Immediate reduction.) & Corollary~\ref{thm:rewrite_MCL} & Theorem~\ref{thm:rewrite_MCL_inv}, \\
    & & & & Lemmas~\ref{lma:Md_inv} and \ref{lma:Md_inv_2}\\
    CL & Lemma~\ref{lma:matrix_formulation_CL} & (Immediate reduction.) & Corollary~\ref{thm:rewrite_CL} & (Notation swap.) \\
    \hline
    \rred{Sub-Conf} & \rred{Lemma~\ref{lma:formulate_Sub-Conf}} & Lemma~\ref{lma:MSub_inv} & \rred{Theorem~\ref{thm:rewrite_Sub-Conf}} & (Notation swap.) \\
    SC-Conf & Lemma~\ref{lma:formulate_SC-Conf} & (Immediate reduction.) & Corollary~\ref{thm:rewrite_SCConf} & (Notation swap.) \\
    Pconf & Lemma~\ref{lma:formulate_Pconf} & (Immediate reduction.) & Corollary~\ref{thm:rewrite_Pconf} & (Notation swap.) \\
    Soft & Lemma~\ref{lma:formulate_soft} & (Immediate reduction.) & Corollary~\ref{thm:rewrite_Soft} & (N/A.) \\
    \hline
\end{tabular}
}   
\end{table}     

The proposed framework is abstract and flexible; hence, we would like to discuss its potential from the following aspects. 
\future 
Firstly,
the performance measure focused on in this paper is the classification risk. 
With proper choices of $P$ and $L$, our framework can be extended to other performance metrics, such as the balanced error rate and cost-sensitive measures \citep{Brodersen/OSB/10, pu_14/Plessis/NS/14, mcd_15/Menon/ROW/15, 16_Scott/Blanchard/FHPS/16, ccn_18/Natarajan/DRT/17, portion_20_Scott/Scott/Z/20}.
\future 
Secondly,
we can explore the formulation capability by exploiting the power of matrix operations.
Cascading matrices allow us to formulate complex scenarios, such as data containing preference relations collected in a noisy environment.
Matrix addition allows us to categorize different contamination mechanisms into cases to capture the structural properties of a problem.
A complicated scenario could undergo a sophisticated formulation procedure, but once we have the resulting contamination matrix, the problem boils down to calculating the corresponding decontamination matrix.
\future 
Thirdly,
the MCD scenarios discussed in this paper (Sections~\ref{sec:formulations_mcds} and \ref{sec:risk_rewrite_MCDs}) belong to binary classification.
A way of extending an MCD formulation to multiclass classification is to extend $\Mcorr{MCD}$ (\ref{eq:MMCD_2}) from a $2 \times 2$ matrix to a $K \times K$ one, in which $K^2-K$ mixture rates are used to characterize the extended $\Mcorr{gMCD}$: the $(i,j)$ entry is $\gamma_{i,j}$ if $i \neq j$ and is $1-\sum_{j \neq i} \gamma_{i,j}$ for the $i$-th entry on the diagonal.
\future 
Fourthly, 
the label-flipping probabilities $\prob{\corr{Y}|Y}$ in \cite{ccn_18/Natarajan/DRT/17} and \cite{partial_20_PCPL/Feng/LHXNGAS/20} assume that the contaminated label $\corr{Y}$ is independent of $X$ condition on the ture label $Y$.
The formulation matrices, $\Mcorr{CCN}$ (\ref{eq:MCCN}) and $\Mcorr{gCCN}$ (\ref{eq:MGeneral_2}), in contrast, take $X$ into consideration.
This formulation enables us to tackle the instance-dependent problem \citep{instance_20/Berthon/HLNS/21} in the future.
\future 
Fifthly,
we hope the marginal chain method can bring up new thoughts for WSL investigations, as it avoids the invertible assumption by exploiting the fact that distributions define the performance measures.
We also project its potential in research regarding the broader sense of contamination and decontamination.
\future 
Sixthly,
the properness of \cite{partial_21_PPL/Wu/LS/23} provides an efficient technique to compute $\prob{Y|S,X}$ needed in $\Mcorr{gCCN}$ (\ref{eq:MGeneral_inv_b}).
It would be intriguing to know if there are any other alternatives.
\future
Finally but not least, the proposed framework operating under matrix multiplication belongs to a broader question of under what circumstances does a function $f^{\dagger}$ exist with $P = f^{\dagger}(\corr{P})$ if $\corr{P} = f(P)$. 


\acks{The authors were supported by the Institute for AI and Beyond, UTokyo.
The first author would like to thank Professor Takashi Ishida (UTokyo) for valuable insights and discussions in extending the coverage of the framework and the colleagues, Xin-Qiang Cai, Masahiro Negishi, Wei Wang, and Yivan Zhang (in alphabetical order), for comments in improving the manuscript.
}

\newpage
\appendix

\section{Notations}
\label{sec:notations}

\begin{table}[H]
\centering
\caption{\label{tab:large_notation_table_v1} Notations and Aliases.}
{\renewcommand{\arraystretch}{1.2}
\begin{tabular}[t]{ |l|l|l|l| } 
    \hline
    Name of the notation & Expression & Aliases & Convention \\
    \hline
    Example & $(y, x)$ & & $(x, y)$ \\
    \hline
    Binary classes & $\{\rmp, \rmn\}$ & & $\{+1, -1\}$ \\
    Multiple classes & $\{1, \cdots, K\}$ & $[K]$ & \\
    Compound classes of $[K]$ & $2^{[K]} \backslash \left\{\emptyset, [K] \right\}$ & $\mathcal{S}$ & \\
    A subset of classes & $\mathcal{Y}_{s} \subset [K]$ & & \\
    \hline
    Joint distribution & $\Pr(Y=y, X=x)$ & $\prob{Y=y,x}$, $\prob{Y=y,X}$, or $\prob{Y,X}$ & $\Pr(x, y)$ \\
    Class prior & $\Pr(Y=y)$ & $\pi_{y}$ & \\
    Marginal & $\Pr(X)$ & $\prob{X}$ & \\
    Class-conditional & $\Pr(X=x \mid Y=y)$ & $\prob{x|y}$, $\prob{X|y}$, $\prob{x|Y=y}$, or $\prob{X|Y=y}$ & \\
    Class probability & $\Pr(Y=y \mid X=x)$ & $\prob{Y=y|x}$, $\prob{Y=y|X}$, or $\prob{Y|X}$ & $\eta(x)$ \\
    Confidence & $\Pr(Y=y\mid X=x)$ & $r_{y}(X)$, $r_{y}(x)$, or $r(X)$ if $y=\rmp$ & $r^{y}(x)$ or $r(x)$ \\
    Sample size probability & $\Pr(|S|=d)$ & $\prob{|S|=d}$ or $q_{|S|}$ & \\
    \hline
    Hypothesis and its space & $g \in \mathcal{G}$ & & \\
    \hline
    Loss of $g$ & $\ell_{Y=y}(g(x))$ & $\ell_{y}$, $\ell_{y}(X)$, or $\ell_{Y}(g(X))$ & $\ell(g(X),Y)$ \\
    \hline
    Classification risk & $\expt{Y,X}{\ell_{Y}(g(X))}$ & $R(g)$ & $\expt{X,Y}{\ell(g(X),Y)}$ \\
    \hline
    The $j$-th entry of vector $V$ & $\left(V\right)_{j}$ & $V_j$ & \\
    \hline
    Indicator function of $E$ & $\ind{E}$ & & \\
    \hline
    Complement of set $s$ & $\mathcal{Y}\backslash s$ & $\corr{s}$ & \\
    \hline
    Identity matrix & $I$ & & \\
    \hline
    MCD parameters & $\mcdpp$ and $\mcdnn$ & & \\
    \hline
    UU parameters & $\mcdp$ and $\mcdn$ & & $1-\theta$ and $\theta'$ \\
    \hline
    CCN parameters & $\prob{\corr{Y}|Y,X}$ & $\prob{S|Y,X}$ or $\prob{\corr{S}|Y,X}$ & $\rho_{+}$ and $\rho_{-}$ \\
    \hline
\end{tabular}
}
\end{table}

\vskip 0.2in
\bibliography{refs}

\end{document}